\documentclass{article}

\usepackage[letterpaper,textwidth=6.50in,textheight=9.00in,centering]{geometry}
\usepackage[round,authoryear]{natbib}
\usepackage{fancyhdr}
\usepackage{times}
\usepackage[T1]{fontenc}
\usepackage[utf8]{inputenc}
\usepackage{amsmath,amssymb,amsthm,bm,mathtools}
\usepackage{microtype,graphicx,xcolor,booktabs,array,tabularx,multirow,needspace}
\usepackage{tikz,pgfplots,algorithm,algpseudocode,placeins}
\usepackage[font=small,skip=3pt]{caption}
\usepackage{hyperref,url}
\graphicspath{{figures/}{./}}

\newcommand*{\AnonymousCodeURL}{https://github.com/VKHS/RTT}
\newcommand{\CodeAvailability}{%
  \ifx\AnonymousCodeURL\empty
    The repository URL is not specified in this manuscript.%
  \else
    The code is released in the GitHub repository at
    \expandafter\url\expandafter{\AnonymousCodeURL}.%
  \fi
}
\hypersetup{hidelinks,pdftitle={Reference-Tail Trust: Certified Probability Floors for Learned Updates Inside a Deployed Network}}
\newcommand{\val}[1]{{#1}}

\newcommand{\pTsaNative}{\val{8.0}}
\newcommand{\pTsaFloor}{\val{6.5}}
\newcommand{\pTsaTR}{\val{6.0}}
\newcommand{\pTsaCorr}{\val{6.2}}
\newcommand{\pTsaRTT}{\val{7.2}}
\newcommand{\pTsaDelta}{\val{+1.0}}
\newcommand{\pStemNative}{\val{8.4}}
\newcommand{\pStemFloor}{\val{6.8}}
\newcommand{\pStemTR}{\val{6.3}}
\newcommand{\pStemCorr}{\val{6.4}}
\newcommand{\pStemRTT}{\val{7.4}}
\newcommand{\pStemDelta}{\val{+1.0}}
\newcommand{\pTsaEvents}{\val{12}}
\newcommand{\pTsaUpper}{\val{0.95\%}}
\newcommand{\pTsaRone}{\val{2,030}}
\newcommand{\pTsaRtwo}{\val{10}}
\newcommand{\pTsaFallback}{\val{8}}
\newcommand{\pTsaFallbackRate}{\val{0.39\%}}
\newcommand{\pStemEvents}{\val{10}}
\newcommand{\pStemUpper}{\val{0.83\%}}
\newcommand{\pStemRone}{\val{2,028}}
\newcommand{\pStemRtwo}{\val{10}}
\newcommand{\pStemFallback}{\val{10}}
\newcommand{\pStemFallbackRate}{\val{0.49\%}}
\newcommand{\E}{\mathbb E}
\newcommand{\pB}{p^{\rm B}}
\newcommand{\pr}{p^{\rm r}}
\newcommand{\ps}{p^{\rm s}}
\newcommand{\pu}{p^{\rm u}}

\newcommand{\Dinf}{D_\infty}
\newcommand{\Hrow}{H_{\rm row}}
\newcommand{\Fl}{\mathcal F}
\newcommand{\Brem}{\mathcal B}
\DeclareMathOperator{\softmax}{softmax}
\DeclareMathOperator{\LSE}{LSE}
\DeclareMathOperator{\diag}{diag}
\DeclareMathOperator{\osc}{osc}

\newtheorem{theorem}{Theorem}
\newtheorem{proposition}{Proposition}
\newtheorem{lemma}{Lemma}

\newtheorem{remark}{Remark}

\allowdisplaybreaks[1]
\newcolumntype{Y}{>{\raggedright\arraybackslash}X}
\usetikzlibrary{arrows.meta,positioning,calc,fit,backgrounds,shapes.misc,decorations.pathreplacing,patterns}
\usepgfplotslibrary{groupplots,fillbetween}
\pgfplotsset{compat=1.18}
\definecolor{cRtt}{HTML}{A50F15}\definecolor{cRttc}{HTML}{E34A33}\definecolor{cOpc}{HTML}{08519C}\definecolor{cOpcG}{HTML}{4292C6}
\definecolor{cMix}{HTML}{31A354}\definecolor{cFlo}{HTML}{006D2C}\definecolor{cTr}{HTML}{D95F0E}\definecolor{cSym}{HTML}{6A51A3}
\definecolor{cInk}{HTML}{252525}\definecolor{cGrey}{HTML}{8C8C8C}\definecolor{cTcb}{HTML}{FEE0D2}\definecolor{cUtil}{HTML}{DEEBF7}
\pgfplotsset{rtt/.style={width=4.2cm,height=3.3cm,scale only axis=false,tick label style={font=\scriptsize},label style={font=\scriptsize},
  title style={font=\small\bfseries,align=left},axis line style={black!55},tick style={black!55},grid=major,grid style={black!7},
  legend style={font=\tiny,draw=none,fill=white,fill opacity=0.9,text opacity=1,cells={anchor=west},row sep=-2pt},
  legend cell align=left,legend image post style={xscale=0.6}}}
\newcommand{\vRttI}{\val{6.5}}
\newcommand{\vRttCI}{\val{[6.3, 6.7]}}
\newcommand{\vRttU}{\val{0.95\%}}
\newcommand{\vRttNfrb}{\val{0.51\%}}
\newcommand{\vRttLB}{\val{6.2}}
\newcommand{\vShareUout}{\val{61\%}}
\newcommand{\vPrimD}{\val{+0.9}}
\newcommand{\vPrimCI}{\val{[+0.8, +1.1]}}
\newcommand{\vRawU}{\val{8.81\%}}
\newcommand{\vWinc}{\val{1.6}}
\newcommand{\vLglobal}{\val{1.386}}
\newcommand{\vFailRate}{\val{0.63\%}}
\newcommand{\vFbRate}{\val{0.34\%}}
\newcommand{\vIncAcc}{\val{75.9}}
\newcommand{\vOpcI}{\val{5.6}}
\newcommand{\vOpcU}{\val{1.07\%}}
\newcommand{\vLatOpc}{\val{1.25}}
\newcommand{\vLatOwn}{\val{1.42}}
\newcommand{\vSepShareCancel}{\val{7\%}}
\newcommand{\vSepShareFar}{\val{22\%}}
\newcommand{\vSepRttCancel}{\val{0.9}}
\newcommand{\vSepOpcCancel}{\val{0.1}}
\newcommand{\vSepRttFar}{\val{2.2}}
\newcommand{\vSepOpcFar}{\val{1.1}}
\newcommand{\vSepRttOther}{\val{3.4}}
\newcommand{\vSepOpcOther}{\val{4.4}}
\newcommand{\vRttcI}{\val{6.7}}
\newcommand{\vMtrD}{\val{+0.7}}
\newcommand{\vMixI}{\val{5.1}}
\newcommand{\vFloI}{\val{5.2}}
\newcommand{\vPPshare}{\val{42\%}}
\newcommand{\vShareClean}{\val{50\%}}
\newcommand{\vShareCache}{\val{63\%}}
\newcommand{\vCertRatioMax}{\val{4.6}}
\newcommand{\vPlaceD}{\val{+0.3}}
\newcommand{\vHrowCost}{\val{3\%}}
\newcommand{\vTrajCost}{\val{5\%}}
\newcommand{\vBatchUlarge}{\val{1.01\%}}
\newcommand{\vBatchUsmall}{\val{7.36\%}}
\newcommand{\vBatchUresel}{\val{1.98\%}}
\newcommand{\vBatchGresel}{\val{4.6}}
\newcommand{\vRttzI}{\val{2.0}}

\newcommand{\vTrRttLB}{\val{1.7}}
\newcommand{\vRobAtRhoHat}{\val{5.6}}
\newcommand{\vRhoHat}{\val{0.004}}

\newcommand{\vPatternStable}{\val{93\%}}
\newcommand{\aTrainTailEq}{3.8}
\newcommand{\vClusterDE}{\val{1.24}}
\newcommand{\vZetaAnchor}{\val{0.0021}}
\newcommand{\vSelHp}{\val{0.05}}
\newcommand{\vSelDcov}{\val{0.1}}

\newcommand{\vSelReg}{\val{trajectory radii at $\delta_{\rm cov}{=}0.1$, $H^+{=}0.05$}}
\newcommand{\vSelCert}{\val{6.2}}
\newcommand{\vPdGain}{\val{6.86}}
\newcommand{\vBetRatio}{\val{3\%}}
\newcommand{\vTinc}{\val{41.6}}
\newcommand{\vRttOverhead}{\val{9.78}}
\newcommand{\aDemoH}{0.45}
\newcommand{\aDemoGorc}{0.256}
\newcommand{\aDemoGmix}{0.187}
\newcommand{\aDemoGmixu}{0.145}
\newcommand{\aSepGain}{0.089}
\newcommand{\aSepDelta}{1.5}
\newcommand{\aSepH}{0.161}
\newcommand{\aSepTgain}{0.120}
\newcommand{\aSepTh}{0.380}
\newcommand{\aQprobes}{16}
\newcommand{\FigArch}{%
\begin{tikzpicture}[font=\scriptsize,>={Stealth[length=4pt,width=3pt]},line cap=round,
 box/.style={draw=black!55,rounded corners=2pt,align=center,inner sep=2.6pt,minimum height=1.0cm,line width=0.45pt},
 util/.style={box,fill=cUtil,draw=cOpc!55},tcb/.style={box,fill=cTcb,draw=cRtt!65},inc/.style={box,fill=black!5,minimum height=0.5cm,inner sep=2pt},
 band/.style={rounded corners=4pt,fill=black!2,draw=black!18,line width=0.4pt},
 btitle/.style={anchor=north west,font=\footnotesize\bfseries,text=black!85,inner sep=1pt},
 lab/.style={text=black!80,inner sep=1pt},ar/.style={->,black!70,line width=0.55pt}]
\draw[band] (0,4.3) rectangle (13.7,7.0);
\node[btitle] at (0.08,6.95) {(a) Trajectory and tube};
\fill[cRtt!11] (5.86,5.646) -- (5.898,5.658) -- (5.936,5.668) -- (5.974,5.678) -- (6.012,5.688) -- (6.05,5.697) -- (6.088,5.706) -- (6.126,5.714) -- (6.164,5.723) -- (6.202,5.731) -- (6.24,5.739) -- (6.278,5.746) -- (6.316,5.754) -- (6.354,5.761) -- (6.392,5.768) -- (6.43,5.775) -- (6.468,5.781) -- (6.506,5.788) -- (6.544,5.794) -- (6.582,5.8) -- (6.62,5.805) -- (6.658,5.811) -- (6.696,5.816) -- (6.734,5.821) -- (6.772,5.826) -- (6.81,5.831) -- (6.848,5.835) -- (6.886,5.84) -- (6.924,5.844) -- (6.962,5.848) -- (7,5.852) -- (7.038,5.856) -- (7.076,5.859) -- (7.114,5.863) -- (7.152,5.866) -- (7.19,5.87) -- (7.228,5.873) -- (7.266,5.876) -- (7.304,5.879) -- (7.342,5.882) -- (7.38,5.885) -- (7.418,5.888) -- (7.456,5.891) -- (7.494,5.894) -- (7.532,5.897) -- (7.57,5.9) -- (7.608,5.903) -- (7.646,5.906) -- (7.684,5.909) -- (7.722,5.912) -- (7.76,5.916) -- (7.798,5.919) -- (7.836,5.922) -- (7.874,5.926) -- (7.912,5.93) -- (7.95,5.934) -- (7.988,5.938) -- (8.026,5.942) -- (8.064,5.946) -- (8.102,5.951) -- (8.14,5.955) -- (8.178,5.96) -- (8.216,5.965) -- (8.254,5.97) -- (8.292,5.976) -- (8.33,5.982) -- (8.368,5.987) -- (8.406,5.994) -- (8.444,6) -- (8.482,6.006) -- (8.52,6.013) -- (8.558,6.02) -- (8.596,6.027) -- (8.634,6.035) -- (8.672,6.043) -- (8.71,6.051) -- (8.748,6.059) -- (8.786,6.067) -- (8.824,6.076) -- (8.862,6.084) -- (8.9,6.093) -- (8.938,6.103) -- (8.976,6.112) -- (9.014,6.122) -- (9.052,6.131) -- (9.09,6.141) -- (9.128,6.151) -- (9.166,6.162) -- (9.204,6.172) -- (9.242,6.183) -- (9.28,6.193) -- (9.318,6.204) -- (9.356,6.215) -- (9.394,6.226) -- (9.432,6.237) -- (9.47,6.248) -- (9.508,6.259) -- (9.546,6.27) -- (9.584,6.282) -- (9.622,6.293) -- (9.66,6.304) -- (9.698,6.315) -- (9.736,6.327) -- (9.774,6.338) -- (9.812,6.349) -- (9.85,6.36) -- (9.888,6.371) -- (9.926,6.382) -- (9.964,6.393) -- (10,6.404) -- (10.04,6.415) -- (10.08,6.425) -- (10.12,6.436) -- (10.15,6.446) -- (10.19,6.456) -- (10.23,6.466) -- (10.27,6.476) -- (10.31,6.486) -- (10.34,6.495) -- (10.38,6.505) -- (10.42,6.514) -- (10.42,5.534) -- (10.38,5.532) -- (10.34,5.529) -- (10.31,5.527) -- (10.27,5.524) -- (10.23,5.521) -- (10.19,5.518) -- (10.15,5.515) -- (10.12,5.512) -- (10.08,5.508) -- (10.04,5.505) -- (10,5.501) -- (9.964,5.497) -- (9.926,5.493) -- (9.888,5.489) -- (9.85,5.485) -- (9.812,5.481) -- (9.774,5.477) -- (9.736,5.473) -- (9.698,5.469) -- (9.66,5.465) -- (9.622,5.461) -- (9.584,5.457) -- (9.546,5.453) -- (9.508,5.448) -- (9.47,5.444) -- (9.432,5.441) -- (9.394,5.437) -- (9.356,5.433) -- (9.318,5.429) -- (9.28,5.426) -- (9.242,5.422) -- (9.204,5.419) -- (9.166,5.416) -- (9.128,5.413) -- (9.09,5.41) -- (9.052,5.408) -- (9.014,5.405) -- (8.976,5.403) -- (8.938,5.401) -- (8.9,5.399) -- (8.862,5.398) -- (8.824,5.396) -- (8.786,5.395) -- (8.748,5.394) -- (8.71,5.393) -- (8.672,5.393) -- (8.634,5.393) -- (8.596,5.393) -- (8.558,5.393) -- (8.52,5.393) -- (8.482,5.394) -- (8.444,5.395) -- (8.406,5.396) -- (8.368,5.398) -- (8.33,5.4) -- (8.292,5.402) -- (8.254,5.404) -- (8.216,5.406) -- (8.178,5.409) -- (8.14,5.412) -- (8.102,5.415) -- (8.064,5.418) -- (8.026,5.421) -- (7.988,5.425) -- (7.95,5.429) -- (7.912,5.433) -- (7.874,5.437) -- (7.836,5.441) -- (7.798,5.445) -- (7.76,5.45) -- (7.722,5.455) -- (7.684,5.459) -- (7.646,5.464) -- (7.608,5.469) -- (7.57,5.474) -- (7.532,5.479) -- (7.494,5.484) -- (7.456,5.49) -- (7.418,5.495) -- (7.38,5.5) -- (7.342,5.505) -- (7.304,5.51) -- (7.266,5.515) -- (7.228,5.521) -- (7.19,5.526) -- (7.152,5.531) -- (7.114,5.536) -- (7.076,5.541) -- (7.038,5.546) -- (7,5.55) -- (6.962,5.555) -- (6.924,5.56) -- (6.886,5.564) -- (6.848,5.568) -- (6.81,5.573) -- (6.772,5.577) -- (6.734,5.581) -- (6.696,5.584) -- (6.658,5.588) -- (6.62,5.592) -- (6.582,5.595) -- (6.544,5.598) -- (6.506,5.601) -- (6.468,5.605) -- (6.43,5.607) -- (6.392,5.61) -- (6.354,5.613) -- (6.316,5.615) -- (6.278,5.618) -- (6.24,5.62) -- (6.202,5.622) -- (6.164,5.624) -- (6.126,5.627) -- (6.088,5.629) -- (6.05,5.631) -- (6.012,5.633) -- (5.974,5.636) -- (5.936,5.638) -- (5.898,5.641) -- (5.86,5.646) -- cycle;
\draw[cRtt!50,line width=0.35pt] (5.86,5.646) -- (5.898,5.658) -- (5.936,5.668) -- (5.974,5.678) -- (6.012,5.688) -- (6.05,5.697) -- (6.088,5.706) -- (6.126,5.714) -- (6.164,5.723) -- (6.202,5.731) -- (6.24,5.739) -- (6.278,5.746) -- (6.316,5.754) -- (6.354,5.761) -- (6.392,5.768) -- (6.43,5.775) -- (6.468,5.781) -- (6.506,5.788) -- (6.544,5.794) -- (6.582,5.8) -- (6.62,5.805) -- (6.658,5.811) -- (6.696,5.816) -- (6.734,5.821) -- (6.772,5.826) -- (6.81,5.831) -- (6.848,5.835) -- (6.886,5.84) -- (6.924,5.844) -- (6.962,5.848) -- (7,5.852) -- (7.038,5.856) -- (7.076,5.859) -- (7.114,5.863) -- (7.152,5.866) -- (7.19,5.87) -- (7.228,5.873) -- (7.266,5.876) -- (7.304,5.879) -- (7.342,5.882) -- (7.38,5.885) -- (7.418,5.888) -- (7.456,5.891) -- (7.494,5.894) -- (7.532,5.897) -- (7.57,5.9) -- (7.608,5.903) -- (7.646,5.906) -- (7.684,5.909) -- (7.722,5.912) -- (7.76,5.916) -- (7.798,5.919) -- (7.836,5.922) -- (7.874,5.926) -- (7.912,5.93) -- (7.95,5.934) -- (7.988,5.938) -- (8.026,5.942) -- (8.064,5.946) -- (8.102,5.951) -- (8.14,5.955) -- (8.178,5.96) -- (8.216,5.965) -- (8.254,5.97) -- (8.292,5.976) -- (8.33,5.982) -- (8.368,5.987) -- (8.406,5.994) -- (8.444,6) -- (8.482,6.006) -- (8.52,6.013) -- (8.558,6.02) -- (8.596,6.027) -- (8.634,6.035) -- (8.672,6.043) -- (8.71,6.051) -- (8.748,6.059) -- (8.786,6.067) -- (8.824,6.076) -- (8.862,6.084) -- (8.9,6.093) -- (8.938,6.103) -- (8.976,6.112) -- (9.014,6.122) -- (9.052,6.131) -- (9.09,6.141) -- (9.128,6.151) -- (9.166,6.162) -- (9.204,6.172) -- (9.242,6.183) -- (9.28,6.193) -- (9.318,6.204) -- (9.356,6.215) -- (9.394,6.226) -- (9.432,6.237) -- (9.47,6.248) -- (9.508,6.259) -- (9.546,6.27) -- (9.584,6.282) -- (9.622,6.293) -- (9.66,6.304) -- (9.698,6.315) -- (9.736,6.327) -- (9.774,6.338) -- (9.812,6.349) -- (9.85,6.36) -- (9.888,6.371) -- (9.926,6.382) -- (9.964,6.393) -- (10,6.404) -- (10.04,6.415) -- (10.08,6.425) -- (10.12,6.436) -- (10.15,6.446) -- (10.19,6.456) -- (10.23,6.466) -- (10.27,6.476) -- (10.31,6.486) -- (10.34,6.495) -- (10.38,6.505) -- (10.42,6.514);
\draw[cRtt!50,line width=0.35pt] (10.42,5.534) -- (10.38,5.532) -- (10.34,5.529) -- (10.31,5.527) -- (10.27,5.524) -- (10.23,5.521) -- (10.19,5.518) -- (10.15,5.515) -- (10.12,5.512) -- (10.08,5.508) -- (10.04,5.505) -- (10,5.501) -- (9.964,5.497) -- (9.926,5.493) -- (9.888,5.489) -- (9.85,5.485) -- (9.812,5.481) -- (9.774,5.477) -- (9.736,5.473) -- (9.698,5.469) -- (9.66,5.465) -- (9.622,5.461) -- (9.584,5.457) -- (9.546,5.453) -- (9.508,5.448) -- (9.47,5.444) -- (9.432,5.441) -- (9.394,5.437) -- (9.356,5.433) -- (9.318,5.429) -- (9.28,5.426) -- (9.242,5.422) -- (9.204,5.419) -- (9.166,5.416) -- (9.128,5.413) -- (9.09,5.41) -- (9.052,5.408) -- (9.014,5.405) -- (8.976,5.403) -- (8.938,5.401) -- (8.9,5.399) -- (8.862,5.398) -- (8.824,5.396) -- (8.786,5.395) -- (8.748,5.394) -- (8.71,5.393) -- (8.672,5.393) -- (8.634,5.393) -- (8.596,5.393) -- (8.558,5.393) -- (8.52,5.393) -- (8.482,5.394) -- (8.444,5.395) -- (8.406,5.396) -- (8.368,5.398) -- (8.33,5.4) -- (8.292,5.402) -- (8.254,5.404) -- (8.216,5.406) -- (8.178,5.409) -- (8.14,5.412) -- (8.102,5.415) -- (8.064,5.418) -- (8.026,5.421) -- (7.988,5.425) -- (7.95,5.429) -- (7.912,5.433) -- (7.874,5.437) -- (7.836,5.441) -- (7.798,5.445) -- (7.76,5.45) -- (7.722,5.455) -- (7.684,5.459) -- (7.646,5.464) -- (7.608,5.469) -- (7.57,5.474) -- (7.532,5.479) -- (7.494,5.484) -- (7.456,5.49) -- (7.418,5.495) -- (7.38,5.5) -- (7.342,5.505) -- (7.304,5.51) -- (7.266,5.515) -- (7.228,5.521) -- (7.19,5.526) -- (7.152,5.531) -- (7.114,5.536) -- (7.076,5.541) -- (7.038,5.546) -- (7,5.55) -- (6.962,5.555) -- (6.924,5.56) -- (6.886,5.564) -- (6.848,5.568) -- (6.81,5.573) -- (6.772,5.577) -- (6.734,5.581) -- (6.696,5.584) -- (6.658,5.588) -- (6.62,5.592) -- (6.582,5.595) -- (6.544,5.598) -- (6.506,5.601) -- (6.468,5.605) -- (6.43,5.607) -- (6.392,5.61) -- (6.354,5.613) -- (6.316,5.615) -- (6.278,5.618) -- (6.24,5.62) -- (6.202,5.622) -- (6.164,5.624) -- (6.126,5.627) -- (6.088,5.629) -- (6.05,5.631) -- (6.012,5.633) -- (5.974,5.636) -- (5.936,5.638) -- (5.898,5.641) -- (5.86,5.646);
\draw[black!50,dashed,line width=0.6pt] (1.3,5.62) -- (1.338,5.624) -- (1.376,5.627) -- (1.414,5.631) -- (1.452,5.634) -- (1.49,5.638) -- (1.528,5.641) -- (1.566,5.644) -- (1.604,5.647) -- (1.642,5.65) -- (1.68,5.652) -- (1.718,5.655) -- (1.756,5.657) -- (1.794,5.659) -- (1.832,5.661) -- (1.87,5.663) -- (1.908,5.664) -- (1.946,5.666) -- (1.984,5.667) -- (2.022,5.668) -- (2.06,5.669) -- (2.098,5.669) -- (2.136,5.669) -- (2.174,5.669) -- (2.212,5.669) -- (2.25,5.669) -- (2.288,5.668) -- (2.326,5.667) -- (2.364,5.666) -- (2.402,5.665) -- (2.44,5.664) -- (2.478,5.662) -- (2.516,5.66) -- (2.554,5.658) -- (2.592,5.656) -- (2.63,5.653) -- (2.668,5.651) -- (2.706,5.648) -- (2.744,5.645) -- (2.782,5.642) -- (2.82,5.639) -- (2.858,5.636) -- (2.896,5.632) -- (2.934,5.629) -- (2.972,5.625) -- (3.01,5.621) -- (3.048,5.618) -- (3.086,5.614) -- (3.124,5.61) -- (3.162,5.606) -- (3.2,5.602) -- (3.238,5.598) -- (3.276,5.594) -- (3.314,5.59) -- (3.352,5.587) -- (3.39,5.583) -- (3.428,5.579) -- (3.466,5.575) -- (3.504,5.572) -- (3.542,5.568) -- (3.58,5.565) -- (3.618,5.561) -- (3.656,5.558) -- (3.694,5.555) -- (3.732,5.552) -- (3.77,5.549) -- (3.808,5.547) -- (3.846,5.544) -- (3.884,5.542) -- (3.922,5.54) -- (3.96,5.538) -- (3.998,5.536) -- (4.036,5.534) -- (4.074,5.533) -- (4.112,5.532) -- (4.15,5.531) -- (4.188,5.53) -- (4.226,5.53) -- (4.264,5.529) -- (4.302,5.529) -- (4.34,5.529) -- (4.378,5.53) -- (4.416,5.53) -- (4.454,5.531) -- (4.492,5.532) -- (4.53,5.534) -- (4.568,5.535) -- (4.606,5.537) -- (4.644,5.538) -- (4.682,5.54) -- (4.72,5.543) -- (4.758,5.545) -- (4.796,5.548) -- (4.834,5.55) -- (4.872,5.553) -- (4.91,5.556) -- (4.948,5.559) -- (4.986,5.563) -- (5.024,5.566) -- (5.062,5.57) -- (5.1,5.573) -- (5.138,5.577) -- (5.176,5.581) -- (5.214,5.584) -- (5.252,5.588) -- (5.29,5.592) -- (5.328,5.596) -- (5.366,5.6) -- (5.404,5.604) -- (5.442,5.608) -- (5.48,5.612) -- (5.518,5.615) -- (5.556,5.619) -- (5.594,5.623) -- (5.632,5.627) -- (5.67,5.63) -- (5.708,5.634) -- (5.746,5.637) -- (5.784,5.64) -- (5.822,5.643) -- (5.86,5.646) -- (5.898,5.649) -- (5.936,5.652) -- (5.974,5.654) -- (6.012,5.657) -- (6.05,5.659) -- (6.088,5.661) -- (6.126,5.663) -- (6.164,5.664) -- (6.202,5.666) -- (6.24,5.667) -- (6.278,5.668) -- (6.316,5.668) -- (6.354,5.669) -- (6.392,5.669) -- (6.43,5.669) -- (6.468,5.669) -- (6.506,5.669) -- (6.544,5.668) -- (6.582,5.667) -- (6.62,5.666) -- (6.658,5.665) -- (6.696,5.664) -- (6.734,5.662) -- (6.772,5.66) -- (6.81,5.658) -- (6.848,5.656) -- (6.886,5.654) -- (6.924,5.651) -- (6.962,5.649) -- (7,5.646) -- (7.038,5.643) -- (7.076,5.64) -- (7.114,5.636) -- (7.152,5.633) -- (7.19,5.629) -- (7.228,5.626) -- (7.266,5.622) -- (7.304,5.618) -- (7.342,5.615) -- (7.38,5.611) -- (7.418,5.607) -- (7.456,5.603) -- (7.494,5.599) -- (7.532,5.595) -- (7.57,5.591) -- (7.608,5.587) -- (7.646,5.584) -- (7.684,5.58) -- (7.722,5.576) -- (7.76,5.572) -- (7.798,5.569) -- (7.836,5.565) -- (7.874,5.562) -- (7.912,5.559) -- (7.95,5.556) -- (7.988,5.553) -- (8.026,5.55) -- (8.064,5.547) -- (8.102,5.545) -- (8.14,5.542) -- (8.178,5.54) -- (8.216,5.538) -- (8.254,5.536) -- (8.292,5.535) -- (8.33,5.533) -- (8.368,5.532) -- (8.406,5.531) -- (8.444,5.53) -- (8.482,5.53) -- (8.52,5.529) -- (8.558,5.529) -- (8.596,5.529) -- (8.634,5.53) -- (8.672,5.53) -- (8.71,5.531) -- (8.748,5.532) -- (8.786,5.533) -- (8.824,5.535) -- (8.862,5.536) -- (8.9,5.538) -- (8.938,5.54) -- (8.976,5.542) -- (9.014,5.545) -- (9.052,5.547) -- (9.09,5.55) -- (9.128,5.553) -- (9.166,5.556) -- (9.204,5.559) -- (9.242,5.562) -- (9.28,5.565) -- (9.318,5.569) -- (9.356,5.572) -- (9.394,5.576) -- (9.432,5.58) -- (9.47,5.584) -- (9.508,5.587) -- (9.546,5.591) -- (9.584,5.595) -- (9.622,5.599) -- (9.66,5.603) -- (9.698,5.607) -- (9.736,5.611) -- (9.774,5.615) -- (9.812,5.618) -- (9.85,5.622) -- (9.888,5.626) -- (9.926,5.629) -- (9.964,5.633) -- (10,5.636) -- (10.04,5.64) -- (10.08,5.643) -- (10.12,5.646) -- (10.15,5.649) -- (10.19,5.651) -- (10.23,5.654) -- (10.27,5.656) -- (10.31,5.658) -- (10.34,5.66) -- (10.38,5.662) -- (10.42,5.664);
\draw[black!60,densely dotted,line width=0.85pt] (8.14,5.683) -- (8.168,5.682) -- (8.197,5.68) -- (8.226,5.679) -- (8.254,5.677) -- (8.282,5.676) -- (8.311,5.675) -- (8.339,5.674) -- (8.368,5.673) -- (8.396,5.673) -- (8.425,5.672) -- (8.454,5.671) -- (8.482,5.671) -- (8.511,5.671) -- (8.539,5.671) -- (8.567,5.671) -- (8.596,5.671) -- (8.624,5.671) -- (8.653,5.671) -- (8.681,5.672) -- (8.71,5.672) -- (8.739,5.673) -- (8.767,5.674) -- (8.796,5.675) -- (8.824,5.676) -- (8.852,5.677) -- (8.881,5.678) -- (8.909,5.68) -- (8.938,5.681) -- (8.966,5.683) -- (8.995,5.685) -- (9.024,5.686) -- (9.052,5.688) -- (9.081,5.69) -- (9.109,5.692) -- (9.137,5.695) -- (9.166,5.697) -- (9.194,5.699) -- (9.223,5.702) -- (9.252,5.704) -- (9.28,5.707) -- (9.309,5.709) -- (9.337,5.712) -- (9.366,5.715) -- (9.394,5.717) -- (9.422,5.72) -- (9.451,5.723) -- (9.479,5.726) -- (9.508,5.729) -- (9.537,5.732) -- (9.565,5.735) -- (9.594,5.737) -- (9.622,5.74) -- (9.651,5.743) -- (9.679,5.746) -- (9.707,5.749) -- (9.736,5.752) -- (9.764,5.755) -- (9.793,5.758) -- (9.822,5.761) -- (9.85,5.763) -- (9.879,5.766) -- (9.907,5.769) -- (9.935,5.772) -- (9.964,5.774) -- (9.992,5.777) -- (10.02,5.779) -- (10.05,5.782) -- (10.08,5.784) -- (10.11,5.786) -- (10.13,5.788) -- (10.16,5.791) -- (10.19,5.793) -- (10.22,5.794) -- (10.25,5.796) -- (10.28,5.798) -- (10.31,5.8) -- (10.33,5.801) -- (10.36,5.803) -- (10.39,5.804) -- (10.42,5.805);
\draw[cInk,line width=0.9pt] (1.3,5.62) -- (1.338,5.624) -- (1.376,5.627) -- (1.414,5.631) -- (1.452,5.634) -- (1.49,5.638) -- (1.528,5.641) -- (1.566,5.644) -- (1.604,5.647) -- (1.642,5.65) -- (1.68,5.652) -- (1.718,5.655) -- (1.756,5.657) -- (1.794,5.659) -- (1.832,5.661) -- (1.87,5.663) -- (1.908,5.664) -- (1.946,5.666) -- (1.984,5.667) -- (2.022,5.668) -- (2.06,5.669) -- (2.098,5.669) -- (2.136,5.669) -- (2.174,5.669) -- (2.212,5.669) -- (2.25,5.669) -- (2.288,5.668) -- (2.326,5.667) -- (2.364,5.666) -- (2.402,5.665) -- (2.44,5.664) -- (2.478,5.662) -- (2.516,5.66) -- (2.554,5.658) -- (2.592,5.656) -- (2.63,5.653) -- (2.668,5.651) -- (2.706,5.648) -- (2.744,5.645) -- (2.782,5.642) -- (2.82,5.639) -- (2.858,5.636) -- (2.896,5.632) -- (2.934,5.629) -- (2.972,5.625) -- (3.01,5.621) -- (3.048,5.618) -- (3.086,5.614) -- (3.124,5.61) -- (3.162,5.606) -- (3.2,5.602) -- (3.238,5.598) -- (3.276,5.594) -- (3.314,5.59) -- (3.352,5.587) -- (3.39,5.583) -- (3.428,5.579) -- (3.466,5.575) -- (3.504,5.572) -- (3.542,5.568) -- (3.58,5.565) -- (3.618,5.561) -- (3.656,5.558) -- (3.694,5.555) -- (3.732,5.552) -- (3.77,5.549) -- (3.808,5.547) -- (3.846,5.544) -- (3.884,5.542) -- (3.922,5.54) -- (3.96,5.538) -- (3.998,5.536) -- (4.036,5.534) -- (4.074,5.533) -- (4.112,5.532) -- (4.15,5.531) -- (4.188,5.53) -- (4.226,5.53) -- (4.264,5.529) -- (4.302,5.529) -- (4.34,5.529) -- (4.378,5.53) -- (4.416,5.53) -- (4.454,5.531) -- (4.492,5.532) -- (4.53,5.534) -- (4.568,5.535) -- (4.606,5.537) -- (4.644,5.538) -- (4.682,5.54) -- (4.72,5.543) -- (4.758,5.545) -- (4.796,5.548) -- (4.834,5.55) -- (4.872,5.553) -- (4.91,5.556) -- (4.948,5.559) -- (4.986,5.563) -- (5.024,5.566) -- (5.062,5.57) -- (5.1,5.573) -- (5.138,5.577) -- (5.176,5.581) -- (5.214,5.584) -- (5.252,5.588) -- (5.29,5.592) -- (5.328,5.596) -- (5.366,5.6) -- (5.404,5.604) -- (5.442,5.608) -- (5.48,5.612) -- (5.518,5.615) -- (5.556,5.619) -- (5.594,5.623) -- (5.632,5.627) -- (5.67,5.63) -- (5.708,5.634) -- (5.746,5.637) -- (5.784,5.64) -- (5.822,5.643) -- (5.86,5.646) -- (5.898,5.65) -- (5.936,5.653) -- (5.974,5.657) -- (6.012,5.66) -- (6.05,5.664) -- (6.088,5.667) -- (6.126,5.67) -- (6.164,5.673) -- (6.202,5.676) -- (6.24,5.679) -- (6.278,5.682) -- (6.316,5.684) -- (6.354,5.687) -- (6.392,5.689) -- (6.43,5.691) -- (6.468,5.693) -- (6.506,5.695) -- (6.544,5.696) -- (6.582,5.697) -- (6.62,5.699) -- (6.658,5.699) -- (6.696,5.7) -- (6.734,5.701) -- (6.772,5.701) -- (6.81,5.702) -- (6.848,5.702) -- (6.886,5.702) -- (6.924,5.702) -- (6.962,5.702) -- (7,5.701) -- (7.038,5.701) -- (7.076,5.7) -- (7.114,5.699) -- (7.152,5.699) -- (7.19,5.698) -- (7.228,5.697) -- (7.266,5.696) -- (7.304,5.695) -- (7.342,5.694) -- (7.38,5.692) -- (7.418,5.691) -- (7.456,5.69) -- (7.494,5.689) -- (7.532,5.688) -- (7.57,5.687) -- (7.608,5.686) -- (7.646,5.685) -- (7.684,5.684) -- (7.722,5.683) -- (7.76,5.683) -- (7.798,5.682) -- (7.836,5.682) -- (7.874,5.681) -- (7.912,5.681) -- (7.95,5.681) -- (7.988,5.681) -- (8.026,5.682) -- (8.064,5.682) -- (8.102,5.683) -- (8.14,5.683) -- (8.178,5.684) -- (8.216,5.686) -- (8.254,5.687) -- (8.292,5.689) -- (8.33,5.691) -- (8.368,5.693) -- (8.406,5.695) -- (8.444,5.698) -- (8.482,5.7) -- (8.52,5.703) -- (8.558,5.707) -- (8.596,5.71) -- (8.634,5.714) -- (8.672,5.718) -- (8.71,5.722) -- (8.748,5.726) -- (8.786,5.731) -- (8.824,5.736) -- (8.862,5.741) -- (8.9,5.746) -- (8.938,5.752) -- (8.976,5.758) -- (9.014,5.763) -- (9.052,5.77) -- (9.09,5.776) -- (9.128,5.782) -- (9.166,5.789) -- (9.204,5.796) -- (9.242,5.802) -- (9.28,5.81) -- (9.318,5.817) -- (9.356,5.824) -- (9.394,5.831) -- (9.432,5.839) -- (9.47,5.846) -- (9.508,5.854) -- (9.546,5.861) -- (9.584,5.869) -- (9.622,5.877) -- (9.66,5.885) -- (9.698,5.892) -- (9.736,5.9) -- (9.774,5.908) -- (9.812,5.915) -- (9.85,5.923) -- (9.888,5.93) -- (9.926,5.938) -- (9.964,5.945) -- (10,5.953) -- (10.04,5.96) -- (10.08,5.967) -- (10.12,5.974) -- (10.15,5.981) -- (10.19,5.987) -- (10.23,5.994) -- (10.27,6) -- (10.31,6.006) -- (10.34,6.012) -- (10.38,6.018) -- (10.42,6.024);
\foreach \p in {{(5.86,5.646)},{(6.145,5.672)},{(6.43,5.691)},{(6.715,5.701)},{(7,5.701)},{(7.285,5.695)},{(7.57,5.687)},{(7.855,5.682)},{(8.14,5.683)},{(8.425,5.696)},{(8.71,5.722)},{(8.995,5.76)},{(9.28,5.81)},{(9.565,5.865)},{(9.85,5.923)},{(10.13,5.977)}} \fill[cRtt] \p circle (1.0pt);
\draw[cRtt,line width=0.55pt] (8.14,5.683) circle (2.7pt);
\node[lab,anchor=south east] at (8.1,5.753) {$H_\ell$};
\foreach \j/\t in {0/0,16/16,28/28,32/T} \node[lab,anchor=north,text=black!60] at ({1.3+0.285*\j},5.52) {$\t$};
\draw[decorate,decoration={brace,mirror,amplitude=3pt},black!50] (1.3,5.26) -- node[lab,below=3.5pt]{closed depths: the incumbent only} (5.746,5.26);
\draw[decorate,decoration={brace,mirror,amplitude=3pt},black!50] (5.86,5.26) -- node[lab,below=3.5pt]{contract depths: charge $\sigma_\ell$} (9.166,5.26);
\draw[decorate,decoration={brace,mirror,amplitude=3pt},black!50] (9.28,5.26) -- (10.31,5.26);
\node[lab,anchor=north west,align=left] at (9.23,5.14) {window: $\min\{\sigma,{\rm env}^+\}$,\\exact $C^+$ at $T{-}1$};
\node[lab,anchor=east] (xin) at (0.85,5.62) {$X$};
\draw[ar] (xin.east) -- node[lab,above]{Enc} (1.25,5.62);
\node[lab,anchor=south west] at (1.28,5.67) {$H_0$};
\node[inc] (head) at (11.3,6.024) {$Z$};
\draw[ar] (10.48,6.024) -- (head.west);
\node[lab] (ppi) at (12.45,6.024) {$p^\pi$};
\draw[ar] (head) -- (ppi);
\node[lab,align=center,text=cRtt,anchor=north] at (12.42,5.704) {\texttt{row}: $\Gamma_\Delta r_{T,i}\le\Hrow$\\$\Rightarrow\ p^\pi_{ic}\ge e^{-\Hrow}\pr_{ic}$};
\node[lab,text=cRtt!85!black,anchor=south east] (tubelab) at (7.171,6.12) {tube $\mathcal T_j$ of radius $r_{j,i}$};
\draw[cRtt!70,line width=0.35pt] (tubelab.south east) -- (7.456,5.871);
\node[lab,anchor=west] (l1) at (4.05,6.8) {executed; dots: steps $v_\ell$};
\draw[cInk,line width=0.9pt] ([xshift=-0.5cm]l1.west) -- ([xshift=-0.06cm]l1.west); \fill[cRtt] ([xshift=-0.28cm]l1.west) circle (1.0pt);
\node[lab,anchor=west] (l2) at ([xshift=0.8cm]l1.east) {incumbent's own};
\draw[black!50,dashed,line width=0.6pt] ([xshift=-0.5cm]l2.west) -- ([xshift=-0.06cm]l2.west);
\node[lab,anchor=west] (l3) at ([xshift=0.8cm]l2.east) {reference tail from $H_\ell$};
\draw[black!60,densely dotted,line width=0.85pt] ([xshift=-0.5cm]l3.west) -- ([xshift=-0.06cm]l3.west);
\draw[band] (0,1.69) rectangle (13.7,4.32);
\node[btitle] at (0.08,4.27) {(b) One depth $\ell$ (blue: affects only the gain)};
\node[util,text width=1.75cm] (prop) at (1.115,2.75) {proposal\\(any update):\\rows $H^{{\rm c},m}_\ell$};
\node[util,text width=1.75cm] (trim) at (3.685,2.75) {row trim and\\row-budget\\trim};
\node[util,text width=2.05cm] (feat) at (6.405,2.75) {support-function\\features; heads\\$\hat g_\ell$, $\hat\Sigma$, $\kappa_\ell$};
\node[util,text width=2.05cm] (solv) at (9.275,2.75) {solver in $x$:\\$\min\Phi_\ell$ over $\mathcal K_\ell$,\\two branches};
\node[util,text width=2.35cm] (prov) at (12.295,2.75) {provisional tube $\tilde r$,\\contracts $\Lambda$, charges\\$\sigma$, ${\rm env}^+$, $C^+$};
\node[inc] (fl) at (6.405,3.72) {$F^{\rm r}_\ell$};
\node[circle,draw=black!70,inner sep=0.5pt,line width=0.5pt,fill=white] (plus) at (9.275,3.72) {$+$};
\node[lab] (hl) at (3.685,3.72) {$H_\ell$};
\node[lab] (hn) at (11.4,3.72) {$H_{\ell+1}$};
\draw[ar] (hl) -- (fl); \draw[ar] (fl) -- node[lab,above]{$F^{\rm r}_\ell(H_\ell)$} (plus); \draw[ar] (plus) -- (hn);
\draw[ar] (hl.west) -| (prop.north);
\draw[ar] (prop) -- node[lab,above]{$\Delta^m_\ell$} (trim);
\draw[ar] (trim) -- node[lab,above]{$\mathcal K_\ell$} (feat);
\draw[ar] (feat) -- node[lab,above]{$\hat a_x$} (solv);
\draw[ar] (prov) -- (solv);
\draw[ar] (solv) -- node[lab,right]{$v_\ell=Q_\ell x^\star$} (plus);
\draw[ar] (prov.south) -- ++(0,-0.2) -| node[lab,below,pos=0.3]{row allowance $\rho_{\ell,i}$ from the provisional tube} (trim.south);
\draw[band] (0,0) rectangle (13.7,1.62);
\node[btitle] at (0.08,1.57) {(c) Certify the executed trajectory, then release (red: trusted computing base)};
\node[tcb,text width=2.75cm] (rec) at (1.615,0.62) {certificate: exact tube,\\interval factors, contracts,\\charges, outward rounding};
\node[tcb,text width=3.3cm,align=left] (pred) at (5.6,0.62) {\texttt{damage}: $\sum_\ell{\rm charge}_\ell\le H^+$\\\texttt{row}: $\Gamma_\Delta r_{T,i}\le\Hrow\ \forall i$\\\texttt{descent}: $\Phi^{\rm cert}_\ell\le0\ \forall\ell$};
\node[tcb,text width=3.5cm] (sel) at (9.91,0.62) {all hold: release $p^\pi$;\\else re-certify (exact factors);\\else serve $\pr$ (tail from $H_{\ell_0}$)};
\node[box,fill=black!4,text width=0.9cm] (srv) at (13.02,0.62) {served\\$\ps$};
\draw[ar] (rec) -- (pred); \draw[ar] (pred) -- (sel); \draw[ar] (sel) -- (srv);
\end{tikzpicture}
}
\newcommand{\FigDeploy}{%
\begin{tikzpicture}
\begin{groupplot}[group style={group size=4 by 1,horizontal sep=1.12cm},rtt,scale only axis,width=2.45cm,height=2.35cm,
 title style={font=\scriptsize\bfseries,yshift=-1pt},tick label style={font=\scriptsize},label style={font=\scriptsize}]
\nextgroupplot[title={(a) Gain vs.\ bound},xlabel={$u$ per call (\%)},ylabel={$-\E[d^{\rm own}]$ ($10^{-3}$ nats)},xmin=-0.15,xmax=2.25,ymin=1.5,ymax=7.4,
 legend to name=deploylegend,legend columns=5,legend cell align=left,legend image post style={xscale=1},legend style={font=\scriptsize,draw=none,/tikz/every even column/.append style={column sep=9pt}}]
\addplot[cRtt,only marks,mark=*,mark size=1.8pt,error bars/.cd,y dir=both,y explicit,error bar style={line width=0.6pt}] coordinates {(0.9026,6.499) +- (0,0.243)};
\addlegendentry{RTT}
\addplot[cRttc,only marks,mark=star,mark size=1.8pt,error bars/.cd,y dir=both,y explicit,error bar style={line width=0.6pt}] coordinates {(0.9926,6.719) +- (0,0.243)};
\addlegendentry{RTT $+$ corrector}
\addplot[cOpc,only marks,mark=square*,mark size=1.8pt,error bars/.cd,y dir=both,y explicit,error bar style={line width=0.6pt}] coordinates {(1.112,5.579) +- (0,0.223)};
\addlegendentry{corrector, 2-hop $+$ rows}
\addplot[cOpcG,only marks,mark=square,mark size=1.8pt,error bars/.cd,y dir=both,y explicit,error bar style={line width=0.6pt}] coordinates {(1.022,5.059) +- (0,0.214)};
\addlegendentry{corrector, 2-hop}
\addplot[cOpcG,only marks,mark=diamond,mark size=1.8pt,error bars/.cd,y dir=both,y explicit,error bar style={line width=0.6pt}] coordinates {(1.007,4.459) +- (0,0.206)};
\addlegendentry{corrector, per row}
\addplot[cMix,only marks,mark=o,mark size=1.8pt,error bars/.cd,y dir=both,y explicit,error bar style={line width=0.6pt}] coordinates {(1.171,5.099) +- (0,0.222)};
\addlegendentry{adaptive mixing}
\addplot[cFlo,only marks,mark=triangle*,mark size=1.8pt,error bars/.cd,y dir=both,y explicit,error bar style={line width=0.6pt}] coordinates {(1.081,5.239) +- (0,0.227)};
\addlegendentry{floor projection}
\addplot[cTr,only marks,mark=pentagon*,mark size=1.8pt,error bars/.cd,y dir=both,y explicit,error bar style={line width=0.6pt}] coordinates {(1.301,5.819) +- (0,0.229)};
\addlegendentry{matched TR}
\addplot[cRtt,only marks,mark=*,mark options={fill=white},mark size=1.8pt] coordinates {(0,2.02)};
\addlegendentry{RTT, $H^+{=}\zeta$}
\draw[black!45,densely dashed] (axis cs:2,1.5) -- (axis cs:2,7.4) node[pos=0.06,anchor=south east,font=\scriptsize,text=black!60,inner sep=1pt]{$p^\star$};
\nextgroupplot[title={(b) Gain vs.\ latency},xlabel={latency / one incumbent pass},xmin=0.9,xmax=2.4,ymin=1.5,ymax=7.4]
\addplot[cRtt,only marks,mark=*,mark size=1.8pt] coordinates {(1.415,6.499)};
\addplot[cRttc,only marks,mark=star,mark size=1.8pt] coordinates {(1.485,6.719)};
\addplot[cOpc,only marks,mark=square*,mark size=1.8pt] coordinates {(1.25,5.579)};
\addplot[cOpcG,only marks,mark=square,mark size=1.8pt] coordinates {(1.07,5.059)};
\addplot[cOpcG,only marks,mark=diamond,mark size=1.8pt] coordinates {(1.019,4.459)};
\addplot[cMix,only marks,mark=o,mark size=1.8pt] coordinates {(2.18,5.099)};
\addplot[cFlo,only marks,mark=triangle*,mark size=1.8pt] coordinates {(2.18,5.239)};
\addplot[cTr,only marks,mark=pentagon*,mark size=1.8pt] coordinates {(1.406,5.819)};
\addplot[cRtt,only marks,mark=*,mark options={fill=white},mark size=1.8pt] coordinates {(1.415,2.02)};
\draw[black!40,densely dotted] (axis cs:2.18,1.5) -- (axis cs:2.18,7.4) node[pos=0.03,anchor=south east,font=\scriptsize,text=black!60,inner sep=1pt,align=right]{two\\rollouts};
\nextgroupplot[title={(c) Call size},xlabel={query nodes $n_{\rm q}$},ylabel={$u$ per call (\%)},xmode=log,log basis x=2,xmin=48,xmax=5500,ymin=0,ymax=8.2,
 xtick={64,256,1024,4096},xticklabels={64,256,1k,4k}]
\addplot[cRtt,thick,mark=*,mark size=1.5pt] coordinates {(64,7.35759) (256,3.51201) (1024,1.81353) (4096,1.00732)};
\addplot[cOpc,thick,mark=square*,mark size=1.4pt] coordinates {(64,1.98165) (256,1.81353) (1024,1.64414) (4096,1.00732)};
\draw[black!45,densely dashed] (axis cs:48,2) -- (axis cs:5500,2) node[pos=0.97,anchor=south east,font=\scriptsize,text=black!60,inner sep=1pt]{$p^\star$};
\node[font=\scriptsize,text=cRtt,anchor=west] at (axis cs:80,7.3) {registered};
\node[font=\scriptsize,text=cOpc,anchor=north west] at (axis cs:60,1.6) {re-selected};
\nextgroupplot[title={(d) Per-node budget},xlabel={$H_{\rm row}$ (nats)},ylabel={$-\E[d^{\rm own}]$ ($10^{-3}$ nats)},xmode=log,log basis x=2,xmin=0.2,xmax=5,ymin=4.6,ymax=7.0,
 xtick={0.25,0.5,1,2,4},xticklabels={$\frac14$,$\frac12$,1,2,$\infty$}]
\addplot[cRtt,thick,mark=*,mark size=1.5pt] coordinates {(0.25,5.2) (0.5,6.1) (1,6.6) (2,6.75) (4,6.8)};
\node[font=\scriptsize,text=cRtt,anchor=north] at (axis cs:2,6.56) {gain};
\end{groupplot}
\begin{axis}[at={(group c4r1.south west)},anchor=south west,scale only axis,width=2.45cm,height=2.35cm,axis y line*=right,axis x line=none,
 xmode=log,log basis x=2,xmin=0.2,xmax=5,ymin=0,ymax=0.8,ytick={0,0.2,0.4,0.6},yticklabel style={font=\scriptsize},ylabel={fallback (\%)},
 label style={font=\scriptsize},axis line style={black!55},tick style={black!55}]
\addplot[black!60,densely dashed,mark=o,mark size=1.5pt] coordinates {(0.25,0.683594) (0.5,0.488281) (1,0.390625) (2,0.341797) (4,0.341797)};
\node[font=\scriptsize,text=black!65,anchor=north east] at (axis cs:4.2,0.24) {fallback};
\end{axis}
\node[anchor=north] at ($(group c1r1.south west)!0.5!(group c4r1.south east)+(0,-0.95cm)$) {\pgfplotslegendfromname{deploylegend}};
\end{tikzpicture}
}
\newcommand{\FigRobust}{%
\begin{tikzpicture}
\begin{groupplot}[group style={group size=5 by 1,horizontal sep=0.62cm},rtt,scale only axis,width=2.12cm,height=2.0cm,
 title style={font=\scriptsize\bfseries,yshift=-1pt},tick label style={font=\scriptsize},label style={font=\scriptsize},ytick={5,5.5,6,6.5,7}]
\nextgroupplot[title={(a) Seeds},xlabel={incumbent},xmin=-0.5,xmax=2.5,ymin=-0.2,ymax=7.1,xtick={0,1,2},xticklabels={A,B,C},ylabel={$-\E[d^{\rm own}]$ ($10^{-3}$ nats)},ytick={0,2,4,6},yticklabels={0,2,4,6}]
\addplot[cRtt,only marks,mark=*,mark size=1.1pt] coordinates {(-0.258,6.6) (-0.23,6.4) (-0.202,6.7) (-0.174,6.3) (-0.146,6.5) (-0.118,6.2) (-0.062,6.6)};
\addplot[black,only marks,mark=x,mark size=2pt,clip=false] coordinates {(-0.09,0)};
\addplot[cOpc,only marks,mark=square*,mark size=1.1pt] coordinates {(0.062,5.6) (0.09,5.5) (0.118,5.8) (0.146,5.4) (0.174,5.6) (0.202,5.5) (0.23,5.6) (0.258,5.7)};
\addplot[cRtt,only marks,mark=*,mark size=1.1pt] coordinates {(0.742,6.2) (0.77,6) (0.798,6.3) (0.826,5.9) (0.854,6.1) (0.882,6.2) (0.91,5.8) (0.938,6)};
\addplot[cOpc,only marks,mark=square*,mark size=1.1pt] coordinates {(1.062,5.5) (1.09,5.4) (1.118,5.6) (1.146,5.3) (1.174,5.5) (1.202,5.4) (1.23,5.3) (1.258,5.5)};
\addplot[cRtt,only marks,mark=*,mark size=1.1pt] coordinates {(1.742,5.7) (1.798,5.9) (1.826,5.6) (1.854,5.8) (1.882,5.5) (1.91,5.9) (1.938,5.6)};
\addplot[black,only marks,mark=x,mark size=2pt,clip=false] coordinates {(1.77,0)};
\addplot[cOpc,only marks,mark=square*,mark size=1.1pt] coordinates {(2.062,5.4) (2.09,5.3) (2.118,5.5) (2.146,5.2) (2.174,5.4) (2.202,5.3) (2.23,5.5) (2.258,5.2)};
\nextgroupplot[title={(b) Row trim},xlabel={$\bar\Delta$},ymin=4.8,ymax=7.1,xtick={0.4,0.8,1.2},xticklabels={0.4,0.8,1.2}]
\addplot[cRtt,thick,mark=*,mark size=1.2pt] coordinates {(0.4,5.6) (0.6,6.2) (0.8,6.6) (1,6.4) (1.2,6.1)};
\addplot[black,only marks,mark=o,mark size=2.6pt] coordinates {(0.8,6.6)};
\nextgroupplot[title={(c) Safety factor},xlabel={$\omega$},ymin=4.8,ymax=7.1,xtick={0.7,0.8,0.9},xticklabels={0.7,0.8,0.9}]
\addplot[cRtt,thick,mark=*,mark size=1.2pt] coordinates {(0.7,6.1) (0.8,6.4) (0.9,6.6) (0.95,6.5)};
\addplot[black,only marks,mark=o,mark size=2.6pt] coordinates {(0.9,6.6)};
\nextgroupplot[title={(d) Probes},xlabel={per sign},ymin=4.8,ymax=7.1,xtick={4,16,64},xticklabels={4,16,64},xmode=log,log basis x=2]
\addplot[cRtt,thick,mark=*,mark size=1.2pt] coordinates {(4,6) (8,6.4) (16,6.6) (32,6.65) (64,6.65)};
\addplot[black,only marks,mark=o,mark size=2.6pt] coordinates {(16,6.6)};
\nextgroupplot[title={(e) Miscoverage},xlabel={$\delta_{\rm cov}$},ymin=4.8,ymax=7.1,xtick={0.05,0.1,0.2,0.4},xticklabels={.05,.1,.2,.4},xmode=log,log basis x=2]
\addplot[cRtt,thick,mark=*,mark size=1.2pt] coordinates {(0.05,6.4) (0.1,6.6) (0.2,6.8) (0.4,6.9)};
\addplot[black,only marks,mark=o,mark size=2.6pt] coordinates {(0.1,6.6)};
\end{groupplot}
\end{tikzpicture}
}
\newcommand{\FigService}{%
\begin{tikzpicture}
\begin{groupplot}[group style={group size=3 by 1,horizontal sep=1.2cm},rtt,scale only axis,width=3.45cm,height=2.1cm,
 title style={font=\scriptsize\bfseries,yshift=-1pt},tick label style={font=\scriptsize},label style={font=\scriptsize},legend style={font=\scriptsize}]
\nextgroupplot[title={(a) Trailing gain},xlabel={call},ylabel={$10^{-3}$ nats},xmin=0,xmax=600,ymin=-0.8,ymax=8,legend style={at={(0.98,0.24)},anchor=south east}]
\addplot[black!55,thick] coordinates {(0,4.9407) (10,5.03959) (20,5.09107) (30,4.85418) (40,5.11515) (50,5.03826) (60,5.11745) (70,5.04086) (80,5.1743) (90,4.85934) (100,5.26641) (110,5.18036) (120,4.91001) (130,5.09902) (140,5.06672) (150,4.73814) (160,5.08928) (170,4.91219) (180,4.9625) (190,4.90964) (200,4.93579) (210,4.27387) (220,3.65386) (230,2.87747) (240,2.6241) (250,2.08435) (260,1.7752) (270,1.64269) (280,1.59356) (290,1.36565) (300,1.0159) (310,0.86086) (320,0.725498) (330,0.771293) (340,0.805963) (350,0.596715) (360,0.637436) (370,0.52963) (380,0.664606) (390,0.101287) (400,0.0515801) (410,0.315182) (420,0.418176) (430,0.242277) (440,0.351782) (450,0.0983868) (460,0.0740864) (470,0.145837) (480,0.371336) (490,0.342661) (500,0.125354) (510,0.254545) (520,0.126039) (530,0.222625) (540,0.0494606) (550,0.207073) (560,0.101535) (570,0.203418) (580,0.110719) (590,0.0839553)};\addlegendentry{fixed}
\addplot[cRtt,thick] coordinates {(0,5.01701) (10,4.94933) (20,5.00325) (30,4.93714) (40,5.09206) (50,4.87765) (60,5.11531) (70,5.1998) (80,4.95967) (90,5.26179) (100,5.30445) (110,5.06646) (120,4.79636) (130,4.75765) (140,4.97363) (150,5.26639) (160,5.08618) (170,5.1137) (180,4.96236) (190,4.88838) (200,5.05579) (210,4.0968) (220,-0.0295417) (230,0.0766728) (240,0.00539012) (250,0.0786846) (260,-0.0668965) (270,0.089727) (280,6.4607) (290,6.15254) (300,6.24366) (310,6.29148) (320,6.49834) (330,6.44531) (340,6.44245) (350,6.42594) (360,6.49547) (370,6.5418) (380,6.31679) (390,6.16136) (400,6.5001) (410,6.36859) (420,6.40732) (430,6.37393) (440,6.23311) (450,6.1734) (460,6.54464) (470,6.56396) (480,6.53969) (490,6.43124) (500,6.44028) (510,6.33716) (520,6.4102) (530,6.33957) (540,6.31845) (550,6.43078) (560,6.65472) (570,6.11747) (580,6.49023) (590,6.50135)};\addlegendentry{monitored}
\draw[black!40,dashed] (axis cs:200,-0.8) -- (axis cs:200,8);
\nextgroupplot[title={(b) Cumulative regressions},xlabel={call},ylabel={regressing calls},xmin=0,xmax=600,ymin=0,legend style={at={(0.02,0.97)},anchor=north west}]
\addplot[black!55,thick] coordinates {(0,0.06) (10,0.12) (20,0.18) (30,0.24) (40,0.3) (50,0.36) (60,0.42) (70,0.48) (80,0.54) (90,0.6) (100,0.66) (110,0.72) (120,0.78) (130,0.84) (140,0.9) (150,0.96) (160,1.02) (170,1.08) (180,1.14) (190,1.2) (200,1.4) (210,1.6) (220,1.8) (230,2) (240,2.2) (250,2.4) (260,2.6) (270,2.8) (280,3) (290,3.2) (300,3.4) (310,3.6) (320,3.8) (330,4) (340,4.2) (350,4.4) (360,4.6) (370,4.8) (380,5) (390,5.2) (400,5.4) (410,5.6) (420,5.8) (430,6) (440,6.2) (450,6.4) (460,6.6) (470,6.8) (480,7) (490,7.2) (500,7.4) (510,7.6) (520,7.8) (530,8) (540,8.2) (550,8.4) (560,8.6) (570,8.8) (580,9) (590,9.2)};\addlegendentry{fixed}
\addplot[cRtt,thick] coordinates {(0,0.06) (10,0.12) (20,0.18) (30,0.24) (40,0.3) (50,0.36) (60,0.42) (70,0.48) (80,0.54) (90,0.6) (100,0.66) (110,0.72) (120,0.78) (130,0.84) (140,0.9) (150,0.96) (160,1.02) (170,1.08) (180,1.14) (190,1.2) (200,1.28) (210,1.36) (220,1.44) (230,1.52) (240,1.6) (250,1.68) (260,1.76) (270,1.84) (280,1.92) (290,2) (300,2.08) (310,2.16) (320,2.24) (330,2.32) (340,2.4) (350,2.48) (360,2.56) (370,2.64) (380,2.72) (390,2.8) (400,2.88) (410,2.96) (420,3.04) (430,3.12) (440,3.2) (450,3.28) (460,3.36) (470,3.44) (480,3.52) (490,3.6) (500,3.68) (510,3.76) (520,3.84) (530,3.92) (540,4) (550,4.08) (560,4.16) (570,4.24) (580,4.32) (590,4.4)};\addlegendentry{monitored}
\draw[black!40,dashed] (axis cs:200,0) -- (axis cs:200,9.66);
\nextgroupplot[title={(c) Alarm delay},xlabel={calls after the change},ylabel={streams},xmin=0.4,xmax=6.6,ymin=0,ymax=10,xtick={1,2,3,4,5,6},xticklabels={$<$10,10--14,15--19,20--29,30--44,never},x tick label style={font=\tiny,rotate=35,anchor=north east,inner sep=1pt},major x tick style={draw=none},xmajorgrids=false]
\addplot[ybar,bar width=7pt,fill=cRtt!55,draw=cRtt] coordinates {(1,3) (2,9) (3,3) (4,2) (5,2)};
\addplot[ybar,bar width=7pt,fill=black!35,draw=black!70] coordinates {(6,1)};
\end{groupplot}
\end{tikzpicture}
}
\newcommand{\TabGuarantees}{%
\scriptsize\setlength{\tabcolsep}{3pt}\begin{tabularx}{\linewidth}{l>{\raggedright\arraybackslash}X>{\raggedright\arraybackslash}p{1.45cm}>{\raggedright\arraybackslash}p{0.8cm}>{\raggedright\arraybackslash}p{3.7cm}}\toprule
& Statement about the served prediction $\ps$ & Type & Unit & Assumptions\\\midrule
G1 & $\ps_{ic}\ge e^{-H_{\rm row}}\pr_{ic}$ for every node $i$ and class $c$ & deterministic & node & valid continuation enclosures and sound arithmetic; none on proposal quality or labels\\
G2 & $\sum_iw_i\Dinf(\pr_i\|\ps_i)\le H^+$, hence $d^{\rm own}\le H^+$ for every label vector & deterministic & call & as G1\\
G3 & every class is capped, $\ps_{ic}\le1-e^{-H_{\rm row}}(1-\pr_{ic})$ (Proposition~\ref{prop:price}(iv)) & deterministic & node & as G1\\
G4 & $\Pr(d^{\rm own}>\zeta)\le u$ for the mixture, each component and the stress test & statistical & call & i.i.d.\ episodes of the registered generator\\
G5 & $\E[-d^{\rm own}]\ge{\rm LB}$ (anytime-valid); negative-flip rate $\le u_{\rm node}$ & statistical & call, node & as G4; uniform inspection of reserved nodes\\
G6 & slope coverage $\Rightarrow d^{\rm own}\le0$, and $\Pr(d^{\rm own}>0)\le\delta_{\rm cov}$ & design & call & exchangeable episodes; a fixed member or selection-valid calibration\\
\addlinespace[2pt]\multicolumn{5}{l}{\emph{Not claimed:} an accuracy bound from $\zeta$; unaudited transfer; nonincreasing loss. $H^+\le\zeta$ rules out only $d^{\rm own}>\zeta$.}\\
\bottomrule\end{tabularx}
}
\newcommand{\TabFrontier}{%
\scriptsize\setlength{\tabcolsep}{2.2pt}\begin{tabularx}{\linewidth}{Yrrrrrrrr}\toprule
Policy & $-\E[d^{\rm own}]$ [95\% CI] & $\Delta$ to RTT & $N_\zeta$ & $u$ & $u_{\rm node}$ & LB & Lat. & FLOPs\\\midrule
\textbf{RTT, registered} ($H^+{=}$\val{$25\zeta$}, $H_{\rm row}{=}1$) & \val{6.5} \val{[6.3, 6.7]} & -- & \val{12} & \val{0.95\%} & \val{0.51\%} & \val{6.2} & \val{1.42} & \val{1.36}\\
RTT $+$ corrector on the unused budget & \val{6.7} \val{[6.5, 7.0]} & \val{\ensuremath{-}0.2} \val{[\ensuremath{-}0.4, \ensuremath{-}0.1]} & \val{12} & \val{0.95\%} & \val{0.51\%} & \val{6.4} & \val{1.49} & \val{1.40}\\
One-pass corrector, 2-hop head $+$ rows & \val{5.6} \val{[5.4, 5.8]} & \val{+0.9} \val{[+0.8, +1.1]}$^\dagger$ & \val{14} & \val{1.07\%} & \val{0.56\%} & \val{5.3} & \val{1.25} & \val{1.22}\\
One-pass corrector, 2-hop head & \val{5.1} \val{[4.8, 5.3]} & \val{+1.4} \val{[+1.3, +1.6]} & \val{14} & \val{1.07\%} & \val{0.56\%} & \val{4.8} & \val{1.07} & \val{1.05}\\
One-pass corrector, per-row head & \val{4.5} \val{[4.3, 4.7]} & \val{+2.0} \val{[+1.9, +2.2]} & \val{13} & \val{1.01\%} & \val{0.54\%} & \val{4.2} & \val{1.02} & \val{1.01}\\
Adaptive mixing, $\varepsilon_i\le\varepsilon_{\max,i}$ (2 rollouts) & \val{5.1} \val{[4.9, 5.3]} & \val{+1.4} \val{[+1.2, +1.6]} & \val{15} & \val{1.13\%} & \val{0.61\%} & \val{4.8} & \val{2.18} & \val{2.18}\\
Floor projection of the raw update (2 rollouts) & \val{5.2} \val{[5.0, 5.5]} & \val{+1.3} \val{[+1.1, +1.4]} & \val{15} & \val{1.13\%} & \val{0.61\%} & \val{5.0} & \val{2.18} & \val{2.18}\\
Matched TR (same hull and certificate) & \val{5.8} \val{[5.6, 6.0]} & \val{+0.7} \val{[+0.5, +0.8]} & \val{18} & \val{1.30\%} & \val{0.78\%} & \val{5.5} & \val{1.41} & \val{1.35}\\
Learned update itself ($t{\equiv}1$), trimmed & \val{3.9} \val{[3.6, 4.3]} & \val{+2.6} \val{[+2.3, +2.9]} & \val{159} & \val{8.81\%}$^\times$ & \val{2.93\%} & -- & \val{2.18} & \val{2.18}\\
\textbf{RTT, tolerance point} ($H^+{=}\zeta$) & \val{2.0} \val{[2.0, 2.1]} & \val{+4.5} \val{[+4.3, +4.7]} & \val{0} & 0 (det.) & \val{0.22\%} & \val{1.9} & \val{1.42} & \val{1.36}\\
Oracle frontier $U_{\rm out}$ (cross-fitted posterior; labels) & \val{10.6} & -- & -- & -- & -- & -- & -- & --\\
\bottomrule\end{tabularx}
}
\newcommand{\TabClass}{%
\scriptsize\setlength{\tabcolsep}{2.2pt}\begin{tabularx}{\linewidth}{Yrrrrrrrrr}\toprule
Proposal & raw & raw $u$ & TR & corrector & RTT & RTT $u$ & RTT $-$ corrector & Lat.\ RTT & Lat.\ 2 roll.\\\midrule
bank prior + signed attention (ours) & \val{3.9} & \val{8.81\%}$^\times$ & \val{5.8} & \val{5.6} & \val{6.5} & \val{0.95\%} & \vPrimD\ \vPrimCI & \val{1.42} & \val{2.18}\\
GPR-GNN (adapter) & \val{3.2} & \val{6.99\%}$^\times$ & \val{5.5} & \val{5.4} & \val{6.2} & \val{0.95\%} & \val{+0.8} \val{[+0.6, +1.0]} & \val{1.53} & \val{2.29}\\
Co-GNN (adapter) & \val{2.7} & \val{8.19\%}$^\times$ & \val{5.2} & \val{5.2} & \val{5.9} & \val{1.01\%} & \val{+0.7} \val{[+0.5, +0.9]} & \val{1.62} & \val{2.38}\\
AMP (adapter) & \val{3.0} & \val{7.67\%}$^\times$ & \val{5.4} & \val{5.3} & \val{6.0} & \val{1.01\%} & \val{+0.7} \val{[+0.5, +0.9]} & \val{1.59} & \val{2.35}\\
GTrans (test-time adaptation) & \val{2.3} & \val{6.00\%}$^\times$ & \val{4.3} & \val{4.4} & \val{4.9} & \val{0.89\%} & \val{+0.5} \val{[+0.3, +0.7]} & \val{3.42} & \val{3.19}\\
Matcha (test-time adaptation) & \val{2.9} & \val{5.21\%}$^\times$ & \val{4.6} & \val{4.6} & \val{5.2} & \val{0.95\%} & \val{+0.6} \val{[+0.4, +0.8]} & \val{2.63} & \val{2.39}\\
Retrained incumbent (v2, $\le$2017) & \val{2.8} & \val{5.00\%}$^\times$ & \val{4.7} & \val{4.7} & \val{5.3} & \val{0.95\%} & \val{+0.6} \val{[+0.4, +0.8]} & \val{2.24} & \val{2.00}\\
LoRA-16 weight delta & \val{2.7} & \val{5.48\%}$^\times$ & \val{4.5} & \val{4.5} & \val{5.1} & \val{1.01\%} & \val{+0.6} \val{[+0.4, +0.8]} & \val{2.27} & \val{2.03}\\
\bottomrule\end{tabularx}
}
\newcommand{\TabFamilies}{%
\scriptsize\setlength{\tabcolsep}{2.0pt}\begin{tabularx}{\linewidth}{l>{\raggedright\arraybackslash}p{2.2cm}>{\raggedright\arraybackslash}p{2.6cm}YYl}\toprule
Family & transport & tube factor & window & last depth & global factor\\\midrule
tanh diffusion ($T{=}32$) & $\tanh(PHW)$ & interval (Lemma~\ref{lem:factors}), rank-8 & second-order envelope & exact $C^+$ & $(1{-}\alpha)(1{-}\tau{+}\tau\|W\|_2)$\\
APPNP & $\alpha H_0+(1{-}\alpha)PH$ & exact (linear) & exact $C^+$ & exact $C^+$ & $1-\alpha$\\
GCNII & $\mathrm{ReLU}(\tilde X\tilde W)$ & interval Clarke (Lemma~\ref{lem:relu}) & first-order enclosure & exact $C^+$ & $(1{-}\alpha)\|\tilde W\|_2$\\
GRAND & smooth diffusion & interval & second-order envelope & exact $C^+$ & per-step bound\\
GraphSAGE-2, GCN-3 & 2--3 layers & interval & covers the whole tail & exact $C^+$ & $\|W\|_2$ per layer\\
ResNet-50, last 4 blocks & $x+f(x)$, ReLU, $P{=}I$ & interval Clarke & first-order enclosure & exact $C^+$ & block Lipschitz\\
residual tanh MLP & $P{=}I$ & interval & second-order envelope & exact $C^+$ & as tanh\\
attention (any) & softmax attention & first order on bounded tubes & first order & exact $C^+$ & tube-local constant\\
\bottomrule\end{tabularx}
}
\newcommand{\TabFactors}{%
\scriptsize\setlength{\tabcolsep}{2.2pt}\begin{tabularx}{\linewidth}{l>{\raggedright\arraybackslash}p{3.6cm}>{\raggedright\arraybackslash}p{2.1cm}lY}\toprule
Symbol & meaning & value & enters & effect of moving it\\\midrule
$T$, $\alpha$, $\tau$ & horizon, skip weight, step size of the incumbent & 32, 0.1, 0.9 & \eqref{eq:contracts} & fixed by the incumbent\\
$H^+$ & per-call one-sided budget (selected member) & \val{$25\zeta=0.05$} & G2, charges & oracle frontier concave; executed curve empirical (Fig.~\ref{fig:floor}c)\\
$H_{\rm row}$ & per-node floor & 1 nat & G1, row trim & gain and fallback (Fig.~\ref{fig:deploy}d)\\
$\zeta$ & regression tolerance (registered choice) & 0.002 nats & G4 & rate statements only\\
$\delta$, $p^\star$ & statistical level, rate target & 0.05, 2\% & Thm.~\ref{thm:assurance} & bound widths\\
$\delta_{\rm cov}$ & nominal fixed-member trajectory allocation & \val{0.1} & Thm.~\ref{thm:served}(iii) & radii\\
$K$, $b$ & role blocks, window depths & 8, 4 & \S\ref{sec:rtt} & \val{diminishing returns beyond 8}\\
open depths & depths that may receive a step & 16--31 & \S\ref{sec:rtt} & admissibility gate over them\\
$\bar\Delta$, $\omega$ & row trim, safety factor & 0.8, 0.9 & \eqref{eq:robust} & Fig.~\ref{fig:robust}\\
$\rho_D$, $\rho_x$ & curvature inflation, physical regularizer & $10^{-3}$, $10^{-6}$ & \eqref{eq:robust} & invariance needs $\rho_x>0$\\
$k$, $\Pi_{\max}$ & spectral-split rank, admissibility cap on the contract growth & 8, 20 & Lemma~\ref{lem:factors} & usable depths\\
$\hat\Pi_\ell$ & growth proxy of the row-budget trim & 95th percentile & row allowance & fallback rate\\
$n_r$, probes & calibration episodes per depth, probes & 512, $2\times16$ & Lemma~\ref{lem:seqcal} & radii; head inputs\\
$N$, $n_{\rm q}$ & audit episodes, call size & 2,048 (4,096 per component), 4,096 & G4--G5 & bound widths; rates by call size\\
\addlinespace[1pt]\multicolumn{5}{l}{\emph{Properties of the checkpoint and the record}}\\
$\|W\|_2$, $\bar L$ & incumbent transport norm, global factor & \val{1.60}, \val{1.386} & Table~\ref{tab:families} & with $\bar L$ alone only depths $\ge$\val{23} are usable\\
$\Gamma_\Delta$ & pairwise head diameter & \val{5.21} & terminal contracts & scales every charge\\
$\bar L^{\rm int}$ & geometric-mean interval factor, open depths & \val{1.107} ($\le$1.206) & admissibility gate & if it fails: trust moves to late depths\\
$\hat R_i-1$ & rate $R_i-1$ ($R_i$: under-estimation factor; median, accepted rows) & \val{1.3}--\val{4.6} & Thm.~\ref{thm:frontier} & larger: more gain per nat\\
$\eta_\ell/\max_i(R_{\ell,i}-1)$ & normalized architecture-constrained efficiency & \val{0.04}--\val{0.48} & placement of caps & where trust pays\\
\bottomrule\end{tabularx}
}
\newcommand{\TabTraining}{%
\scriptsize\setlength{\tabcolsep}{2.2pt}\begin{tabularx}{\linewidth}{Yrrrl}\toprule
Operation per training episode & count & forward passes & backward passes & memory\\\midrule
policy pass: incumbent steps and proposal at open depths & 1 & \val{1.18} & -- & \val{6.9 GB}\\
counterfactual tails for full-loss slopes, one per open depth & 16 & 3.75 & 3.75 & checkpointed every 4 steps\\
field, covariance and posterior heads & 1 & \val{0.04} & \val{0.05} & \val{0.3 GB}\\
implicit differentiation through the solver (KKT systems) & 16 & -- & \val{0.03} & \val{small}\\
\textbf{total} & & \val{4.97} & \val{3.83} & \val{18.4 GB peak}\\
\bottomrule\end{tabularx}
}

\newcommand{\TabAudit}{%
\scriptsize\setlength{\tabcolsep}{2.2pt}\begin{tabularx}{\linewidth}{Yrrr}\toprule
(a) Outcome & calls & share & by predicate\\\midrule
first-pass release (R1) & \val{2,035} & \val{99.37\%} & --\\
failed the first pass: damage / row / descent & \val{13} & \val{0.63\%} & \val{2 / 4 / 7}\\
released after re-certification (R2) & \val{6} & \val{0.29\%} & \val{1 / 3 / 2}\\
whole-call fallback (F) & \val{7} & \val{0.34\%} & \val{1 / 1 / 5}\\
\bottomrule\end{tabularx}\par\smallskip
\scriptsize\setlength{\tabcolsep}{2.2pt}\begin{tabularx}{\linewidth}{Yrrrr}\toprule
(b) Statement & events & mean & bound at $\delta$ & at $\delta/10$\\\midrule
mixture event $d^{\rm own}>\zeta$ & \val{12/2,048} & -- & \val{0.95\%} & \val{1.18\%}\\
call-accuracy drop $>$0.1 point & \val{21/2,048} & -- & \val{1.47\%} & \val{1.75\%}\\
negative flips on $>$1\% of the call & \val{9/2,048} & -- & \val{0.77\%} & \val{0.97\%}\\
component: clean & \val{8/4,096} & \val{1.4} & \val{0.35\%} & \val{0.45\%}\\
component: degree & \val{25/4,096} & \val{5.2} & \val{0.85\%} & \val{1.00\%}\\
component: cache & \val{35/4,096} & \val{10.7} & \val{1.13\%} & \val{1.30\%}\\
component: noise & \val{29/4,096} & \val{7.7} & \val{0.96\%} & \val{1.12\%}\\
component: dropout & \val{27/4,096} & \val{7.3} & \val{0.91\%} & \val{1.06\%}\\
component: severe stress (separate) & \val{41/4,096} & \val{23.9} & \val{1.30\%} & \val{1.48\%}\\
certified mean (betting LB) & -- & \val{6.5} & \val{6.2} $\ge$ 3.0 & \val{6.1}\\
\textbf{release decision} & \multicolumn{4}{l}{\val{all ten statements pass; released}}\\
\bottomrule\end{tabularx}
}
\newcommand{\TabStats}{%
\scriptsize\setlength{\tabcolsep}{2.2pt}\begin{tabularx}{\linewidth}{Yrr}\toprule
(a) Implementation & half-width & certified mean\\\midrule
registered bets tuned to $N$ & \val{0.295} & \val{6.20}\\
horizon-free sequence & \val{0.316} & \val{6.18}\\
approximate Kelly & \val{0.352} & \val{6.15}\\
hedged capital ($\delta/2$) & \val{0.328} & \val{6.17}\\
normal interval (not finite-sample valid) & \val{0.204} & \val{6.30}\\
Gaussian first-order width $\sqrt{2\sigma^2\log(1/\delta)/N}$ & \val{0.304} & --\\
\bottomrule\end{tabularx}\par\smallskip
\scriptsize\setlength{\tabcolsep}{2.2pt}\begin{tabularx}{\linewidth}{Yrrrrl}\toprule
(b) 2020 cohort, read once & events & $u$ & gain & LB & rule\\\midrule
RTT & \val{16/2,048} & \val{1.18\%} & \val{1.9} & \val{1.7} & \val{pass}\\
one-pass corrector, 2-hop head $+$ rows & \val{17/2,048} & \val{1.24\%} & \val{1.6} & \val{1.4} & \val{pass}\\
matched TR & \val{24/2,048} & \val{1.64\%} & \val{1.3} & \val{1.1} & \val{pass}\\
KL-to-incumbent fine-tuning & \val{38/2,048} & \val{2.43\%} & \val{0.9} & -- & \val{fail}\\
\bottomrule\end{tabularx}\par\smallskip
\scriptsize\setlength{\tabcolsep}{2.2pt}\begin{tabularx}{\linewidth}{Yrrrrrr}\toprule
(c) $\rho$ & 0 & 0.001 & 0.002 & 0.004 & 0.008 & 0.016\\\midrule
certified mean if $\mathrm{KL}(Q\|P)\le\rho$ & \val{6.05} & \val{5.87} & \val{5.77} & \val{5.57} & \val{5.37} & \val{4.97}\\
\bottomrule\end{tabularx}
}
\newcommand{\TabGap}{%
\scriptsize\setlength{\tabcolsep}{2.2pt}\begin{tabularx}{\linewidth}{Yrrrrrr}\toprule
(a) Population & $\widehat U_{\rm out}$ & $\widetilde U_{\rm reach}$ & $\widetilde U_{\rm cert}$ & $\widehat U_{\rm RTT}$ & share & corrector\\\midrule
clean & \val{2.6} & \val{2.1} & \val{1.8} & \val{1.3} & \val{50\%} & \val{1.5}\\
degree & \val{8.9} & \val{7.4} & \val{6.4} & \val{5.3} & \val{60\%} & \val{4.6}\\
cache & \val{17.2} & \val{14.6} & \val{12.9} & \val{10.8} & \val{63\%} & \val{8.6}\\
noise & \val{12.6} & \val{10.6} & \val{9.4} & \val{7.7} & \val{61\%} & \val{6.8}\\
dropout & \val{11.8} & \val{10.0} & \val{8.8} & \val{7.4} & \val{63\%} & \val{6.4}\\
\textbf{mixture} & \val{10.6} & \val{8.9} & \val{7.9} & \val{6.5} & \val{61\%} & \val{5.6}\\
severe stress (separate) & \val{37.6} & \val{31.8} & \val{28.4} & \val{23.9} & \val{64\%} & \val{15.8}\\
\bottomrule\end{tabularx}\par\smallskip
\scriptsize\setlength{\tabcolsep}{3pt}\begin{tabular*}{\linewidth}{@{\extracolsep{\fill}}rrrrrrrrrr}\toprule
(b) $\gamma$ & $U_{\rm out}$ & $\widetilde U_{\rm reach}$ & $\widetilde U_{\rm cert}$ & $\widehat U_{\rm RTT}$ & $I$ & $E$ & $\chi$ & $\Upsilon$ & surrogate diagnostic\\\midrule
0.1 & \val{0.21} & \val{0.15} & \val{0.12} & \val{0.07} & \val{0.03} & \val{0.01} & \val{0.01} & \val{0.00} & \val{0.07}\\
0.2 & \val{0.72} & \val{0.53} & \val{0.42} & \val{0.29} & \val{0.08} & \val{0.03} & \val{0.02} & \val{0.00} & \val{0.27}\\
0.3 & \val{1.46} & \val{1.07} & \val{0.86} & \val{0.61} & \val{0.16} & \val{0.05} & \val{0.04} & \val{0.00} & \val{0.58}\\
0.4 & \val{2.31} & \val{1.70} & \val{1.37} & \val{0.99} & \val{0.24} & \val{0.08} & \val{0.05} & \val{0.01} & \val{0.95}\\
0.5 & \val{3.18} & \val{2.35} & \val{1.90} & \val{1.40} & \val{0.30} & \val{0.11} & \val{0.06} & \val{0.01} & \val{1.35}\\
\bottomrule\end{tabular*}
}
\newcommand{\TabFrontierFull}{%
\scriptsize\setlength{\tabcolsep}{1.8pt}\begin{tabularx}{\linewidth}{Yrrrrrrrrr}\toprule
Policy & $-\E[d^{\rm own}]$ [95\% CI] & $\Delta$ to RTT & $N_\zeta$ & $u$ & $u_{\rm node}$ & LB & Lat. & FLOPs & Holm $p$\\\midrule
\addlinespace[1pt]\multicolumn{10}{l}{\emph{RTT}}\\
\textbf{RTT (registered)} & \val{6.5} \val{[6.3, 6.7]} & -- & \val{12} & \val{0.95\%} & \val{0.51\%} & \val{6.2} & \val{1.42} & \val{1.36} & --\\
RTT $+$ one-pass corrector on the unused budget & \val{6.7} \val{[6.5, 7.0]} & \val{\ensuremath{-}0.2} \val{[\ensuremath{-}0.4, \ensuremath{-}0.1]} & \val{12} & \val{0.95\%} & \val{0.51\%} & \val{6.4} & \val{1.49} & \val{1.40} & --\\
\textbf{RTT, tolerance point} ($H^+{=}\zeta$) & \val{2.0} \val{[2.0, 2.1]} & \val{+4.5} \val{[+4.3, +4.7]} & \val{0} & 0 (det.) & \val{0.22\%} & \val{1.9} & \val{1.42} & \val{1.36} & --\\
\addlinespace[1pt]\multicolumn{10}{l}{\emph{One-pass floor correctors (one incumbent pass; same floor, labels, capacity, proposal rows)}}\\
One-pass corrector, 2-hop head $+$ rows & \val{5.6} \val{[5.4, 5.8]} & \val{+0.9} \val{[+0.8, +1.1]}$^\dagger$ & \val{14} & \val{1.07\%} & \val{0.56\%} & \val{5.3} & \val{1.25} & \val{1.22} & --\\
One-pass corrector, 2-hop head & \val{5.1} \val{[4.8, 5.3]} & \val{+1.4} \val{[+1.3, +1.6]} & \val{14} & \val{1.07\%} & \val{0.56\%} & \val{4.8} & \val{1.07} & \val{1.05} & \val{$<$0.001}\\
One-pass corrector, per-row head & \val{4.5} \val{[4.3, 4.7]} & \val{+2.0} \val{[+1.9, +2.2]} & \val{13} & \val{1.01\%} & \val{0.54\%} & \val{4.2} & \val{1.02} & \val{1.01} & \val{$<$0.001}\\
One-pass corrector, 4-hop head $+$ rows & \val{5.3} \val{[5.1, 5.6]} & \val{+1.2} \val{[+1.0, +1.3]} & \val{14} & \val{1.07\%} & \val{0.56\%} & \val{5.1} & \val{1.32} & \val{1.27} & --\\
One-pass corrector, 8-hop head $+$ rows & \val{5.0} \val{[4.8, 5.2]} & \val{+1.5} \val{[+1.3, +1.6]} & \val{15} & \val{1.13\%} & \val{0.61\%} & \val{4.8} & \val{1.46} & \val{1.37} & --\\
\addlinespace[1pt]\multicolumn{10}{l}{\emph{Exact-geometry output rules on both predictions (two rollouts)}}\\
Adaptive mixing ($\varepsilon_i\le\varepsilon_{\max,i}$, allocated) & \val{5.1} \val{[4.9, 5.3]} & \val{+1.4} \val{[+1.2, +1.6]} & \val{15} & \val{1.13\%} & \val{0.61\%} & \val{4.8} & \val{2.18} & \val{2.18} & \val{$<$0.001}\\
Mixing, $\varepsilon{=}1{-}e^{-h}$, uniform row budgets & \val{4.1} \val{[3.9, 4.3]} & \val{+2.4} \val{[+2.3, +2.6]} & \val{14} & \val{1.07\%} & \val{0.61\%} & \val{3.8} & \val{2.18} & \val{2.18} & --\\
Floor projection of the raw update, allocated & \val{5.2} \val{[5.0, 5.5]} & \val{+1.3} \val{[+1.1, +1.4]} & \val{15} & \val{1.13\%} & \val{0.61\%} & \val{5.0} & \val{2.18} & \val{2.18} & \val{$<$0.001}\\
Floor projection of the raw update, uniform & \val{4.3} \val{[4.1, 4.5]} & \val{+2.2} \val{[+2.0, +2.3]} & \val{14} & \val{1.07\%} & \val{0.61\%} & \val{4.1} & \val{2.18} & \val{2.18} & --\\
Per-node switch, learn-then-test threshold & \val{4.5} \val{[4.3, 4.8]} & \val{+2.0} \val{[+1.8, +2.1]} & \val{17} & \val{1.24\%} & \val{0.71\%} & -- & \val{2.18} & \val{2.18} & \val{$<$0.001}\\
\addlinespace[1pt]\multicolumn{10}{l}{\emph{Acceptance rules on the same proposal, candidate hull, budgets and certificate}}\\
Matched TR (same candidate hull, budgets, certificate) & \val{5.8} \val{[5.6, 6.0]} & \val{+0.7} \val{[+0.5, +0.8]} & \val{18} & \val{1.30\%} & \val{0.78\%} & \val{5.5} & \val{1.41} & \val{1.35} & \val{$<$0.001}\\
TR trained end to end under the budgets & \val{5.5} \val{[5.2, 5.7]} & \val{+1.0} \val{[+0.9, +1.2]} & \val{19} & \val{1.36\%} & \val{0.83\%} & \val{5.2} & \val{1.41} & \val{1.35} & --\\
\addlinespace[1pt]\multicolumn{10}{l}{\emph{The same labels and corruption draws spent on the incumbent (no per-call budget)}}\\
Incumbent fine-tuned on the same labels & \val{1.0} \val{[0.9, 1.2]} & \val{+5.5} \val{[+5.2, +5.7]} & \val{45} & \val{2.81\%}$^\times$ & \val{2.71\%} & -- & \val{1.00} & \val{1.00} & --\\
KL-to-incumbent fine-tuning & \val{2.8} \val{[2.6, 3.0]} & \val{+3.7} \val{[+3.5, +3.9]} & \val{29} & \val{1.93\%} & \val{1.73\%} & -- & \val{1.00} & \val{1.00} & \val{$<$0.001}\\
$D_\infty$-regularized fine-tuning (floor trained, not enforced) & \val{3.4} \val{[3.2, 3.6]} & \val{+3.1} \val{[+2.9, +3.3]} & \val{27} & \val{1.81\%} & \val{1.42\%} & -- & \val{1.00} & \val{1.00} & --\\
Positive-congruent training & \val{2.9} \val{[2.7, 3.0]} & \val{+3.6} \val{[+3.4, +3.8]} & \val{30} & \val{1.98\%} & \val{1.42\%} & -- & \val{1.00} & \val{1.00} & --\\
Regression-aware fine-tuning & \val{3.1} \val{[2.9, 3.3]} & \val{+3.4} \val{[+3.2, +3.6]} & \val{30} & \val{1.98\%} & \val{1.64\%} & -- & \val{1.00} & \val{1.00} & --\\
Incumbent trained on the registered corruptions & \val{3.4} \val{[3.1, 3.6]} & \val{+3.1} \val{[+2.9, +3.3]} & \val{62} & \val{3.73\%}$^\times$ & \val{3.54\%} & -- & \val{1.00} & \val{1.00} & --\\
\addlinespace[1pt]\multicolumn{10}{l}{\emph{The learned update itself}}\\
Learned update itself ($t{\equiv}1$), trimmed to $\bar\Delta$ & \val{3.9} \val{[3.6, 4.3]} & \val{+2.6} \val{[+2.3, +2.9]} & \val{159} & \val{8.81\%}$^\times$ & \val{2.93\%} & -- & \val{2.18} & \val{2.18} & \val{$<$0.001}\\
\addlinespace[1pt]\multicolumn{10}{l}{\emph{Declared variants and one-at-a-time swaps}}\\
Per-depth radii ($q{=}0.1$ per depth; Theorem~\ref{thm:served}(iii) not claimed) & \val{6.9} \val{[6.6, 7.1]} & \val{\ensuremath{-}0.4} \val{[\ensuremath{-}0.5, \ensuremath{-}0.2]} & \val{13} & \val{1.01\%} & \val{0.54\%} & \val{6.6} & \val{1.42} & \val{1.36} & --\\
No per-row budget ($H_{\rm row}{=}\infty$) & \val{6.7} \val{[6.4, 6.9]} & \val{\ensuremath{-}0.2} \val{[\ensuremath{-}0.4, 0.0]} & \val{12} & \val{0.95\%} & \val{0.54\%} & \val{6.4} & \val{1.42} & \val{1.36} & --\\
Spend-based caps instead of efficiency placement & \val{6.2} \val{[5.9, 6.4]} & \val{+0.3} \val{[+0.2, +0.5]} & \val{12} & \val{0.95\%} & \val{0.51\%} & \val{5.9} & \val{1.42} & \val{1.36} & --\\
Window solved under the envelope only & \val{6.4} \val{[6.2, 6.7]} & \val{+0.1} \val{[\ensuremath{-}0.1, +0.2]} & \val{12} & \val{0.95\%} & \val{0.51\%} & \val{6.2} & \val{1.42} & \val{1.36} & --\\
Mean-pooled candidate features (not hull-invariant) & \val{6.4} \val{[6.1, 6.6]} & \val{+0.1} \val{[0.0, +0.3]} & \val{13} & \val{1.01\%} & \val{0.54\%} & \val{6.1} & \val{1.42} & \val{1.36} & --\\
Probe-loss slopes instead of full-loss slopes & \val{6.2} \val{[5.9, 6.4]} & \val{+0.3} \val{[+0.2, +0.5]} & \val{15} & \val{1.13\%} & \val{0.54\%} & \val{5.9} & \val{1.42} & \val{1.36} & --\\
Symmetric charges at the same budget ($|d|\le0.05$) & \val{3.9} \val{[3.7, 4.1]} & \val{+2.6} \val{[+2.4, +2.7]} & \val{10} & \val{0.83\%} & \val{0.46\%} & -- & \val{1.42} & \val{1.36} & --\\
Symmetric charges at $0.3$ nats & \val{5.5} \val{[5.3, 5.7]} & \val{+1.0} \val{[+0.9, +1.2]} & \val{14} & \val{1.07\%} & \val{0.63\%} & -- & \val{1.42} & \val{1.36} & --\\
Global factors only (no interval tube) & \val{4.1} \val{[3.9, 4.3]} & \val{+2.4} \val{[+2.2, +2.5]} & \val{11} & \val{0.89\%} & \val{0.51\%} & \val{3.9} & \val{1.42} & \val{1.36} & --\\
Late-8 acceptance ($\ell\ge24$ only) & \val{4.8} \val{[4.6, 5.0]} & \val{+1.7} \val{[+1.5, +1.8]} & \val{11} & \val{0.89\%} & \val{0.46\%} & \val{4.6} & \val{1.24} & \val{1.19} & --\\
\addlinespace[1pt]\multicolumn{10}{l}{\emph{Diagnostic (uses labels)}}\\
Posterior-slope gate ($\approx\hat U_{\rm cert}$; cross-fitted posterior) & \val{7.9} \val{[7.6, 8.1]} & \val{\ensuremath{-}1.4} \val{[\ensuremath{-}1.5, \ensuremath{-}1.2]} & \val{0} & \val{0.15\%} & \val{0.27\%} & \val{7.6} & \val{1.42} & \val{1.36} & --\\
\addlinespace[1pt]Oracle frontier $U_{\rm out}$ (cross-fitted posterior) & \val{10.6} & -- & -- & -- & -- & -- & -- & -- & --\\
\bottomrule\end{tabularx}
}
\newcommand{\TabSwaps}{%
\scriptsize\setlength{\tabcolsep}{2.0pt}\begin{tabularx}{\linewidth}{l>{\raggedright\arraybackslash}p{3.2cm}rrY}\toprule
Mechanism & swap & gain & $u$ & what changes\\\midrule
Per-node budget & none $\to$ $H_{\rm row}{=}1$ nat & \val{6.69} $\to$ \val{6.50} & \val{0.95\%} $\to$ \val{0.95\%} & per-node floor becomes deterministic; fallback \val{0.34\%}$\to$\val{0.34\%}\\
Radii & per depth $\to$ sequential, trajectory-valid & \val{6.86} $\to$ \val{6.50} & \val{1.01\%} $\to$ \val{0.95\%} & trajectory statement of Theorem~\ref{thm:served}(iii); selection caveat applies\\
Placement & spend-based $\to$ certified efficiency & \val{6.17} $\to$ \val{6.50} & \val{0.95\%} $\to$ \val{0.95\%} & trust moves to depths with high $\eta_\ell$\\
Window & envelope only $\to$ two branches & \val{6.45} $\to$ \val{6.50} & \val{0.95\%} $\to$ \val{0.95\%} & optimum over the union of both feasible sets\\
Features & mean pooling $\to$ support functions & \val{6.39} $\to$ \val{6.50} & \val{1.01\%} $\to$ \val{0.95\%} & executed step exactly invariant (Table~\ref{tab:checker})\\
Slopes & probe loss $\to$ full loss & \val{6.16} $\to$ \val{6.50} & \val{1.13\%} $\to$ \val{0.95\%} & accounting identity closes\\
Pricing & symmetric $\to$ one-sided & \val{3.92} $\to$ \val{6.50} & \val{0.83\%} $\to$ \val{0.95\%} & same budget $0.05$ nats\\
Tube & global $\to$ interval factors & \val{4.12} $\to$ \val{6.50} & \val{0.89\%} $\to$ \val{0.95\%} & open depths admissible\\
Depths & late eight only $\to$ efficiency-placed & \val{4.84} $\to$ \val{6.50} & \val{0.89\%} $\to$ \val{0.95\%} & latency \val{1.24}$\to$\val{1.42}\\
\bottomrule\end{tabularx}
}

\newcommand{\TabNeighbours}{%
\scriptsize\setlength{\tabcolsep}{2.0pt}\begin{tabularx}{\linewidth}{Yrrr}\toprule
(a) Neighbor & gain & $u$ & Lat.\\\midrule
WiSE-FT (incumbent $\leftrightarrow$ v2), $\alpha$ by learn-then-test & \val{3.4} & \val{2.15\%}$^\times$ & \val{1.00}\\
LoRA scaling $\alpha$ by learn-then-test & \val{3.2} & \val{2.20\%}$^\times$ & \val{1.00}\\
EATA (anti-forgetting test-time adaptation) & \val{2.6} & \val{3.51\%}$^\times$ & \val{2.10}\\
RDumb (periodic reset) & \val{2.2} & \val{2.86\%}$^\times$ & \val{2.05}\\
$D_\infty$-regularized fine-tuning (floor trained, not enforced) & \val{3.4} & \val{1.81\%} & \val{1.00}\\
\bottomrule\end{tabularx}\par\smallskip
\scriptsize\setlength{\tabcolsep}{2.2pt}\begin{tabularx}{\linewidth}{Yrrr}\toprule
(b) Collapse stream & mean damage & regressing calls & fallback\\\midrule
GTrans, raw & \val{+38.0} & \val{31.0\%} & --\\
EATA & \val{+9.1} & \val{12.4\%} & --\\
RDumb & \val{+5.2} & \val{8.9\%} & --\\
GTrans, RTT-gated & \val{+0.4} & \val{1.3\%} & \val{6.1\%}\\
\bottomrule\end{tabularx}
}
\newcommand{\TabBreadth}{%
\scriptsize\setlength{\tabcolsep}{1.4pt}\begin{tabular*}{\linewidth}{@{\extracolsep{\fill}}>{\raggedright\arraybackslash}p{2.35cm}>{\raggedright\arraybackslash}p{2.0cm}rrrrrrrrrrl@{}}\toprule
Dataset / shift & incumbent & $\zeta$ & $\hat U_{\rm out}$ & RTT & corr. & TR & share & $u$ & $u_{\rm node}$ & verified & Lat. & metric\\\midrule
ogbn-arxiv, registered mixture & tanh diffusion, $T{=}32$ (tube) & 2 & \val{10.6} & \val{6.5} & \val{5.6} & \val{5.8} & \val{61\%} & \val{0.95\%} & \val{0.51\%} & \val{99.37\%} & \val{1.42} & acc. \val{75.7}$\to$\val{76.1}\\
ogbn-arxiv, registered mixture & APPNP (linear tail, exact) & 2 & \val{9.4} & \val{6.0} & \val{5.4} & \val{5.3} & \val{64\%} & \val{0.89\%} & \val{0.46\%} & \val{100.00\%} & \val{1.22} & acc. \val{75.5}$\to$\val{75.9}\\
ogbn-arxiv, registered mixture & GraphSAGE, 2 layers (shallow) & 2 & \val{8.1} & \val{5.2} & \val{5.0} & \val{4.6} & \val{64\%} & \val{0.83\%} & \val{0.46\%} & \val{100.00\%} & \val{1.52} & acc. \val{72.1}$\to$\val{72.4}\\
ogbn-arxiv, registered mixture & GCN, 3 layers (shallow) & 2 & \val{8.4} & \val{5.4} & \val{5.2} & \val{4.8} & \val{64\%} & \val{0.89\%} & \val{0.46\%} & \val{100.00\%} & \val{1.47} & acc. \val{71.8}$\to$\val{72.2}\\
ogbn-arxiv, 2020 cohort (temporal) & tanh diffusion (tube) & 2 & \val{2.9} & \val{1.9} & \val{1.6} & \val{1.3} & \val{66\%} & \val{1.18\%} & -- & \val{99.51\%} & \val{1.42} & acc. \val{74.1}$\to$\val{74.2}\\
GOOD-Arxiv, degree (covariate) & GCNII (ReLU, interval Jacobians) & 2.3 & \val{9.7} & \val{5.8} & \val{5.1} & \val{4.9} & \val{60\%} & \val{1.13\%} & \val{0.61\%} & \val{99.22\%} & \val{1.43} & acc. \val{63.8}$\to$\val{64.1}\\
GOOD-Arxiv, time (concept) & GRAND (smooth) & 2.3 & \val{8.3} & \val{5.2} & \val{4.6} & \val{4.3} & \val{63\%} & \val{1.18\%} & \val{0.63\%} & \val{99.41\%} & \val{1.47} & acc. \val{66.7}$\to$\val{67.0}\\
Twitch-explicit (cross-graph) & APPNP (exact) & 3.1 & \val{6.4} & \val{4.2} & \val{3.9} & \val{3.5} & \val{66\%} & \val{1.13\%} & \val{0.66\%} & \val{100.00\%} & \val{1.20} & AUC \val{63.4}$\to$\val{63.8}\\
Facebook-100 (cross-graph) & GCNII (ReLU) & 3.4 & \val{5.1} & \val{3.1} & \val{2.9} & \val{2.6} & \val{61\%} & \val{1.07\%} & \val{0.71\%} & \val{99.32\%} & \val{1.41} & acc. \val{55.6}$\to$\val{55.8}\\
Elliptic (temporal, post-shutdown) & GRAND (smooth) & 2.8 & \val{24.6} & \val{16.9} & \val{14.1} & \val{14.3} & \val{69\%} & \val{1.30\%} & \val{0.78\%} & \val{99.12\%} & \val{1.46} & F1 \val{47.9}$\to$\val{49.0}\\
ogbn-products, registered mixture & tanh diffusion (tube) & 1.6 & \val{10.4} & \val{7.6} & \val{6.9} & \val{6.8} & \val{73\%} & \val{1.07\%} & \val{0.61\%} & \val{99.51\%} & \val{1.39} & acc. \val{80.1}$\to$\val{80.5}\\
ImageNet-C, severity 3 (15 corruptions) & ResNet-50, last 4 blocks; Tent & 10 & \val{452.0} & \val{214.0} & \val{171.0} & \val{168.0} & \val{47\%} & \val{1.47\%} & -- & \val{99.07\%} & \val{1.34} & top-1 \val{40.1}$\to$\val{45.6}\\
Folktables ACSIncome, CA $\to$ five states & residual tanh MLP, $P{=}I$ & 2 & \val{6.2} & \val{3.9} & \val{3.7} & \val{3.4} & \val{63\%} & \val{1.07\%} & -- & \val{99.90\%} & \val{1.21} & acc. \val{80.6}$\to$\val{80.9}\\
\bottomrule\end{tabular*}
}
\newcommand{\TabMol}{%
\scriptsize\setlength{\tabcolsep}{2.4pt}\begin{tabularx}{\linewidth}{Yrrrrr}\toprule
Reference family & AP ref. & AP TR & AP RTT & gain ($10^{-3}$) & RTT $-$ TR (AP)\\\midrule
Contractive diffusion & \val{26.4} & \val{27.8} & \val{28.5} & \val{1.52} & \val{+0.7} \val{[+0.4, +1.0]}\\
Centered GIN-style & \val{27.3} & \val{28.2} & \val{28.4} & \val{1.03} & \val{+0.2} \val{[0.0, +0.4]}\\
Incumbent GIN-style (deployed) & \val{27.7} & \val{28.2} & \val{28.6} & \val{0.88} & \val{+0.4} \val{[+0.1, +0.7]}\\
GIN-style $+$ virtual node & \val{28.4} & \val{28.8} & \val{29.2} & \val{0.79} & \val{+0.4} \val{[+0.2, +0.6]}\\
\addlinespace[1pt]\multicolumn{6}{l}{\emph{Official test split (audited once)}}\\
Incumbent GIN-style (deployed) & \val{27.7} & -- & \val{28.5} & -- & --\\
GIN-style $+$ virtual node & \val{28.4} & -- & \val{29.1} & -- & --\\
\bottomrule\end{tabularx}
}
\newcommand{\TabCosts}{%
\scriptsize\setlength{\tabcolsep}{2.2pt}\begin{tabularx}{\linewidth}{Yrrrr}\toprule
(a) Proposal & its cost (ms) & RTT & 2 rollouts, sequential & 2 devices\\\midrule
bank prior + signed attention (ours) & \val{7.5} & \val{1.42} & \val{2.18} & \val{1.18}\\
GPR-GNN (adapter) & \val{12.1} & \val{1.53} & \val{2.29} & \val{1.29}\\
Co-GNN (adapter) & \val{16.0} & \val{1.62} & \val{2.38} & \val{1.38}\\
AMP (adapter) & \val{14.6} & \val{1.59} & \val{2.35} & \val{1.35}\\
GTrans (test-time adaptation) & \val{91.0} & \val{3.42} & \val{3.19} & \val{2.19}\\
Matcha (test-time adaptation) & \val{58.0} & \val{2.63} & \val{2.39} & \val{1.39}\\
Retrained incumbent (v2, $\le$2017) & \val{41.6} & \val{2.24} & \val{2.00} & \val{1.00}\\
LoRA-16 weight delta & \val{43.0} & \val{2.27} & \val{2.03} & \val{1.03}\\
one-pass corrector, 2-hop head $+$ rows & \val{10.4} & \val{1.25} & -- & --\\
\bottomrule\end{tabularx}\par\smallskip
\scriptsize\setlength{\tabcolsep}{2.2pt}\begin{tabularx}{\linewidth}{Yrrr}\toprule
(b) Stage & episode evaluations & labels & device-hours\\\midrule
incumbent temperature (2018 labels; network already trained) & 0 & yes & \val{0.3}\\
phase 2: proposal, field and covariance heads ($\le$2017) & \val{6,144} & yes & \val{6.9}\\
one-pass correctors, matched training effort ($\le$2017) & \val{6,144} & yes & \val{2.8}\\
adapters of the external proposals ($\le$2017) & \val{3,072} & yes & \val{1.4}\\
efficiency placement replay ($\le$2017, replayed) & \val{2,048} & yes & \val{0.5}\\
sequential calibration, 16 depths $\times$ 512, four members (2018) & \val{32,768} & yes & \val{3.1}\\
joint selection: 5 members on four pools of 1,024 (mixture, clean, cache, noise; 2018) & \val{20,480} & yes & \val{2.1}\\
development half (2019): 25 configurations $\times$ 2,048 (sweeps, call sizes, $\zeta$ anchor) & \val{51,200} & yes & \val{8.8}\\
final audit, $2{,}048+6\times4{,}096$ (2019 final half) & \val{26,624} & yes & \val{4.6}\\
per-node audit, two calls of the registered size (4,096 labels inspected) & \val{2} & yes & \val{0.0}\\
transfer, reserved 2020 cohort & \val{2,048} & yes & \val{0.3}\\
held-out family: audit, re-calibration, re-audit & \val{5,120} & yes & \val{0.3}\\
service streams, 20 drifted $\times$ 600 & \val{12,000} & yes & \val{0.4}\\
test-time adaptation collapse streams, 20 $\times$ 600 & \val{12,000} & yes & \val{0.4}\\
service streams, 20 clean $\times$ 600 (label-free) & \val{12,000} & no & \val{0.4}\\
stress, invariance and checker tests (label-free; offline second rollout) & \val{6,144} & no & \val{0.6}\\
breadth: families, shifts, molecules, non-graph (own splits) & \val{30,720} & yes & \val{5.8}\\
\textbf{total} & \val{228,514} & \val{210,370 labeled} & \val{38.7}\\
\bottomrule\end{tabularx}
}
\newcommand{\TabChecker}{%
\scriptsize\setlength{\tabcolsep}{2.2pt}\begin{tabularx}{\linewidth}{>{\raggedright\arraybackslash}p{4.1cm}Yl}\toprule
Test & result & verdict\\\midrule
planted violations, hidden margins $10^{-7}$--$10^{-3}$ nats & \val{10,000 of 10,000 rejected} & \val{test passed}\\
50-digit re-evaluation of certificates & \val{1,000 certificates; float64 enclosure above the high-precision reference value, by at most $3.1\times10^{-10}$} & \val{test passed}\\
exact endpoint cost (offline second rollout) & \val{0 violations on 2,048 calls; certified / exact median 1.61} & \val{test passed}\\
adversarial directions within the trim & \val{released 41\%; released damage mean 29, max 49 ($10^{-3}$); over all calls 11.9} & \val{bound met; audit fails}\\
provisional factors under-estimated by 30\% & \val{214 of 214 violating calls caught (fallback)} & \val{test passed}\\
outward rounding switched off & \val{3 of 2,048 certificates below the exact cost by $\le10^{-6}$} & \val{diagnostic}\\
duplicate a candidate & \val{displacement change 0.4\% (mean pooling) vs 0\% (support functions)} & \val{unchanged}\\
duplicate, rescaled by $\tfrac12$ & \val{displacement change 0.4\% (mean pooling) vs 0\% (support functions)} & \val{unchanged}\\
permute the candidates & \val{displacement change 0\% (mean pooling) vs 0\% (support functions)} & \val{unchanged}\\
add a candidate inside the post-trim hull & \val{displacement change 0.3\% (mean pooling) vs 0\% (support functions)} & \val{unchanged}\\
equivalent spanning set of the same hull & \val{displacement change 0.2\% (mean pooling) vs 0\% (support functions)} & \val{unchanged}\\
replace a candidate by a shorter one (the hull changes) & \val{displacement change 21\% (mean pooling) vs 21\% (support functions)} & \val{hull changes}\\
\bottomrule\end{tabularx}
}
\newcommand{\TabPositioning}{%
\scriptsize\setlength{\tabcolsep}{2.0pt}\begin{tabularx}{\linewidth}{>{\raggedright\arraybackslash}p{3.6cm}llYll}\toprule
Line & check & changes & guarantee & reference run & updates\\\midrule
Differential verification \citep{reludiff,neurodiff,verydiff} & offline & a second network & bounds on the difference & yes & any fixed pair\\
Certified upgrades \citep{zhang2026slice,balachandran2026} & offline audit & replacement model & statistical non-regression & yes & replacement\\
Conformal policy control \citep{cpc}; LTT, CALM \citep{ltt,calm} & calibration & policy or exit & risk control & yes / partly & policies, exits\\
CPI, PPO \citep{kakade,ppo} & training & policy & CPI: model-based improvement; PPO: clipped surrogate & -- & policies\\
Weight interpolation \citep{wiseft} & offline & weights & none per call & -- & weight deltas\\
Safe / graph TTA \citep{eata,rdumb,matcha,tsa} & online & states or weights & none per call & -- & test-time updates\\
Steering with abstention \citep{mera} & online & activations & statistical & partly & steering vectors\\
\textbf{RTT} & online, per call & internal states & deterministic per node and call; audited rates & fallback only & any bounded proposal\\
\bottomrule\end{tabularx}
}
\newcommand{\TabNotation}{%
\scriptsize\setlength{\tabcolsep}{3.0pt}\begin{tabularx}{\linewidth}{lY}\toprule
Symbol & meaning\\\midrule
$\pr,\ p^\pi,\ \ps,\ \pu,\ \pB$ & incumbent's, adapted, served, raw update's prediction; Bayes posterior\\
$\alpha,\ \tau,\ T$ & the incumbent's skip weight, step size and horizon\\
$h_i=\Dinf(\pr_i\|\ps_i)$, $S^\pm$ & one-sided damage of row $i$; worst-case loss increase and decrease of a call\\
$H^+,\ H_{\rm row},\ \mathcal B_\ell,\ \rho_{\ell,i}$ & per-call and per-node budgets; remaining call budget; row allowance\\
$\Fl_h(p)$ & floor set $e^{-h}p+(1-e^{-h})\Delta$\\
$\mathcal O,\ \mathcal G$ & label-free information; expected gain given it\\
$r_{j,i},\ \tilde r,\ L_{j,i},\ \mu_{j,i}$ & tube radius, provisional radius, interval factor, curvature factor\\
$\Pi_\ell,\ \hat\Pi_\ell$ & contract growth from depth $\ell$ to $T$; its registered proxy in the row-budget trim\\
$\Lambda,\ \Xi$ & first- and second-order row contracts\\
$\sigma,\ {\rm env}^+,\ C^+$ & contract charge, window envelope, exact last-depth charge\\
$\mathcal K_\ell,\ D_\ell=Q_\ell\Theta_\ell,\ v,\ x,\ t$ & candidate hull, directions, physical step, its coordinates, candidate coefficients\\
$\hat g,\ \hat a,\ \hat\Sigma,\ \kappa_\ell,\ \delta_\ell$ & field, slopes, covariance, radius, per-depth miscoverage\\
$R_i,\ G_i,\ \eta_\ell$ & under-estimation factor, row frontier, architecture-constrained efficiency\\
$U_{\rm out},U_{\rm reach},U_{\rm cert},U_{\rm RTT}$ & the gap chain of Theorem~\ref{thm:frontier}\\
$m_\ell,\ I_\ell,\ \chi_\ell,\ \Upsilon_\ell,\ \bar\Xi_\ell$ & mean slope given the heads' inputs, information gap, conservatism, execution gap, certified curvature (equation~\eqref{eq:localgap})\\
$E$ & summed estimation term $\sum_\ell\frac12\|m_\ell-\hat a_\ell\|^2_{\bar\Xi_\ell^{-1}}$ (Table~\ref{tab:gap})\\
$\zeta,\ u,\ u_{\rm node}$ & regression tolerance, per-call and per-node bounds\\
\bottomrule\end{tabularx}
}

\title{Reference-Tail Trust:\\Certified Probability Floors for Learned\\Updates Inside a Deployed Network}
\author{Abdolvahab Khalili Sadaghiani, Jose Nunez-Yanez}
\date{}
\begin{document}
\maketitle
\begin{abstract}
Graph neural networks (GNNs) need to exploit improved message passing without surrendering control over predictions already trusted in deployment. We introduce Reference-Tail Trust (RTT), a framework that admits learned updates inside a frozen GNN and certifies the prediction actually served. RTT couples graph-based proposal states with a constrained internal optimizer: each displacement is charged for its worst-case terminal cross-entropy increase through the incumbent's remaining message-passing layers. A trajectory-validated tube and an independent checker enforce per-node probability floors, $\ps_{ic}\ge e^{-\Hrow}\pr_{ic}$, and a call-level budget, $\sum_iw_i\Dinf(\pr_i\|\ps_i)\le H^+$, uniformly over labels. Calls whose adapted outputs pass certification require no separate full incumbent rollout; failed certificates trigger whole-call fallback. We derive the exact probability-floor frontier by water-filling, characterize architecture-constrained efficiency, and establish conditions under which internal propagation exploits evidence unavailable to restricted output correctors. In the reported ogbn-arxiv audit, RTT achieves \vRttI$\times10^{-3}$ nats of mean gain per call, with a one-sided 95\% regression-rate upper bound of \vRttU\ and a 95\% negative-flip upper bound of \vRttNfrb\ on the uninspected part of the reserved node population. Its mean gain is \vShareUout\ of a cross-fitted posterior-based frontier estimate and exceeds the strongest matched one-pass corrector by \vPrimD$\times10^{-3}$ nats. Reported experiments span eight proposals, six graph-incumbent families, structural and temporal graph shifts, and molecular prediction, with additional image and tabular evaluations. RTT makes GNN adaptation a budgeted, certifiable inference decision rather than an unconditional model replacement.
\end{abstract}

\section{Introduction}\label{sec:intro}
A deployed network is a reference point, not merely a checkpoint waiting to be replaced. Its predictions support decisions that users and downstream systems have already learned to rely on. New propagation rules, test-time adaptation, retraining, and weight updates may improve validation performance while introducing regressions on particular inputs. An average improvement therefore leaves an operational question unanswered: how much of an update should be admitted on the current call, and what can be guaranteed about the prediction that is actually delivered?

We study this question for learned changes to the internal states of a frozen network, which we call the \emph{incumbent}. The objective is not to certify that every update is beneficial. Without assumptions linking deployment information to labels, a nontrivial label-free change cannot generally guarantee improvement. Instead, we separate a deterministic protection requirement from a statistical performance claim. Each served prediction must remain within a specified one-sided cross-entropy budget relative to the incumbent, for every possible label vector. Expected gain and the probability of exceeding a registered regression tolerance are then assessed on fresh deployment episodes. This distinction makes the unit of protection explicit: individual output rows receive a probability floor, whereas calls receive an aggregate damage budget.

Reference-Tail Trust (RTT) implements this separation inside the network.\footnote{\CodeAvailability} A learned proposal supplies candidate state displacements, and a constrained optimizer selects how much of them to execute. Each displacement is evaluated through the \emph{reference tail}: the incumbent's continuation from the current internal state to the prediction head. Bounds on that continuation translate an internal movement into a worst-case terminal loss increase. A tube centered on the executed trajectory encloses the relevant counterfactual continuations, allowing the certificate to account for interactions between interventions at different depths. The final checker recomputes these bounds on the trajectory that was actually executed and either releases the adapted output or serves the incumbent for the whole call. Learned modules determine utility; the checker determines admissibility.

The resulting protection has a direct probabilistic interpretation. A row budget $\Hrow$ guarantees $\ps_{ic}\ge e^{-\Hrow}\pr_{ic}$ for every class, and the call budget bounds the weighted order-$\infty$ R\'enyi divergence $\sum_iw_i\Dinf(\pr_i\|\ps_i)$. These guarantees do not require deployment labels or a trustworthy proposal. They also do not imply that internal intervention is intrinsically superior to output correction: every floor-feasible prediction can be expressed as a mixture containing the incumbent prediction. Accordingly, we derive an exact output-space frontier, identify the information and receptive-field restrictions under which internal steps can help, and compare RTT with a learned one-pass corrector given the same floor, labels, head capacity, and proposal features. In the reported arxiv audit, RTT gains \vRttI$\times10^{-3}$ nats per call against the corrector's \vOpcI$\times10^{-3}$, at latencies of \vLatOwn$\times$ and \vLatOpc$\times$ one incumbent pass, respectively (Table~\ref{tab:frontier}). This comparison tests the additional utility and cost of internal admission without crediting RTT for protection that a simpler output rule also provides.

Prior work addresses related parts of this problem. Backward-compatible learning and positive-congruent training reduce prediction regressions during model development \citep{srivastava2020,trauble2021,yan2021}, while certified model-upgrade procedures audit replacement decisions against an incumbent \citep{zhang2026slice,balachandran2026}. Differential verification compares fixed networks \citep{reludiff,neurodiff,verydiff}, and bound propagation provides tools for certifying neural computations \citep{crown,autolirpa}. Learned graph propagation and adaptive message passing expand the available update mechanisms \citep{gcnii,gpr,cognn,amp}; weight interpolation, low-rank updates, and test-time adaptation provide further proposals but do not themselves enforce RTT's per-call floor \citep{wiseft,lora,tent,eata,sar,rdumb,gtrans,matcha}. Among recent conference methods, Test-Time Structural Alignment (TSA; AISTATS 2026) adjusts neighborhood influence, self--neighbor balance, and the decision boundary \citep{tsa}, while STEM (The Web Conference 2026) uses a state-space controller to generate adapters for a frozen GNN under evolving domain shifts \citep{stem2026}. These methods motivate complementary tests of RTT: admitting structure-aware corrections and certifying the outputs of history-dependent adaptation. Risk-controlled calibration, adaptive prediction, conformal policy control, and steering with abstention regulate other aspects of prediction or policy selection \citep{ltt,calm,cpc,mera}. Conservative policy iteration motivates telescoping comparisons, whereas PPO's clipped surrogate is a distinct training device, not an all-class probability-floor certificate \citep{kakade,ppo}. Finally, calibration and uncertainty under shift explain why a label-free correction can have useful conditional information \citep{guo2017,gats,ovadia2019}. RTT brings these questions together at a different decision point: admission of internal changes on an individual deployment call, with protection attached to the served output.

RTT combines certified internal admission, an exact probability-floor frontier, and explicit statistical evaluation in one framework. It can accept useful parts of a learned update, withhold uncertified adapted outputs, and expose the gain--protection--cost trade-off without changing the incumbent's weights.

\section{Reference-Tail Trust: Model and Certified Execution}\label{sec:rtt}\label{sec:problem}
This section develops the execution mechanism from the deployment objective to the final release decision. The analytical and statistical consequences are established in Section~\ref{sec:theory}; the complete implementation primitives and parameter definitions are given in Appendices~\ref{app:model} and~\ref{app:repro}.

\subsection{Deployment objective and the probability-floor contract}\label{sec:formulation}
Let a graph with $n$ nodes and features $X$ be processed by a frozen incumbent with $H_0=\mathrm{Enc}(X)$, transitions $H^{\rm r}_{\ell+1}=F^{\rm r}_\ell(H^{\rm r}_\ell)$ for $\ell<T$, and a row-wise affine logit head $Z$. An episode is one deployment call with a query set of $n_{\rm q}$ scored rows, nonnegative weights $w_i$ summing to one, and a draw of the registered perturbation. The adapted and incumbent executions are compared on the same input, support, cache realization, and scheduled exogenous transformations. Any state replacement must be included in both continuation maps used by the certificate. Labels are unavailable to the deployed admission rule. They are used offline for training, calibration, and evaluation. A proposal supplies candidate states at selected depths, either directly or through a frozen adapter into the incumbent's state space.

RTT executes $H_{\ell+1}=F^{\rm r}_\ell(H_\ell)+v_\ell$ and computes the adapted prediction $p^\pi=\softmax(Z(H_T))$. The prediction actually served is
\begin{equation}
 \ps=\mathcal S(p^\pi,\mathrm{certificate},\pr)
 =\begin{cases}p^\pi,&\text{if the certificate holds},\\
 \pr,&\text{otherwise: whole-call fallback},\end{cases}
 \label{eq:served}
\end{equation}
where $\pr=\softmax(Z(H_T^{\rm r}))$ is the incumbent's output on the same call. A call whose adapted output passes the checker does not require a separate full incumbent rollout. The short reference-tail computations used near the terminal depth remain part of RTT's cost; on fallback, the unmodified tail is evaluated from the cached state preceding the first intervention.

For weighted cross-entropy $J(p;y)=-\sum_iw_i\log p_{i,y_i}$, define $d^{\rm own}=J(\ps;y)-J(\pr;y)$, so a positive value is damage and $-d^{\rm own}$ is gain. The deployment objective separates pointwise constraints from the audited regression-rate requirement:
\begin{equation}
 \begin{aligned}
 \max\quad &\E_e[-d^{\rm own}_e]\\
 \text{subject to}\quad
 &\Dinf(\pr_i\|\ps_i)\le\Hrow\quad\text{for every row and call},\\
 &\sum_iw_i\Dinf(\pr_i\|\ps_i)\le H^+\quad\text{for every call},\\
 &\Pr_e(d^{\rm own}_e>\zeta)\le p^\star.
 \end{aligned}\label{eq:problem}
\end{equation}
The first two constraints must hold for all label vectors, independently of the proposal. The last is a statistical statement about a specified deployment generator. On arxiv, $\zeta=0.002$ nats is a registered tolerance, anchored to the development-set median change of \vZetaAnchor\ nats among calls whose accuracy drops by 0.1 percentage point. This anchor is not a conversion from cross-entropy to accuracy; accuracy and negative flips are audited separately.

The label-uniform price follows exactly from cross-entropy. If $h_{ic}=\log(\pr_{ic}/\ps_{ic})$, then
\begin{equation}
 S^+=\sup_y d^{\rm own}=\sum_iw_i\max_c h_{ic}
     =\sum_iw_i\Dinf(\pr_i\|\ps_i),
 \qquad S^-=\sum_iw_i\Dinf(\ps_i\|\pr_i).
 \label{eq:exactdamage}
\end{equation}
Thus $\Dinf(\pr_i\|p)\le h$ is equivalent to membership in the floor set
\begin{equation}
 \Fl_h(\pr_i)=\{p\in\Delta:p_c\ge e^{-h}\pr_{ic}\ \forall c\}
             =e^{-h}\pr_i+(1-e^{-h})\Delta.
 \label{eq:floor}
\end{equation}
The same constraints also imply a ceiling $p_c\le1-e^{-h}(1-\pr_{ic})$. Table~\ref{tab:guarantees} distinguishes these deterministic protections from the statistical claims.

\begin{table}[t]\centering
\caption{Guarantees for the served prediction~\eqref{eq:served}. G1--G3 require valid continuation bounds and sound certificate arithmetic, but no deployment labels or assumptions on proposal quality. G4--G5 concern the audited generator. G6 separates a pathwise coverage implication from a probability statement requiring selection-valid calibration (Section~\ref{sec:assurance}).}\label{tab:guarantees}
\TabGuarantees
\end{table}

\subsection{Candidate geometry and intervention placement}\label{sec:candidates}
The main arxiv incumbent is a deep diffusion,
\begin{equation}
 F^{\rm r}_\ell(H)=\alpha H_0+(1-\alpha)\big[(1-\tau)H+\tau\tanh(PHW)\big],
 \label{eq:incumbent}
\end{equation}
where $P$ is row-stochastic, $T=32$, $\alpha=0.1$, $\tau=0.9$, and $\|W\|_2=\vWinc$. Its global one-step bound is $(1-\alpha)(1-\tau+\tau\|W\|_2)=\vLglobal>1$, so useful admission cannot be justified by assuming global contraction. RTT instead uses interval factors on a validated trajectory tube. Appendix~\ref{app:model} gives the corresponding primitives for APPNP, GCNII, GRAND, and shallow GraphSAGE/GCN \citep{appnp,gcnii,grand,hamilton2017,kipf2017}.

At open depths $\ell=16,\ldots,31$, the native proposal supplies two candidates per row: a state-conditioned convex combination of 16 bank vectors and a signed one-hop attention update. External proposals are mapped through frozen linear adapters. Starting from $\bar H_{\ell+1}=F^{\rm r}_\ell(H_\ell)$, the candidate displacements $H^{\mathrm c,m}_\ell-\bar H_{\ell+1}$ undergo norm trimming and the row-budget trimming described below. Their resulting directions form a matrix $D_\ell$, partitioned into $K=8$ role blocks. The admissible physical displacement is
\begin{equation}
 v_\ell=D_\ell t_\ell\in\mathcal K_\ell,
 \qquad t_{\ell km}\ge0,\quad\sum_m t_{\ell km}\le\bar c_{\ell k}.
 \label{eq:hull}
\end{equation}
The zero displacement is feasible. If the direction matrix has rank zero, the step is defined as zero and no covariance inverse or optimizer is needed. Consequently, each block can retain the incumbent or accept a capped combination of its candidates. The caps allocate opportunity across depths using replayed gain per certified nat; this placement is an efficiency heuristic motivated by Section~\ref{sec:frontier}, not a claim that the internal optimization attains the output-space oracle.

The optimizer operates on the displacement, rather than on a redundant candidate representation. A thin factorization $D_\ell=Q_\ell\Theta_\ell$ gives physical coordinates $v_\ell=Q_\ell x$ with $x=\Theta_\ell t$. Candidate-dependent head inputs are support-function features,
\begin{equation}
 s_{ij}=\max\{0,\max_m\langle u_j,d_{\ell,i}^{(\cdot,m)}\rangle\},
 \label{eq:support}
\end{equation}
over fixed linear probes. Duplicating or permuting candidates, or adding points inside the same \emph{post-trim} hull, leaves these features unchanged. This invariance does not cover a replacement that changes the hull, and nonlinear trimming must be accounted for before comparing hulls (Lemma~\ref{lem:invariance}).

\begin{figure}[t]
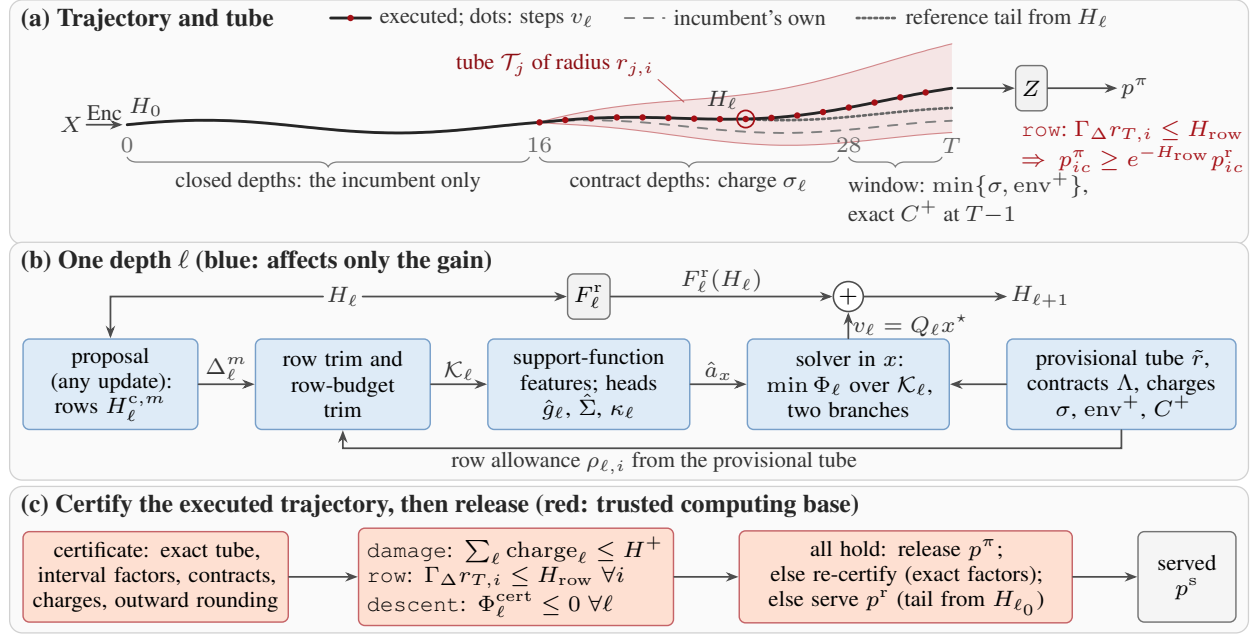
\centering\resizebox{\linewidth}{!}{\FigArch}
\caption{RTT execution and certification. (a) Internal steps leave the incumbent trajectory; validated tubes enclose the reference continuations and bound terminal row damage. (b) Proposal features and learned heads choose a displacement inside the candidate hull. (c) The checker recomputes the executed trajectory's tube, factors, charges, and descent predicates before release; a persistent failure triggers whole-call fallback. Blue modules determine utility, while the red certificate path determines deterministic validity. Labels are used offline, not by the deployed admission decision.}\label{fig:arch}
\end{figure}

\paragraph{Recent GNN proposals through the same certificate interface.}
The TSA and STEM extensions use the existing external-proposal interface rather than changing RTT's incumbent or certificate. Let $a\in\{\mathrm{TSA},\mathrm{STEM}\}$ index the proposal, let $\mathcal H_e$ denote its permitted label-free history, and let $\widetilde H_{\ell+1}^{(a)}$ be a row-aligned state exported at a predeclared hook. A frozen adapter $\Omega_\ell^{(a)}$ gives
\begin{equation}
 H_\ell^{\mathrm c,a}=\widetilde H_{\ell+1}^{(a)}\Omega_\ell^{(a)},
 \qquad
 d_\ell^{(a)}=\mathrm{Trim}_\ell\!\left(
 H_\ell^{\mathrm c,a}-F_\ell^{\rm r}(H_\ell)\right),
 \label{eq:recentinterface}
\end{equation}
where $\mathrm{Trim}_\ell$ applies the same norm and row-budget restrictions as the original proposal. Identity adapters are used when the coordinates match; otherwise, depth correspondence and adapters are fixed using development/source data before the audit. TSA supplies structure-aligned representations, whereas STEM supplies representations generated with its history-dependent adapters. Their internal objectives choose proposals; only RTT's checker authorizes a served change. For every realized history, sound incumbent-tail bounds therefore yield the same deterministic contract. Statistical guarantees still require the evaluation units specified in Section~\ref{sec:recentmethods}.

A native method can also refine its classifier or output distribution. Its complete native output is retained as a separate comparator: the exported-state variant is explicitly named \emph{TSA-derived} or \emph{STEM-derived}, not represented as an exact reproduction of the entire method. The experiment records the loss of utility introduced by state export, depth mapping, and trimming before attributing any improvement to admission. Section~\ref{sec:recentmethods} reports these comparisons separately from the original no-prefix audit.

\subsection{Reference-tail contracts and one-sided pricing}\label{sec:contracts}
For fixed labels, let $V_j(H)$ be the terminal cross-entropy of the incumbent's continuation from state $H$ at depth $j$. The key comparison is between $V_{\ell+1}(\bar H_{\ell+1}+v_\ell)$ and $V_{\ell+1}(\bar H_{\ell+1})$: both use the incumbent tail, even though their starting state belongs to the adapted execution. To make every such comparison valid, RTT constructs row-wise tubes $\mathcal T_j=\{X:\|X_i-H_{j,i}\|\le r_{j,i}\}$ around the executed states. These tubes enclose the incumbent trajectory and every incumbent continuation started along an executed intervention segment.

For smooth transports, bounds $L_{j,i}$ on first derivatives and $\mu_{j,i}$ on curvature give row contracts $|DV_j[Y]|\le\sum_i\Lambda_{j,i}\|Y_i\|$ and $|D^2V_j[Y,Y]|\le\sum_i\Xi_{j,i}\|Y_i\|^2$. For~\eqref{eq:incumbent}, their first-order backward recursion and the tube's forward recursion are
\begin{equation}
 \begin{aligned}
 \Lambda_j&=(1-\alpha)\big[(1-\tau)\Lambda_{j+1}
                  +\tau P^\top(L_j\odot\Lambda_{j+1})\big],\\
 r_{j+1,i}&=(1-\alpha)\big[(1-\tau)r_{j,i}
                  +\tau L_{j,i}(Pr_j)_i\big]+\|v_{j,i}\|,
 \end{aligned}\label{eq:contracts}
\end{equation}
with $\Lambda_{T,i}=\Gamma_\Delta w_i$ and $\Xi_{T,i}=\Gamma_\Delta^2w_i/4$, where $\Gamma_\Delta=\max_{c,c'}\|\mathbf w_c-\mathbf w_{c'}\|_2$ is the maximum pairwise distance between the effective class-head vectors, including the frozen temperature. We use $\|v\|_\Xi^2=\sum_i\Xi_{\ell+1,i}\|v_i\|_2^2$ at depth $\ell$. Proposition~\ref{prop:contracts} gives the second-order recursion; Lemma~\ref{lem:tube} proves enclosure. For ReLU transports, interval generalized-Jacobian enclosures replace unsupported smooth Taylor bounds (Lemma~\ref{lem:relu}).

For tanh, a pre-activation tube radius $\varrho$ bounds the change of unit $u$ by $\varrho\|W_{:,u}\|_2$. Bounding the derivative at the smallest absolute pre-activation within that interval yields $L_{j,i}$, refined by a rank-8 spectral bound and capped by $\|W\|_2$. These are tube-wide bounds, not derivatives sampled only on the executed path. The distinction is essential: a local factor that fails on a counterfactual continuation cannot certify the served prediction.

RTT uses increasingly informative charges as the remaining tail shortens. At contract depths, it charges $\sigma_\ell(v)=\sum_i\Lambda_{\ell+1,i}\|v_i\|$. This is conservative because norms discard the sign of a useful movement. Within the last $b=4$ depths, a tangent basis $U_\ell$ of width $2K$ also yields
\begin{equation}
 \mathrm{env}_\ell^+(t)
 =\sum_iw_i\max_c\langle p_{\ell,i}-e_c,(U_\ell t)_i\rangle
       +\tfrac12\|D_\ell t\|_\Xi^2,
 \label{eq:envelope}
\end{equation}
where $p_{\ell,i}$ is the reference-tail prediction from $\bar H_{\ell+1}$, not necessarily the original incumbent endpoint after earlier interventions. At the last depth, the affine head permits the exact charge
\begin{equation}
 \begin{aligned}
 C^+(t)&=\sum_iw_i\big[\LSE(z_i+\Delta z_i(t))-\LSE(z_i)
                              -\min_c\Delta z_{ic}(t)\big]\\
       &=\sum_iw_i\Dinf(p_{\ell,i}\|p_i^t).
 \end{aligned}\label{eq:cplus}
\end{equation}
Both the contract and the envelope are valid upper bounds. Their minimum is valid too, but its budget sublevel set is a \emph{union} of convex sets. RTT therefore solves the two window branches separately and retains the better feasible displacement, rather than treating their union as convex (Lemma~\ref{lem:twobranch}).

\subsection{Learning to allocate certified budget}\label{sec:optimization}
The field head predicts $\hat g_\ell$ on the executed support, inducing candidate slopes $\hat a_t=D_\ell^\top\hat g_\ell$ and physical-coordinate slopes $\hat a_x=Q_\ell^\top\hat g_\ell$. A block-diagonal covariance head $\hat\Sigma_x$ and a calibrated radius $\kappa_\ell$ describe uncertainty in these slopes. Within each feasible branch, RTT minimizes
\begin{equation}
 \Phi_\ell(x)=\hat a_x^\top x+\kappa_\ell\|\hat\Sigma_x^{1/2}x\|
 +\frac{(1+\rho_D)\|v\|_\Xi^2+\rho_x\|v\|^2}{2\omega},
 \qquad v=Q_\ell x,
 \label{eq:robust}
\end{equation}
subject to the candidate hull and remaining call budget. Equation~\eqref{eq:robust} uses a certified quadratic remainder at smooth, nonterminal depths. At the final depth it is replaced by
\[
 \Phi_{T-1}(x)=\hat a_x^\top x+\kappa_{T-1}\|\hat\Sigma_x^{1/2}x\|
 +\frac{\varrho_{T-1}(v)+\mathcal R_{T-1}(v)}{\omega},\qquad v=Q_{T-1}x,
\]
where $\varrho_{T-1}(v)=\sum_iw_i\mathrm{KL}(p_{T-1,i}\|\mathrm{softmax}(z_i+\Delta z_i(v)))$ is the exact label-independent remainder and $\mathcal R_\ell(v)=(\rho_D\|v\|_\Xi^2+\rho_x\|v\|_2^2)/2$. A nonsmooth tail requires its own validated remainder model, rather than a smooth quadratic bound across a kink. Here $0<\omega<1$ is a safety factor and $\rho_x>0$ makes the objective strictly convex in the physical displacement. The uncertainty term is the support function of the slope ellipsoid: on the coverage event it upper-bounds the unknown linear loss change for every feasible direction. The curvature term bounds the nonlinear remainder, and the regularizer makes a redundant candidate parameterization immaterial.

At a contract depth, a block whose candidate slopes are all nonnegative can be screened out: setting that block to zero cannot increase its linear contribution, the joint uncertainty norm, the displacement penalty, or the contract charge. The linear, quadratic, and contract terms separate by block, but the single ellipsoidal norm in~\eqref{eq:robust} couples their magnitudes. The remaining convex program therefore retains that joint norm; its scalar budget multiplier has nonincreasing spend and can be located by bisection, with a coupled convex solve at each multiplier. Each window branch has a unique optimal displacement; the union need not. A deterministic tie-break expressed in physical state coordinates handles equal branch objectives without introducing candidate-order dependence. Minimum-norm candidate coefficients are recovered only after choosing the physical step.

A provisional tube restricts the proposed row movements during execution; the final checker enforces the row guarantee. With $n_{\rm left}$ open depths remaining and a registered growth proxy $\hat\Pi_\ell$, the row allowance is
\begin{equation}
 \rho_{\ell,i}=\frac{(\Hrow/\Gamma_\Delta-\hat\Pi_\ell\tilde r_{\ell,i})_+}
                         {\hat\Pi_{\ell+1}n_{\rm left}}.
 \label{eq:rowallowance}
\end{equation}
For an open depth, the proxy denominator is positive and $n_{\rm left}\ge1$. Candidates are scaled so that $\bar c_{\ell k}\|d_{\ell,i}^{(k,m)}\|\le\rho_{\ell,i}$. A zero cap or zero direction requires no division and yields zero displacement. This preserves convexity of the resulting hull and reserves room for later steps. The proxy influences gain and fallback frequency, not validity: the final checker recomputes the tube and verifies $\Gamma_\Delta r_{T,i}\le\Hrow$ for every served row.

\subsection{Independent certification and whole-call release}\label{sec:release}
Algorithm~\ref{alg:rtt} separates proposing, executing, and certifying. After execution, the checker reconstructs the tube, interval factors, contracts, charges, and robust descent models using certified remainders. It checks three predicates: $\sum_\ell\mathrm{charge}_\ell\le H^+$, $\max_i\Gamma_\Delta r_{T,i}\le\Hrow$, and $\Phi_\ell^{\rm cert}\le0$ at every accepted depth. A failure triggers tighter per-row spectral factors; a persistent failure triggers whole-call fallback. Row-wise mixing after a failed call is not substituted for this selector, because the proof accounts for the served policy as a whole.

The real-arithmetic guarantees extend to floating-point execution only when all numerical residuals are enclosed. The checker described in Appendix~\ref{app:cost} uses directed rounding, bounds the adapted execution residual and the incumbent continuation's float32 error, and charges both against the budgets. Thus a poor proposal or inaccurate utility head can reduce gain or induce fallback, but it does not invalidate a correctly implemented deterministic certificate.

\begin{algorithm}[t]
\caption{RTT for one deployment call. All displacement optimizations include the post-trim hull.}\label{alg:rtt}
\small
\begin{algorithmic}[1]
\Require $H_0$; incumbent; proposal and adapters; caps; heads and radii; $H^+,\Hrow,b$
\State $\Brem\gets H^+$; initialize provisional tube; $\ell_0\gets\mathrm{none}$
\For{$\ell=0,\ldots,T-1$}
 \State $\bar H_{\ell+1}\gets F^{\rm r}_\ell(H_\ell)$; initially set $v_\ell\gets0$
 \If{$\ell$ is open and $\Brem>0$}
  \State Form candidates; apply norm and row-budget trims; form $D_\ell=Q_\ell\Theta_\ell$
  \State Compute support features, slopes, and covariance for $\Phi_\ell$
  \If{$\ell=T-1$}
   \State Minimize $\Phi_\ell$ subject to $C^+\le\Brem$, using the exact remainder
  \ElsIf{$\ell\ge T-b$}
   \State Solve branches $\sigma_\ell\le\Brem$ and $\mathrm{env}^+_\ell\le\Brem$ separately
   \State Choose the better feasible physical displacement; resolve ties deterministically
  \Else
   \State Screen nonnegative-slope blocks; solve under $\sigma_\ell\le\Brem$ by bisection
  \EndIf
  \State $v_\ell\gets Q_\ell x^\star$; subtract the provisional charge from $\Brem$
  \If{$v_\ell\ne0$ and $\ell_0=\mathrm{none}$}
   \State Cache $H_\ell$ and set $\ell_0\gets\ell$
  \EndIf
 \EndIf
 \State $H_{\ell+1}\gets\bar H_{\ell+1}+v_\ell$; update the provisional tube
\EndFor
\State Recompute the executed-trajectory certificate with outward-rounded arithmetic
\If{damage, row, and descent predicates all hold}
 \State \Return $p^\pi$ \Comment{R1: first-pass release}
\EndIf
\State Re-certify using tighter per-row factors
\If{all predicates hold}
 \State \Return $p^\pi$ \Comment{R2: re-certified release}
\EndIf
\State Run the incumbent tail from cached $H_{\ell_0}$; if no step was taken, use the incumbent output
\State \Return $\pr$ \Comment{F: whole-call fallback}
\end{algorithmic}
\end{algorithm}

\Needspace{8\baselineskip}
\section{Analytical Guarantees and Statistical Assurance}\label{sec:theory}
The analysis answers three distinct questions: what every served call is guaranteed to satisfy, how much gain is possible under that protection, and which probability statements the calibration and audit data support. Proofs and implementation-specific lemmas are collected in Appendix~\ref{app:proofs}.

\subsection{Trajectory accounting and deterministic protection}\label{sec:deterministic}
The reference-tail construction makes internal effects telescope even though the executed trajectory differs from the incumbent's. For fixed labels, write
\begin{equation}
 \begin{aligned}
 f_\ell(v_\ell)&=V_{\ell+1}(\bar H_{\ell+1}+v_\ell)
                         -V_{\ell+1}(\bar H_{\ell+1}),\\
 J(p^\pi;y)-J(\pr;y)&=\sum_{\ell=0}^{T-1}f_\ell(v_\ell).
 \end{aligned}
 \label{eq:telescoping}
\end{equation}
This identity follows from $V_\ell=V_{\ell+1}\circ F^{\rm r}_\ell$ and the common initial state. It compares each step through the same incumbent continuation, rather than using a gradient of the final adapted policy as an interchangeable target.

\begin{theorem}[Certified internal admission]\label{thm:served}
Assume that the checker validates continuation enclosures and loss remainders for the executed trajectory, including numerical allowances when certifying a floating-point incumbent. For any proposal and episode, Algorithm~\ref{alg:rtt} has the following properties.

\noindent(i) The served prediction satisfies $\Dinf(\pr_i\|\ps_i)\le\Hrow$ for every row and $\sum_iw_i\Dinf(\pr_i\|\ps_i)\le H^+$; hence $d^{\rm own}\le H^+$ for every label vector.

\noindent(ii) On a call released with the adapted output, $d^{\rm own}=-\mathcal D_T+\mathcal P_T-\mathcal S_T$, where $\mathcal D_T\ge0$ is certified robust descent, $\mathcal P_T\ge0$ is the realized slope shortfall, and $\mathcal S_T\ge0$ contains remainder, safety-factor, and regularization slack. On fallback, $d^{\rm own}=0$.

\noindent(iii) If every accepted-depth slope lies in its uncertainty set and the certified remainder model applies, then $\mathcal P_T=0$ and $d^{\rm own}\le0$. For a prespecified member with sequential calibration as in Lemma~\ref{lem:seqcal}, this event has marginal probability at least $1-\delta_{\rm cov}$ over calibration and an exchangeable test episode. A data-selected member requires fresh post-selection calibration or a simultaneous guarantee across the selectable members before the same probability statement can be asserted for the selected procedure.
\end{theorem}

Each valid charge bounds $f_\ell$ uniformly over labels. Summing in~\eqref{eq:telescoping} therefore proves the call budget without evaluating the original incumbent endpoint on a call released with the adapted output. The tube also implies $\|H_{T,i}-H^{\rm r}_{T,i}\|\le r_{T,i}$. For the affine head,
\begin{equation}
 \Dinf(\pr_i\|p_i^\pi)
 \le\osc_c\big(Z(H_T)_{ic}-Z(H_T^{\rm r})_{ic}\big)
 \le\Gamma_\Delta r_{T,i}\le\Hrow,
 \label{eq:rowproof}
\end{equation}
which proves the floor. This is a deterministic damage guarantee, not a promise of improved accuracy or of nonincreasing cross-entropy. A positive damage budget permits some loss increase. At the tolerance point $H^+\le\zeta$, however, the event $d^{\rm own}>\zeta$ is impossible by construction.

\subsection{Conditional gain and an analytical floor projection}\label{sec:geometry}
Let $\mathcal O$ contain the label-free deployment information and define the conditional label law $\pB_i=\Pr(y_i\mid\mathcal O)$. Write $\mathcal G=-\E[d^{\rm own}\mid\mathcal O]$, $r_{ic}=\pB_{ic}/\pr_{ic}$, and $R_i=\max_c r_{ic}$. Incumbent probabilities are strictly positive; expectations condition on the relevant frozen model and deployment information.

\begin{proposition}[Gain, price, and floor geometry]\label{prop:price}
(i) $\mathcal G=\sum_iw_i[\mathrm{KL}(\pB_i\|\pr_i)-\mathrm{KL}(\pB_i\|\ps_i)]$. If $\pr=\pB$, no $\mathcal O$-measurable rule has positive conditional gain.

\noindent(ii) With $\mathrm{TV}_{\max}=\max_i\mathrm{TV}(\pB_i,\pr_i)$, $\mathcal G\le2\mathrm{TV}_{\max}\max(S^+,S^-)$, with the factor 2 sharp.

\noindent(iii) $\mathcal G\le\max_i(R_i-1)S^+$, and the row-wise rate $R_i-1$ is attainable to first order.

\noindent(iv) The floor set is~\eqref{eq:floor} and implies the corresponding class ceiling. For $h>0$, the minimizer of $\mathrm{KL}(q\|p)$ over $p\in\Fl_h(\pr_i)$ is
\begin{equation}
 p'_c=\max\{e^{-h}\pr_{ic},q_c/\lambda\},
 \qquad \sum_c\max\{e^{-h}\pr_{ic},q_c/\lambda\}=1.
 \label{eq:projection}
\end{equation}
When $q_c>0$ for every class, the minimizer of $\mathrm{KL}(p\|q)$ is the same. At $h=0$, the feasible set is the singleton $\{\pr_i\}$.
\end{proposition}

The projection is an analytical active-set solution, not a new network. For floor-active classes $A$, the normalizer is $\lambda=\sum_{c\notin A}q_c/(1-e^{-h}\sum_{c\in A}\pr_{ic})$; sorting probability-to-floor ratios or solving the scalar normalization equation identifies the consistent active set. Strict positivity of $q$ is needed for a finite reverse-KL projection because the floor forces positive mass on every class.

Part (i) identifies the source of useful label-free adaptation: information in $\mathcal O$ that corrects conditional incumbent error. Parts (ii)--(iii) explain the value of one-sided pricing. A symmetric certificate charges large improvements as well as damage, while RTT only limits the latter. For a binary incumbent that assigns probability $0.99$ where the conditional probability is $0.7$, $R_i-1=29$; for probabilities $0.6$ and $0.8$, it is approximately $0.33$. The relevant quantity is relative underestimation of a plausible class, not an unconditional claim that every shifted input is overconfident.

\subsection{The exact frontier and architecture-constrained efficiency}\label{sec:frontier}
Define the best row gain under budget $h$ by
\begin{equation}
 G_i(h)=\mathrm{KL}(\pB_i\|\pr_i)
        -\min_{p\in\Fl_h(\pr_i)}\mathrm{KL}(\pB_i\|p).
 \label{eq:rowfrontier}
\end{equation}
Let $U_{\rm out}$ denote the optimum over all outputs satisfying the row and call budgets. Let $U_{\rm reach}$ be the optimum over interventions reachable through RTT's candidate hulls with exact endpoint damage, $U_{\rm cert}$ the optimum when certified charges are enforced, and $U_{\rm RTT}$ the gain of the executed policy. These quantities use the same conditional label law and budgets; population versions average over deployment information.

\begin{theorem}[Frontier, efficiency, and nested opportunities]\label{thm:frontier}
(i) $G_i$ is concave and nondecreasing, with $G_i(0)=0$, $G_i'(0^+)=R_i-1$, and saturation once $h\ge\Dinf(\pr_i\|\pB_i)$. The saturation threshold is interpreted in the extended real sense; it can be infinite when $\pB_i$ has a zero component. The output-space frontier is
\begin{equation}
 U_{\rm out}=\max_{0\le h_i\le\Hrow,\ \sum_iw_ih_i\le H^+}
                     \sum_iw_iG_i(h_i).
 \label{eq:waterfill}
\end{equation}
An optimum is obtained by floor-projecting $\pB_i$ at water-filled budgets: positive-weight, unsaturated interior rows have equal marginal gain $G_i'(h_i)=\lambda$; at a nondifferentiable point, the equivalent condition uses a supergradient.

\noindent(ii) At depth $\ell$, let $\delta z(v)=J_\ell v$ be the first-order terminal-logit change and $p_{\ell,i}$ the reference-tail prediction. Define
\begin{equation}
 \begin{aligned}
 g_\ell(v)&=\sum_iw_i(\pB_i-p_{\ell,i})^\top\delta z_i(v),\\
 c_\ell(v)&=\sum_iw_i[p_{\ell,i}^\top\delta z_i(v)-\min_c\delta z_{ic}(v)].
 \end{aligned}\label{eq:efficiency}
\end{equation}
Then, over directions with positive price,
$\eta_\ell=\sup_{v\in\mathcal K_\ell,\ c_\ell(v)>0}[g_\ell(v)]_+/c_\ell(v)
\le\max_i(R_{\ell,i}-1)$, where $R_{\ell,i}=\max_c\pB_{ic}/p_{\ell,ic}$. The piecewise-linear denominator admits an epigraph formulation and a linear-fractional reduction on the polyhedral candidate set.

\noindent(iii) The exact optima satisfy $U_{\rm out}\ge U_{\rm reach}\ge U_{\rm cert}\ge U_{\rm RTT}$.
\end{theorem}

The floor geometry and marginal equalization are illustrated in Figure~\ref{fig:floor}(a--b); panel (c) reports their experimental counterpart. The first result supplies a common standard for comparing internal and output-space methods. Its initial slope follows from moving probability mass toward a class maximizing $\pB_{ic}/\pr_{ic}$. The water-filling condition then allocates budget until marginal gains equalize or a row saturates. The second result measures how much of this ideal direction the incumbent's continuation and proposal hull can realize. It motivates placing trust where useful directions are available and their certified prices are small. The third result separates architectural restriction, certificate conservatism, and policy suboptimality at the level of exact feasible-set optima. A numerical solver need not attain these optima, so computed diagnostics must not be presented as a proved equality or an exact global gap decomposition.

A complementary local statement makes the learning terms explicit. At a fixed depth and \emph{fixed convex branch}, condition on information $\mathcal O_\ell$ that includes the feasible geometry and a positive-definite physical-coordinate majorizer $\bar\Xi_\ell$. Let $m_\ell=\E[a_\ell\mid\mathcal O_\ell]$, let $A_\ell^{\rm cert}$ be the corresponding quadratic opportunity with the conditional slope given all of $\mathcal O$, and let $I_\ell\ge0$ be the information gap. With uncertainty penalty $\chi_\ell$ evaluated at the conditional-mean minimizer and certified objective suboptimality $\Upsilon_\ell\ge0$, the branchwise comparison is
\begin{equation}
 \begin{aligned}
 \E[\mathrm{gain}_\ell\mid\mathcal O_\ell]\ \ge\ &
 \E[A_\ell^{\rm cert}\mid\mathcal O_\ell]-I_\ell-\chi_\ell\\
 &-\frac12\left(\|m_\ell-\hat a_\ell\|_{\bar\Xi_\ell^{-1}}
                       +\sqrt{2\Upsilon_\ell}\right)^2.
 \end{aligned}\label{eq:localgap}
\end{equation}
Appendix~\ref{app:proofs} derives this inequality under its stated conditioning and convexity assumptions. It does not identify $\sum_\ell A_\ell^{\rm cert}$ with the global $U_{\rm cert}$, and a comparison spanning the nonconvex union of window branches needs separate branch-selection accounting. On real data, the conditional posterior is estimated by cross-fitting, so reported frontier shares are plug-in diagnostics rather than certified fractions of the unknown Bayes frontier.

\subsection{When internal intervention helps, and how certificates compose}\label{sec:separation}
\begin{proposition}[Separation under explicit information restrictions]\label{prop:separation}
(a) Every RTT output can be represented by some floor-feasible output rule. There is no unrestricted superiority over all such rules.

\noindent(b) An inert row rule has the form $\phi(\pr_i,\pu_i,h_i)$ with $\phi(p,p,h)=p$. Mixing, switching, and floor projection of the raw update are inert. When two raw internal effects cancel so that $\pu_i=\pr_i$, these rules cannot change that row. There exist feasible instances where admitting only one internal effect instead gains \aSepGain\ nats.

\noindent(c) A corrector restricted to incumbent terminal states and proposal rows within $k$ hops cannot respond differently to two instances that are identical in those inputs but differ only in unseen evidence. An RTT intervention can transmit such evidence through the incumbent tail when the transport supports the path and $T-\ell-1\ge\mathrm{dist}_i$. Zero conditional gain for the restricted corrector additionally requires that the incumbent is already Bayes-optimal given that corrector's information.
\end{proposition}

These are information-restricted separations, not claims that a wider corrector or a different information set can never match RTT. The cancellation example in Figure~\ref{fig:sepinst} is exact; the empirical near-cancellation group in Section~\ref{sec:experiments} is defined by a numerical tolerance and is interpreted accordingly. The hop condition counts transitions after $v_\ell$ is injected into $H_{\ell+1}$, which avoids crediting the intervention with a nonexistent extra propagation step.

RTT and output correction are also complementary. For any intermediate prediction $p'$, $\Dinf(\pr_i\|p_i'')\le\Dinf(\pr_i\|p_i')+\Dinf(p_i'\|p_i'')$. Therefore, after an adapted output passes certification, a corrector can use the certified row allowance $\Hrow-\Gamma_\Delta r_{T,i}$ and remaining call allowance $H^+-\sum_\ell\mathrm{charge}_\ell$ relative to that output. After whole-call fallback, the served output is $\pr$ and its spent budgets are zero; a subsequent corrector must use allowances computed from that served state, not from the rejected trajectory. The original deterministic floor and call budget are preserved (Remark~\ref{rem:compose}). Statistical performance and coverage claims for the composed policy must nevertheless be established for that policy, not inherited automatically from RTT alone.

\subsection{Sequential calibration and deployment-level assurance}\label{sec:assurance}
Slope calibration proceeds from early to late depths. At depth $\ell$, fresh episodes are executed with all upstream radii fixed, and the joint physical-coordinate Mahalanobis error $s_\ell=\|\hat\Sigma_{x,\ell}^{-1/2}(a_{x,\ell}-\hat a_{x,\ell})\|$ is recorded. The covariance used for inversion must be positive definite on the nonzero physical span. A rank-zero step has score zero; an infinite radius permits only directions with zero uncertainty norm, including the zero step. The radius is the $\lceil(1-\delta_\ell)(n_r+1)\rceil$-th order statistic, interpreted as $+\infty$ if that index exceeds $n_r$. On $s_\ell\le\kappa_\ell$, Cauchy--Schwarz yields the robust linear bound in~\eqref{eq:robust}. A union bound over $\sum_\ell\delta_\ell=\delta_{\rm cov}$ gives the fixed-member trajectory statement. This guarantee is marginal over the calibration sample and an exchangeable test episode, not conditional coverage for every realized calibration set.

Policy selection is a separate operation. The reported protocol compares several calibrated members before freezing one for a fresh audit. The deterministic certificate remains valid after selection. The nominal pre-audit coverage bound for the selected member, however, requires a simultaneous construction over selectable members or fresh calibration after selection. We therefore distinguish the fixed-member calibration theorem from the independent final-audit guarantees rather than using one as a substitute for the other.

For a policy frozen before $N$ independent episodes, the regression indicator is Bernoulli and a one-sided Clopper--Pearson bound is
\begin{equation}
 u=\mathrm{Beta}^{-1}(1-\delta;N_\zeta+1,N-N_\zeta),
 \label{eq:cp}
\end{equation}
with $u=1$ when $N_\zeta=N$. Theorem~\ref{thm:assurance} combines this with an intersection--union release test, fixed-sequence selection, and a betting lower confidence bound on mean gain. The betting construction needs only the deterministic upper bound $d^{\rm own}\le H^+$ and is anytime-valid under its sampling assumptions. Simultaneous component bounds support explicitly specified reweightings, while Theorem~\ref{thm:node} uses a reserved finite population and uniform label inspection to bound node-level negative flips. These statements concern different units and should not be interchanged.

Finally, a KL-robustness curve is conditional on the new deployment law belonging to the stated divergence ball. Even the exact divergence of label-free score distributions can only lower-bound the full episode-law divergence; its empirical estimate cannot establish membership in that ball. Accordingly, the temporal transfer cohort is audited directly. The full sampling, selection, multiplicity, and event-count protocol is given in Appendix~\ref{app:stats}.
\FloatBarrier

\section{Experimental Results}\label{sec:experiments}
The experiments evaluate whether RTT's certified internal decisions provide useful gain beyond the strongest matched output-space alternative, and whether the gain occurs in the regimes identified by the analysis. We organize the evidence around the primary comparison, mechanism-level diagnostics, deployment cost, and transfer across proposals and tasks. The tables and figures distinguish exact certificates, audited probabilities, and posterior-based oracle estimates. The main audit and the two additional conference-method comparisons in Section~\ref{sec:recentmethods} are reported separately because their evaluation histories differ.

\subsection{Protocol, baselines, and evaluation units}\label{sec:protocol}
The arxiv incumbent achieves \vIncAcc\% accuracy on the official test split. Its temperature is fitted using held-out 2018 labels and then frozen. Training uses labels through 2017; the 2019 cohort is divided into development and final-audit halves. The final audit contains 2,048 calls from an equal-weight mixture of clean, degree-shift, real-cache, feature-noise, and edge-dropout populations, with a separate 4,096-call audit for each component \citep{ogb,good,gas}. Severe stress, which rewrites 30\% of rows at depth $T/2$, is audited separately and excluded from the headline mixture. A reserved 2020 cohort supplies the temporal-transfer evaluation.

The registered configuration has $\Hrow=1$ nat, $p^\star=2\%$, $\delta=0.05$, $K=8$, $b=4$, and 4,096 scored nodes per call. The selected member uses $H^+=25\zeta=\vSelHp$ and nominal fixed-member calibration allocation $\delta_{\rm cov}=\vSelDcov$. The primary comparison and eight Holm-corrected secondaries are specified before the final audit. Confidence intervals are episode-level and paired when methods share episodes; the node-population bootstrap is reported separately rather than folded into the betting guarantee (Appendix~\ref{app:stats}).

The strongest one-pass corrector runs the incumbent once and uses its terminal states, predictions, and proposal candidate rows. It receives the same training labels, head capacity, and number of gradient steps as RTT, and projects its predicted distribution onto the same floor using allocated row budgets. Its two-hop field is selected on development calls among two, four, and eight hops. Thus it shares G1--G3 with RTT and is the relevant low-cost baseline, not an unprotected raw update. Additional comparisons include exact-geometry output rules requiring two rollouts, a matched trust rule (TR) using the same candidate hull and checker, training-time anti-regression methods, and direct test-time adaptation. Table~\ref{tab:frontierfull} provides the complete comparison set.

\subsection{Primary comparison and the probability-floor frontier}\label{sec:primary}
Table~\ref{tab:frontier} reports a mean gain of \vRttI\ \vRttCI\ in units of $10^{-3}$ nats for RTT, with betting lower bound \vRttLB, call-level regression-rate upper bound \vRttU, and negative-flip-rate upper bound \vRttNfrb\ on the uninspected reserved-node population. The strongest corrector obtains \vOpcI\ at a call-level bound of \vOpcU. The primary paired difference is \vPrimD\ \vPrimCI, while the corrector is cheaper: \vLatOpc$\times$ versus \vLatOwn$\times$ incumbent latency. The raw update obtains 3.9 but has a regression-rate upper bound of \vRawU, illustrating why mean gain alone is insufficient. The reported secondary comparisons survive Holm correction (Table~\ref{tab:frontierfull}).

\begin{table}[t]\centering
\caption{Primary arxiv comparison on 2,048 final-audit mixture episodes; gain and its intervals are in $10^{-3}$ nats. $\Delta$ is RTT minus the row method; $\dagger$ marks the primary comparison. $N_\zeta$ is the regression count, $u$ the one-sided call-level Clopper--Pearson bound, $u_{\rm node}$ the reserved-node bound, and LB the betting lower bound on mean gain. Latency and FLOPs are relative to one incumbent pass. The oracle row is a cross-fitted posterior estimate, not a deployable method or a certified true-Bayes optimum.}\label{tab:frontier}
\TabFrontier
\end{table}

Figure~\ref{fig:floor} connects the geometry to this comparison. Its first two panels illustrate the exact row-floor projection and water-filling solution, while the third reports the budget sweep. RTT attains \vShareUout\ of the estimated output-space frontier on the mixture. This share depends on the cross-fitted posterior used to estimate the label law; it is not a guarantee of recovering that fraction of the unknown Bayes-optimal gain. Adaptive mixing and raw-output floor projection obtain \vMixI\ and \vFloI, respectively, and both require two rollouts. Matched TR obtains 5.8, leaving a \vMtrD\ gain increment for RTT's reference-tail optimization at the same hull and certificate.

\begin{figure}[t]\centering
\includegraphics[width=\linewidth]{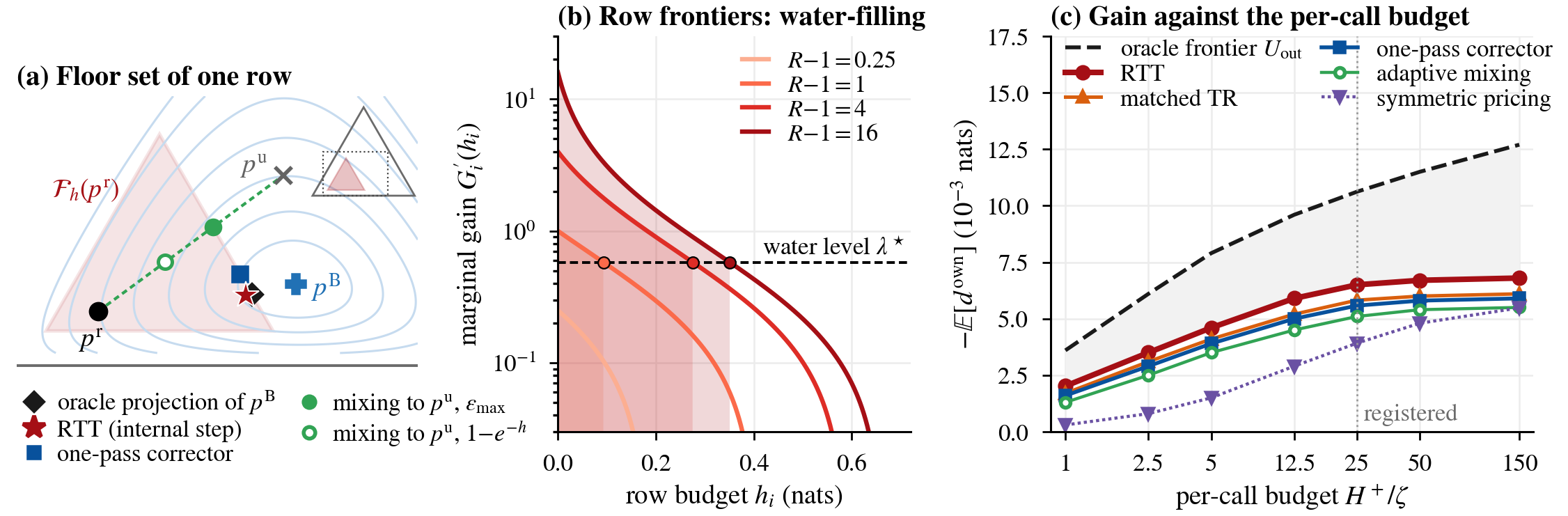}
\caption{Floor geometry and its empirical frontier. (a) The row floor at $h=\aDemoH$: the Bayes-posterior projection gains \aDemoGorc\ nats; mixing toward the raw update gains \aDemoGmix\ at its largest feasible weight and \aDemoGmixu\ at the universal weight $1-e^{-h}$. The RTT point is schematic, not an additional measured result. (b) Water-filling equalizes marginal gains across unsaturated rows. (c) The arxiv budget sweep compares RTT, one-pass correction, two-rollout rules, matched TR, and symmetric pricing; the dashed frontier uses the cross-fitted posterior.}\label{fig:floor}
\end{figure}

\subsection{Where the gain comes from and what each mechanism contributes}\label{sec:mechanisms}
The row-group analysis in Figure~\ref{fig:where}(d) tests the separation mechanisms rather than only the aggregate score. Near-cancellation rows make up \vSepShareCancel\ of the audit mixture: the raw and incumbent outputs differ by less than $10^{-3}$ in $\ell_1$ distance despite internal candidate movement. RTT contributes \vSepRttCancel\ to mixture gain on these rows, while the strongest corrector contributes at most \vSepOpcCancel; inert rules have negligible reported gain. The empirical tolerance is not the exact equality assumed by Proposition~\ref{prop:separation}(b). On the \vSepShareFar\ of rows whose tangent influence is predominantly beyond two hops, RTT contributes \vSepRttFar\ against the corrector's \vSepOpcFar. On the remaining rows, the corrector leads, \vSepOpcOther\ versus \vSepRttOther. Four- and eight-hop correctors obtain 5.3 and 5.0 overall. These results support a regime-dependent internal advantage rather than universal superiority. Applying the corrector to RTT's unused certified budget increases the mean to \vRttcI.

\begin{figure}[t]\centering
\includegraphics[width=\linewidth]{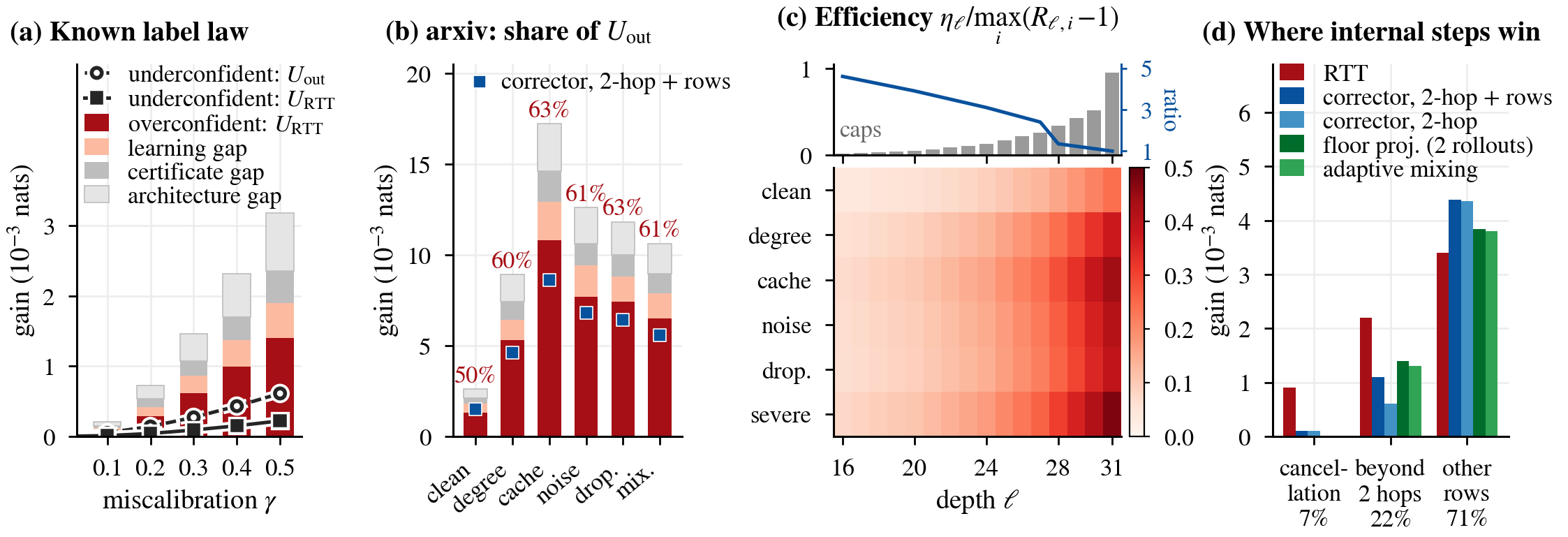}
\caption{Mechanism diagnostics. (a) Planted partitions with a known label law: the analytic output frontier and the reported numerical reachability, certification, and policy levels. The numerical internal levels are optimization diagnostics, not proofs that their global suprema were attained. (b) Posterior-based frontier shares on arxiv. (c) Architecture-constrained efficiency by depth, together with cap placement and certified-to-exact price ratios. (d) Contributions from near-cancellation, beyond-two-hop, and remaining rows. Appendix~\ref{app:gap} defines the groups and the limitations of the numerical gap calculation.}\label{fig:where}
\end{figure}

With the known planted label law, the calibrated incumbent admits no conditional gain, as Proposition~\ref{prop:price}(i) predicts. At miscalibration $\gamma=0.3$, RTT reaches \vPPshare\ of the output-space frontier. The reported numerical ordering suggests that architectural restriction is the largest source of shortfall, followed by the learned decision rule; it does not establish an exact additive decomposition of globally optimized gaps. On arxiv, estimated shares range from \vShareClean\ on clean calls to \vShareCache\ under the real-cache shift. Efficiency concentrates at late depths, and the median certified-to-exact price ratio reaches \vCertRatioMax\ at earlier contract depths. This identifies certificate conservatism as a constraint on useful early intervention.

The one-at-a-time mechanism swaps are collected in Table~\ref{tab:swaps}. One-sided rather than symmetric pricing changes the reported gain from 3.92 to 6.50 at the same 0.05-nat budget; interval rather than global factors changes it from 4.12 to 6.50. Replacing spend-based caps by certified-efficiency placement adds approximately \vPlaceD, while the row floor costs approximately \vHrowCost\ of gain. Sequential rather than independent per-depth radii costs approximately \vTrajCost\ and supplies the fixed-member trajectory guarantee, subject to the selection distinction in Section~\ref{sec:assurance}. The smaller effects of two-branch optimization, support-function features, and full-loss slope targets are also reported. These swaps are conditional comparisons around one configuration, not additive contributions to a unique total gain.

\subsection{Call size, reliability, and computational cost}\label{sec:deploymentresults}
Figure~\ref{fig:deploy} presents the deployment trade-off. The probability floor is row-specific and independent of call size, but the distribution of average call damage is not. On the development half, the registered member's regression-rate bound increases from \vBatchUlarge\ at 4,096 nodes to \vBatchUsmall\ at 64 nodes. Reselecting a member for 64-node calls restores \vBatchUresel\ with gain \vBatchGresel. At $H^+=\zeta$, the call-level regression event is ruled out deterministically and RTT still gains \vRttzI. The separate node-level audit remains necessary because absence of tolerance exceedance is not absence of negative classification flips.

\begin{figure}[t]
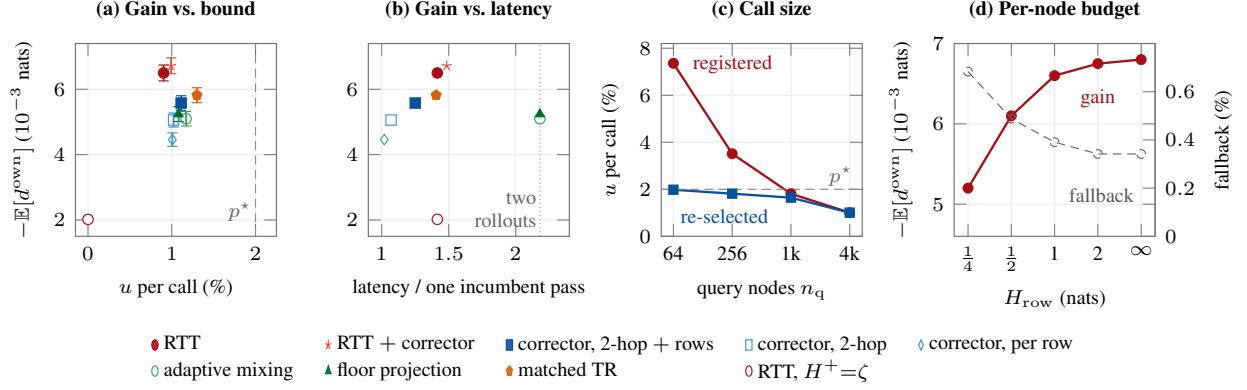
\centering\resizebox{\linewidth}{!}{\FigDeploy}
\caption{Deployment trade-offs. (a--b) Final-audit mean gain against the call-level regression bound and relative latency. Equal-bound markers are offset for visibility; Table~\ref{tab:frontier} gives exact values. The tolerance point is a deterministic statement. (c--d) Development-set effects of call size and row budget, including members reselected per call size and observed fallback rates.}\label{fig:deploy}
\end{figure}

The final-audit release ledger contains 2,035 first-pass releases, six releases after tighter certification, and seven whole-call fallbacks out of 2,048 calls. Thus \vFailRate\ fail the first pass and \vFbRate\ ultimately fall back (Table~\ref{tab:audit}). The cost comparison includes proposal computation, heads, short reference tails, solver, provisional bounds, checker, and fallback under one accounting convention. For the native in-network proposal, RTT costs \vLatOwn$\times$ incumbent latency and 1.36$\times$ its FLOPs, versus 2.18$\times$ for two sequential rollouts. Two rollouts on two devices cost 1.18$\times$ latency but use twice the hardware. For full-network proposals RTT can be slower than two sequential rollouts: with the retrained incumbent, the reported values are 2.24$\times$ and 2.00$\times$. Appendix~\ref{app:cost} provides the full accounting and numerical-checker diagnostics; no universal speed advantage is claimed.

\subsection{Proposal breadth, other tasks, and transfer}\label{sec:breadth}
Across the eight previously evaluated proposals, RTT exceeds the matched one-pass corrector by $0.5$--$0.9$ in units of $10^{-3}$ nats (Table~\ref{tab:class}). The set covers GPR-GNN, Co-GNN, AMP, GTrans, Matcha, retraining, and a LoRA-16 update, in addition to the native proposal. TSA and STEM extend this comparison set under the separate history-aware protocol of Section~\ref{sec:recentmethods}; they are not pooled into the eight-proposal main audit. The appropriate comparison includes the proposal's own cost once; it is not legitimate to omit a full adaptation network when comparing RTT with an output rule.

\begin{table}[t]\centering
\caption{Proposal breadth on arxiv. Gain is in $10^{-3}$ nats and $u$ is the call-level regression-rate upper bound. Raw proposals are norm-trimmed but not admitted by RTT. The corrector receives each proposal's rows; paired intervals compare RTT with that corrector. Latencies include the proposal and are relative to one incumbent pass.}\label{tab:class}
\TabClass
\end{table}

Training-time and test-time neighbors answer a different protection question. Weight interpolation and LoRA scaling selected by learn-then-test obtain 3.4 and 3.2 at one-network cost, while EATA and periodic reset obtain 2.6 and 2.2; the reported configurations fail the statistical release criterion and do not enforce G1--G3 (Table~\ref{tab:neighbors}). In the separate batch-shift stream, raw GTrans has mean damage $+38.0\times10^{-3}$ nats per shifted call, whereas RTT-gated serving reduces it to $+0.4\times10^{-3}$ with 6.1\% fallback. A positive residual damage is consistent with the stated budget and is not described as a no-regression result.

Table~\ref{tab:breadth} extends the evaluation to APPNP, GCNII, GRAND, and shallow GraphSAGE/GCN incumbents; GOOD-Arxiv degree and time shifts; Twitch, Facebook-100, and Elliptic shifts; ogbn-products; ImageNet-C; and Folktables \citep{good,eerm,twitch,fb100,elliptic,ogb,imagenetc,folktables}. On the graph configurations, RTT reaches 60\%--73\% of the estimated oracle frontier, with call-level bounds up to 1.30\%. Its lead on shallow incumbents shrinks to about $0.2\times10^{-3}$ nats, consistent with a shorter continuation offering less propagation advantage. On ogbg-molpcba, the custom reserved-scaffold protocol improves the deployed GIN-style incumbent from 27.7 to 28.6 AP; the separately audited official test result is 28.5 AP (Table~\ref{tab:mol}; \citealp{xu2019gin}). On ImageNet-C, adapting the last four ResNet-50 blocks with Tent proposals increases top-1 accuracy from 40.1\% to 45.6\% under the floor (Table~\ref{tab:breadth}; \citealp{resnet,tent}).

The reserved temporal cohort is evaluated directly rather than certified from a label-free shift estimate. Its call-level bound is 1.18\%, with a betting lower bound of $1.7\times10^{-3}$ nats on mean gain (Table~\ref{tab:stats}). Figure~\ref{fig:statistics} separates pointwise, simultaneous, node-level, and conditional-divergence results. Across three incumbents and eight training seeds each, the reported selection procedure passes 22 of 24 times for RTT and 24 of 24 for the corrector; the failed selections revert to the incumbent (Figure~\ref{fig:robust}). Separate drift streams evaluate monitoring and recalibration (Figure~\ref{fig:service}), while Figure~\ref{fig:mechanism} compares certificate tightness and verifier cost. Together, these experiments delimit where RTT's extra computation yields useful certified gain and where simpler correction is competitive.
\FloatBarrier

\subsection{Two recent conference methods: matched TSA/STEM evaluation}\label{sec:recentmethods}
\textbf{Comparison question.} We evaluate two additional proposals: TSA \citep{tsa} and STEM \citep{stem2026}. The comparison tests whether RTT retains useful structural and history-dependent graph adaptation while enforcing the same incumbent-relative protection. All values in this comparison refer to the common evaluation law specified here, rather than to leaderboard results from the cited papers.

\textbf{Complementary mechanisms.} TSA addresses structural distribution shift through uncertainty-aware neighborhood weighting, signal-to-noise-based self--neighbor balancing, and boundary refinement. It tests whether useful structure-aware changes survive the floor constraint. STEM generates domain-specific adapter parameters for a frozen base GNN using a recurrent state-space controller and regularizes parameter evolution. It tests whether useful history-dependent updates can be admitted without treating the controller's stability objective as an incumbent-relative, per-call certificate. These are distinct mechanisms for producing proposals; the same RTT certificate evaluates both.

\textbf{Native fidelity and matched admission.} Each method's documented native configuration uses a development-selected variant and retains its full output path. The matched comparison uses the same incumbent, label set, calls, and proposal features for the full native output, its floor projection, the exported-state proposal without RTT, matched TR, the strongest matched one-pass corrector, and RTT using that proposal. Head capacity, fitting labels, optimization effort, row-budget allocation, and access to proposal/history information are matched between the learned admission alternatives. The native adaptation procedure receives no deployment labels. The controlled extension uses the registered arxiv population and budgets; a compatible GCN/GraphSAGE replication checks that a deep-diffusion port has not created an artificial advantage. The replication specification records ports, omitted native components, and export discrepancies. Native task accuracy/AUC/F1 and negative flips are distinct evaluation quantities and are not inferred from cross-entropy gain.

\begin{table}[!htbp]\centering
\caption{TSA/STEM mean gains in $10^{-3}$ nats per scored call under the common evaluation law specified here. These are results of the present comparison, not results quoted from the cited papers. ``Native'' retains the complete method; ``floor'' projects its output onto RTT's budgets. TR, corrector, and RTT receive the same exported proposal. $\Delta$ is RTT minus the matched corrector. Native gain can exceed certified gain: protection is not claimed to be free.}\label{tab:recentgain}
\small\setlength{\tabcolsep}{5pt}
\begin{tabular*}{\linewidth}{@{\extracolsep{\fill}}lrrrrrr@{}}\toprule
Proposal & \shortstack{Native\\full} & \shortstack{Native\\+ floor} & \shortstack{Matched\\TR} & \shortstack{Matched\\corrector} & \shortstack{RTT\\derived} & $\Delta$\\\midrule
TSA & \pTsaNative & \pTsaFloor & \pTsaTR & \pTsaCorr & \pTsaRTT & \pTsaDelta\\
STEM & \pStemNative & \pStemFloor & \pStemTR & \pStemCorr & \pStemRTT & \pStemDelta\\\bottomrule
\end{tabular*}
\end{table}

\textbf{Independent histories versus continual streams.} To preserve a call-level audit for the stateful proposal, each independent episode starts from the same frozen source initialization, processes a predeclared unlabeled context prefix, and scores one subsequent call. Prefixes are independently drawn under a fixed generator and supplied to every comparator; prefix length and all reset rules are frozen before audit. The same setting is used for TSA. Its numerical outcomes are therefore not pooled with the earlier no-prefix audit. A separate native continual evaluation carries STEM's state across domains and measures retention and adaptation on fixed domain orders. Its correlated calls are not treated as independent binomial trials: resampling and comparisons use independent streams as clusters, with no borrowed call-level Clopper--Pearson claim. RTT's deterministic certificate remains call-wise valid in either mode when its checker is sound.

\textbf{Calibration and inference.} Adapters and heads use the existing fitting split, variants are selected on development data, and each frozen extended policy uses selection-valid calibration. The two paired gain comparisons against the matched corrector form a separate two-test Holm family, specified before the extension audit. The main audit's eight-secondary-comparison claim is unchanged and does not include these additions. Reused cohorts constitute reanalysis; a new population-generalization claim requires an untouched cohort. Fresh episode draws can audit a fixed conditional generator, but do not erase earlier access to its underlying graph or labels. For this comparison, Table~\ref{tab:recentaudit} reports event counts and the call-level bound $u$; we do not report confidence intervals, adjusted $p$-values, betting bounds, or node-level bounds.

\begin{table}[!htbp]\centering
\caption{Event counts and pre-specified criteria for the two RTT-derived policies. Each audit contains $N=2{,}048$ independent scored episodes; $u$ is the one-sided 95\% Clopper--Pearson bound from the regression count. The lower panel lists pre-specified thresholds; the intervals, bounds, and cost measurements needed to evaluate them are not reported for this comparison.}\label{tab:recentaudit}
\small\setlength{\tabcolsep}{5pt}
\begin{tabularx}{\linewidth}{Ycc}\toprule
Quantity & TSA-derived + RTT & STEM-derived + RTT\\\midrule
Observed regressions $N_\zeta$ & \pTsaEvents & \pStemEvents\\
Pointwise call bound $u$ & \pTsaUpper & \pStemUpper\\
First-pass / re-certified releases & \pTsaRone\ / \pTsaRtwo & \pStemRone\ / \pStemRtwo\\
Whole-call fallbacks / fraction & \pTsaFallback\ / \pTsaFallbackRate & \pStemFallback\ / \pStemFallbackRate\\\midrule
Per-node negative-flip criterion & \val{$\le0.60\%$} & \val{$\le0.60\%$}\\
Adjusted paired gain criterion & \val{lower CI endpoint $>0$} & \val{lower CI endpoint $>0$}\\
Additional RTT machinery cost criterion & \val{$\le0.35\,t_{\rm inc}$} & \val{$\le0.35\,t_{\rm inc}$}\\\bottomrule
\end{tabularx}
\end{table}

\textbf{Gain retention and interpretation.} Table~\ref{tab:recentgain} reports native-gain retention of \val{90.0\%} for TSA and \val{88.1\%} for STEM, while the unprotected native methods remain ahead in mean gain. These ratios compare identical populations with positive native gain; they are descriptive and do not establish statistical significance. The paired mean differences against the matched corrector are \pTsaDelta\ and \pStemDelta\ (in $10^{-3}$ nats) for TSA and STEM. Because we do not report their confidence intervals or adjusted $p$-values, we make no significance claim. Component- and seed-level comparisons, including cases where the corrector wins, and the near-cancellation and long-range diagnostics remain separate analyses: neither mechanism is presumed to explain a new gain. A native method's low regression rate is a favorable baseline, not a condition that it must fail. The comparison concerns the gain--protection--cost trade-off, not improvement at every operating point.

The cost criterion in Table~\ref{tab:recentaudit} concerns RTT's added machinery only. Total latency includes native adaptation, context processing, state export, short reference-tail evaluations, solver, checker, and expected fallback; cold-start and amortized continual costs are separate quantities. We do not report timing, FLOP, or memory measurements for this comparison. Appendix~\ref{app:recentprotocol} records the native implementations and replication protocol. These history-aware comparisons remain separate from the abstract's eight-proposal main audit and the corresponding frontier plots.
\FloatBarrier

\Needspace{23\baselineskip}
\section{Conclusion}\label{sec:conclusion}
Reference-Tail Trust turns an internal learned update into a budgeted deployment decision. By pricing each displacement through the frozen incumbent's continuation and validating the executed trajectory before release, RTT attaches a per-row probability floor and a call-level loss-increase bound to the prediction actually served. Whole-call fallback preserves this contract when certification fails. The analytical floor projection and water-filling frontier provide a common basis for evaluating internal admission and output correction, while the separation results identify the information restrictions under which an incumbent tail can make internal intervention useful.

The reported experiments show gains over a matched one-pass corrector across the studied proposal and incumbent families, with benefits concentrated in near-cancellation and long-range-evidence regimes. The TSA and STEM comparisons extend the evaluation to structural and continual adaptation under a separate history-aware protocol; their descriptive gains and event ledgers are reported separately from the main audit. The existing experiments also expose the costs: earlier interventions can be limited by conservative continuation bounds, full-network proposals can make RTT slower than two rollouts, and a short incumbent tail reduces its propagation advantage. Deterministic protection survives arbitrary proposal quality when the checker is sound; statistical regression and calibration statements remain tied to their sampling and selection assumptions. These distinctions position RTT as a controlled mechanism for using learned updates, not as an unconditional replacement for output correction or a guarantee of improvement under unaudited distribution shift.

\section*{Reproducibility and Code Availability}
\CodeAvailability\ The implementation description, release logic, primitive contracts, and training targets are given in Appendices~\ref{app:model}--\ref{app:training}; proofs are in Appendix~\ref{app:proofs}; the audit protocol and numerical diagnostics are in Appendices~\ref{app:stats}--\ref{app:cost}. Table~\ref{tab:factors} records parameters and measured checkpoint properties, Table~\ref{tab:costs} reports label and compute accounting, and Table~\ref{tab:notation} defines the notation. The experimental protocol fixes the pool manifest, selection order, node reservation, and seeds before the final audit; the associated per-episode records and integer event ledger are the basis for reproducing the statistical results.

\par Reproduction of the TSA/STEM extension requires pinned upstream revisions, the declared state-export adapters, independent calibration and audit manifests, and event, timing, and prediction logs. Tables~\ref{tab:recentgain}--\ref{tab:recentaudit} give the gains and event counts for this comparison; its intervals, bounds, and resource measurements are not reported.

\section*{Use of AI Tools}
We used a large language model to improve the English of this manuscript, including grammar, wording, and readability. The authors reviewed all AI-assisted text and take full responsibility for the content of this paper.

\FloatBarrier

\clearpage
\appendix
\section{Model, families, budgets and the release procedure}\label{app:model}
\textbf{Reference families.} Every incumbent is written as primitives whose tube factors and envelopes are known (Table~\ref{tab:families}). The deep tanh diffusion of Section~\ref{sec:rtt} uses interval factors and second-order window envelopes. APPNP propagates linearly after an MLP, so its tail is linear, the first-order contract is exact up to the triangle inequality across rows, the logits are affine in $t$ at every depth and the exact envelope~\eqref{eq:cplus} applies everywhere. GCNII applies $\mathrm{ReLU}(\tilde X\tilde W)$ with $\tilde W=(1-\beta)I+\beta W$ to $\tilde X=(1-\alpha)PX+\alpha H_0$; it enters through Lemma~\ref{lem:relu}, which encloses the Clarke Jacobian of \emph{every} downstream row on the tube, so kinks crossed anywhere in the tail are covered, and its global factor is $(1-\alpha)\|\tilde W\|_2$, not one. GRAND is a smooth attention-free diffusion discretized on the committed horizon. Shallow GraphSAGE and GCN incumbents have $T\le3$: their whole tail is the window and the last depth, so every charge is an envelope. For ResNet-50 the last four bottleneck blocks are the intervened depths, with $P=I$ denoting the absence of graph transport and ReLU handled by Lemma~\ref{lem:relu}; the convolution, residual, and pooling operations still belong to the continuation being enclosed; attention references are covered at first order on bounded tubes \citep{kim2021,dasoulas2021}.

\textbf{Proposals, adapters and heads.} Our proposal returns two candidates per node at every open depth: $H^{{\rm c},1}_\ell$ from a bank prior, a learned convex combination of 16 bank rows conditioned on the current state, and $H^{{\rm c},2}_\ell$ from a signed one-hop attention branch. A published module runs alongside the incumbent on the same call, and its depth-$\ell$ state $\tilde H_\ell$ is mapped into the incumbent's state space by a frozen linear adapter $\Omega_\ell$ fitted by least squares on $\le$2017 episodes, $H^{{\rm c},1}_\ell=\tilde H_\ell\Omega_\ell$, with its one-hop message as the second candidate. GTrans transforms the call's features and edges at test time and its candidates are the incumbent's states on the transformed call; the retrained incumbent and the LoRA-16 delta share the architecture, so their states map exactly. The heads are the node-weight head ($\vartheta_i\in[0,1]$, which scales a row's candidates), the field head $\hat g_\ell$, the covariance head $\hat\Sigma$ and the posterior head $\hat p^{\rm B}$, which estimates $\pB$ for the field's inputs; all read the candidates only through the support-function features below. \emph{Matched TR} (a learned trust rule) replaces the field, the robust model and the solver by a learned trust scale in $[0,1]$, computed from the same features, that scales the proposal's step inside the same candidate hull, budgets and certificate.

\begin{table}[!htbp]\centering
\caption{Primitive decomposition of every implemented incumbent family: the transport, where it is certified, how the window and the last depth are charged, and the factor used when a tube saturates.}\label{tab:families}
\TabFamilies
\end{table}

\subsection{Replication specification for the TSA/STEM comparison}\label{app:recentprotocol}
This subsection specifies the TSA/STEM replication protocol. The native TSA implementation is available at \url{https://github.com/Graph-COM/TSA}; it exposes variants with different boundary-refinement components. The native baseline retains the complete prediction path of the development-selected variant, with the variant and upstream revision recorded. The native STEM implementation is available at \url{https://github.com/miaomiao1220/stem}; its node-level pipeline carries a source-initialized controller state across target domains. Controller initialization, graph order, parameter-update/reset rules, and the source-training budget are part of the replication record. The conference references, rather than the repositories or preprints, are the bibliographic citations \citep{tsa,stem2026}.

For each method, the replication record contains the full native prediction, the exported-state prediction before and after trimming, and every accepted RTT displacement. A pre-audit manifest records node correspondence, depth mapping, adapter dimensions and fitting split, input normalization, self-loops, class order, temperature, and all classifier/output changes. Source graphs used by native initialization must stay within the fitting/development allocation; target labels are inaccessible to adaptation. Unit tests compare the reproduced native output with the upstream implementation and verify that disabling all admissions returns the exact incumbent. A port to the 32-step diffusion is labeled a port and accompanied by a compatible-backbone check; a failed fidelity check is reported rather than hidden by retuning on the audit.

The fixed-history audit freezes the complete proposal state before each scored call. Admission does not feed back into the external proposal controller in the matched comparison: all admission rules see the same candidate record. The native continual experiment is separate and preserves the method's documented evolution. Admission-dependent feedback defines another policy requiring its own calibration and audit, not an automatic extension of these statistics. For both modes, proposal runtime is accounted for separately from RTT overhead, and the accounting includes all label uses and all selected/failed seeds. The values in Table~\ref{tab:recentgain} are reporting outcomes, not stopping rules or tuning objectives.

\textbf{Interval factors.} At depth $j$ the pre-activation of row $i$ is $x_{j,i}=(PH_j)_i$ with units $z_{j,iu}=x_{j,i}W_{:,u}$, and a state in the tube moves it by at most $\varrho_{j,i}=(Pr_j)_i$.
\begin{lemma}[Interval factors]\label{lem:factors}
Let $\bar s_{j,iu}=\tanh'(\max\{0,|z_{j,iu}|-\varrho_{j,i}\|W_{:,u}\|_2\})$.
Every state in the tube satisfies $\|D\varphi\|_2\le\|W\diag(\bar s_{j,i})\|_2\le\|W\|_2$.
Write $W^\top W=U\Sigma U^\top$, with top-$k$ factors $U_k,\Sigma_k$ and $(k+1)$-st eigenvalue $\sigma_{k+1}^2$. Then
\[
 \|W\diag(\bar s)\|_2^2
 \le\lambda_{\max}\!\left(\Sigma_k^{1/2}U_k^\top\diag(\bar s^2)U_k\Sigma_k^{1/2}\right)
      +\sigma_{k+1}^2\max_u\bar s_u^2.
\]
Thus $L_{j,i}$ can be the minimum of $\|W\|_2$ and the square root of this bound. A valid curvature factor is
\[
 \mu_{j,i}=\|W\|_2^2\max_u\sup\left\{|\tanh''(y)|:
       \big||y|-|z_{j,iu}|\big|\le\varrho_{j,i}\|W_{:,u}\|_2\right\}.
\]
\end{lemma}
\begin{lemma}[Piecewise-linear transports]\label{lem:relu}
Let $\varphi(x)=\mathrm{ReLU}(x\tilde W)$ and let the tube move the pre-activation of unit $u$ of row $i$ by at most $\varrho_{j,i}\|\tilde W_{:,u}\|_2$. Let $\bar s_{j,iu}=0$ if $z_{j,iu}+\varrho_{j,i}\|\tilde W_{:,u}\|_2<0$ and $\bar s_{j,iu}=1$ otherwise. Then every Clarke generalized Jacobian of $\varphi$ at every state of the tube satisfies $\|\partial\varphi\|_2\le\|\tilde W\diag(\bar s_{j,i})\|_2\le\|\tilde W\|_2$, the first-order contract and the tube hold with $L_{j,i}=\|\tilde W\diag(\bar s_{j,i})\|_2$ for every row, downstream rows included, and the terminal-logit change lies in an interval enclosure of $\overline{\mathrm{conv}}\{J v:J\in\mathcal J_\ell\}$, where $\mathcal J_\ell$ covers generalized Jacobians of the entire downstream tail on the validated segments. This first-order enclosure replaces a smooth second-order window bound when activation boundaries can be crossed.
\end{lemma}
A primitive row is affine where all its units have fixed activation status. An exact affine-tail envelope is justified only when every contributing downstream activation also has fixed status; otherwise the complete tail is charged through its generalized-Jacobian enclosure. On GCNII \vPatternStable\ of rows are affine on their tube per call.

\textbf{Per-row budgets during execution.} The provisional tube $\tilde r$ is~\eqref{eq:contracts} with cap-seeded provisional factors, updated after each step at the cost of one sparse product. The growth proxy $\hat\Pi_\ell$ is the 95th percentile, over rows and development calls, of the ratio of the terminal to the depth-$\ell$ provisional radius, registered at freeze time. At depth $\ell$ the row allowance is $\rho_{\ell,i}=(\Hrow/\Gamma_\Delta-\hat\Pi_\ell\tilde r_{\ell,i})_+/(\hat\Pi_{\ell+1}n_{\rm left})$ and every candidate row of block $k$ is scaled by $\min\{1,\rho_{\ell,i}/(\bar c_{\ell k}\|d^{(k,m)}_{\ell,i}\|)\}$. Because the scaling acts on the candidates before the solve, the feasible set is again the scaled hull of the scaled candidates. The proxy affects only the gain and the fallback rate; the certificate checks the exact tube.

\textbf{Coordinates.} With the thin factorization $D_\ell=Q_\ell\Theta_\ell$, block by block ($Q_\ell$ orthonormal, $\Theta_\ell$ a rank-revealing coordinate map per block), $v=Q_\ell x$ and $x=\Theta_\ell t$. Slopes pull back as $\hat a_x=Q_\ell^\top\hat g_\ell$ and $\hat a_t=\Theta_\ell^\top\hat a_x=D_\ell^\top\hat g_\ell$; covariances as $\hat\Sigma_x={\rm diag}_k(Q_{\ell k}^\top\hat\Sigma_gQ_{\ell k})$, block-diagonal by construction (cross-block covariances are not modeled; joint errors are calibrated with this same positive-definite physical-coordinate covariance), and $\hat\Sigma_t=\Theta_\ell^\top\hat\Sigma_x\Theta_\ell$; the curvature is $\|v\|^2_\Xi=\sum_i\Xi_{\ell+1,i}\|v_i\|^2$ and $\|v\|^2$ is physical. The feasible set $\{x=\Theta_\ell t:t\ge0,\sum_mt_{km}\le\bar c_{\ell k}\}$ is a polytope in $x$. The contract-depth screen acts on $\hat a_t$: a block whose candidate slopes are all non-negative is set to zero (Lemma~\ref{lem:step}); componentwise non-negativity of $\hat a_x$ would not suffice. The physical problem is written in $x$. An implementation may retain auxiliary $t$ variables to impose the hull constraints, but its objective depends on them only through $v=D_\ell t$. Minimum-norm coefficients are recovered after the physical optimum has been selected.

\textbf{Support-function features.} Every candidate-dependent input of the heads is pooled as $s_{ij}=h_{\mathcal K_{\ell,i}}(u_j)=\max\{0,\max_m\langle u_j,d^{(\cdot,m)}_{\ell,i}\rangle\}$ over $2\times\aQprobes$ fixed probes $\pm u_j$ drawn at freeze time. This is the support function of the row hull. Mean pooling is not hull-invariant ($\mathrm{mean}\{0,1\}=\frac12$, $\mathrm{mean}\{0,1,1\}=\frac23$), and neither is the maximum of a non-convex feature; the maximum of a linear function over a polytope is attained at a vertex, so duplicated, permuted and interior candidates cannot change $s$.

\begin{lemma}[Invariance]\label{lem:invariance}
If the heads read the candidates only through support-function features, and the robust model, the charges and the budgets depend on the step only through $v$ and on the candidates only through the post-trim hull, then the executed displacement is unchanged by candidate operations that preserve the post-trim hull, including duplication, permutation, and insertion of interior points. Rescaling an existing vertex is covered only if the hull is unchanged (for example, when a shorter duplicate is added while the original remains). If the two convex window branches have equal minima, a fixed tie-break in physical state coordinates is required. The trim is nonlinear, so a candidate interior to the raw hull need not be interior to the post-trim hull; the statement concerns the latter.
\end{lemma}
\begin{lemma}[Two-branch window solve]\label{lem:twobranch}
$\{t:\min\{\sigma(t),{\rm env}^+(t)\}\le\Brem\}=\{\sigma\le\Brem\}\cup\{{\rm env}^+\le\Brem\}$ is a union of two convex sets. The minimum of $\Phi$ over the union is the smaller of the two convex minima, and the recorded charge $\min\{\sigma,{\rm env}^+\}$ of the chosen step is at most $\Brem$. Optimizing under ${\rm env}^+$ alone is valid but may be suboptimal; charging $\min$ while optimizing under ${\rm env}^+$ is valid.
\end{lemma}
\begin{lemma}[Regularized block step]\label{lem:step}
For a fixed convex charge branch and multiplier $\mu\ge0$, $\min\Phi_\ell+\mu\,{\rm charge}_\ell$ over its feasible polytope has a unique optimal displacement, continuous in $\mu$, and the spend is non-increasing in $\mu$; for a positive budget, either the constraint is slack at $\mu=0$ or a multiplier can be bracketed and found by bisection under the stated feasibility assumptions. At zero remaining budget, Algorithm~\ref{alg:rtt} takes the zero step rather than requiring a finite multiplier. At contract depths a block with $\hat a_{t,k}\ge0$ componentwise has $t_k=0$ at the optimum.
\end{lemma}

\textbf{The release procedure.} Algorithm~\ref{alg:rtt} is the policy that is audited; the fallback is part of it.
\begin{table}[!htbp]\centering
\caption{Factors, parameters and variables behind the reported values: symbol, meaning, registered value or observed range, where it enters, and what moving it does. The lower block reports properties of the trained checkpoint and of the record.}\label{tab:factors}
\TabFactors
\end{table}

\section{Training, calibration and their cost}\label{app:training}
\textbf{Two phases.} Phase one fits the incumbent's temperature on held-out 2018 labels and records its usable depths; the network itself is not touched. Phase two trains, on $\le$2017 episodes and through the executed gate, the proposal and the heads of Appendix~\ref{app:model}, and fits the adapters of external proposals. The one-pass correctors are trained on the same episodes with the same labels, head capacity and number of gradient steps.

\textbf{Full-loss targets need one counterfactual tail per open depth.} The slope that Theorem~\ref{thm:served}(ii) needs at depth $\ell$ is $a_\ell=D_\ell^\top\nabla V_{\ell+1}(\bar H_{\ell+1})$ with $\bar H_{\ell+1}=F^{\rm r}_\ell(H_\ell)$, the gradient of the incumbent's continuation from the \emph{executed} state. Continuations from different depths start from different states, so a training episode needs one forward and one reverse pass through the tail from each open depth: $\sum_{\ell\ \rm open}(T-\ell-1)$ tail steps in each direction, about \aTrainTailEq\ incumbent passes for the registered caps. The passes are batched across depths, with activations checkpointed every four steps (Table~\ref{tab:training}). A single reverse pass along the executed trajectory would give the gradient of the \emph{adapted} policy's loss, which is a different quantity. Probe-loss slopes are cheaper but biased, and serve only as an ablation (Table~\ref{tab:swaps}).

\textbf{Sequential calibration.} Radii are calibrated depth by depth. For depth $\ell$, $n_r=512$ fresh 2018 episodes are executed with the policy and the radii of depths $<\ell$ frozen. The score is the joint Mahalanobis norm $s_\ell=\|\hat\Sigma_{x,\ell}^{-1/2}(a_{x,\ell}-\hat a_{x,\ell})\|$ of the slope error in physical coordinates at the executed state, with the block-diagonal $\hat\Sigma_{x,\ell}$ of Appendix~\ref{app:model}, and $\kappa_\ell$ is its $\lceil(1-\delta_\ell)(n_r+1)\rceil$-th order statistic, with $+\infty$ used when the index exceeds $n_r$. The covariance is assumed positive definite on the physical span; its numerical regularization is part of the frozen calibration rule. On the event $s_\ell\le\kappa_\ell$, Cauchy--Schwarz gives $a_{x,\ell}^\top x\le\hat a_{x,\ell}^\top x+\kappa_\ell\|\hat\Sigma_{x,\ell}^{1/2}x\|$ for every $x$, which is the robust term of~\eqref{eq:robust}. The allocation is $\delta_\ell=\delta_{\rm cov}\,\pi_\ell$, with $\pi_\ell$ the registered trust share of depth $\ell$.
\begin{lemma}[Sequential calibration]\label{lem:seqcal}
For a member fixed independently of its calibration data, the depth-$\ell$ score of a fresh episode does not depend on $\kappa_\ell$ or on any later radius, and it is exchangeable with the $n_r$ calibration scores of depth $\ell$, which were produced by the same frozen policy up to depth $\ell$. Hence $\Pr(s_\ell>\kappa_\ell)\le\delta_\ell$, and the event that every open depth is covered has probability at least $1-\sum_\ell\delta_\ell=1-\delta_{\rm cov}$ over calibration and test episode.
\end{lemma}
Feeding the radius back into the trajectory that generates its own calibration scores would break exchangeability; the sequential construction freezes everything upstream of the score. The probability is marginal over calibration and test episodes. Selecting a member using calibrated performance is not covered by this fixed-member statement: the selected procedure needs fresh post-selection calibration or a simultaneous guarantee over its selectable members. The independent final audit in Appendix~\ref{app:stats} remains a distinct source of statistical assurance.

\begin{table}[!htbp]\centering
\caption{Training cost of one episode (operation counts; A100 80GB).}\label{tab:training}
\TabTraining
\end{table}

\section{Proofs}\label{app:proofs}
\textbf{Proof of Proposition~\ref{prop:price}.} (i)~For cross-entropy and any $\mathcal O$-measurable logits, $\E[\ell(z_i,y_i)\mid\mathcal O]=H(\pB_i)+\mathrm{KL}(\pB_i\|\softmax(z_i))$ \citep{banerjee2005}. Taking the difference of the two predictions gives the identity. If $\pr=\pB$, the first divergence vanishes and the second is non-negative.
(ii)~Fix a row, write $q=\pB_i$, $p=\pr_i$ and $h_c=h_{ic}$. Since $\sum_cp_ch_c=\mathrm{KL}(p\|\ps_i)\ge0$, the gain of the row is $-\sum_cq_ch_c\le\sum_c(p_c-q_c)h_c\le\mathrm{TV}(q,p)\osc_ch_c$. Summing with the weights, and using $\sum_iw_i\osc_ch_{ic}=S^++S^-\le2\max(S^+,S^-)$, gives the bound. For sharpness take $p=(\frac12,\frac12)$, $q=(\frac12+\epsilon,\frac12-\epsilon)$ and $\Delta z=(t,-t)$: then $\mathcal G/(\epsilon\max(S^+,S^-))\to2$ as $t\downarrow0$.
(iii)~Write $r_c=q_c/p_c$ and $R=\max_cr_c$. Then
$$\textstyle\sum_cq_ch_c=R\sum_cp_ch_c-\sum_cp_c(R-r_c)h_c\ge-(R-1)\max_ch_c,$$
so the gain of the row is at most $(R-1)\max_ch_c$. For attainment let $c^\star\in\arg\max_cr_c$ and $p^\varepsilon=(1-\varepsilon)p+\varepsilon e_{c^\star}$. Then $h_c=-\log(1-\varepsilon)=\varepsilon+O(\varepsilon^2)$ for $c\ne c^\star$ and $h_{c^\star}=-\varepsilon(1/p_{c^\star}-1)+O(\varepsilon^2)$, so $S^+=\varepsilon+O(\varepsilon^2)$ and
$$-\textstyle\sum_cq_ch_c=-(1-q_{c^\star})\varepsilon+q_{c^\star}\varepsilon(1/p_{c^\star}-1)+O(\varepsilon^2)=\varepsilon(R-1)+O(\varepsilon^2).$$
(iv)~For $h=0$, the floor set is the singleton $\{\pr\}$. For $h>0$, $p\in\Fl_h(\pr)$ if and only if $s=(p-e^{-h}\pr)/(1-e^{-h})$ is non-negative and sums to one, that is $s\in\Delta$. For the ceiling, $p_c=1-\sum_{c'\ne c}p_{c'}\le1-e^{-h}(1-\pr_c)$. For the forward projection, and for the reverse projection when every $q_c>0$, minimize respectively $-\sum_cq_c\log p_c$ or $\sum_cp_c\log(p_c/q_c)$ subject to $p_c\ge a_c=e^{-h}\pr_c$ and $\sum_cp_c=1$. Stationarity gives $p_c=q_c/(\nu-\beta_c)$ in the first problem and $p_c=q_ce^{\beta_c-1-\nu}$ in the second, with $\beta_c\ge0$ and $\beta_c(p_c-a_c)=0$. In both, $p_c=\max\{a_c,q_c/\lambda\}$ with $\lambda$ fixed by normalization, which has a unique positive solution for $h>0$. If some $q_c=0$, the forward projection still follows this formula, but the positive floor makes the reverse KL infinite. At $h=0$ the output remains unique even though a multiplier representation need not be. \qed

\begin{proposition}[Row contracts on a tube]\label{prop:contracts}
Under the conditions of Section~\ref{sec:rtt}, $|DV_j(X)[Y]|\le\sum_i\Lambda_{j,i}\|Y_i\|$ and $|D^2V_j(X)[Y,Y]|\le\sum_i\Xi_{j,i}\|Y_i\|^2$ for every state $X$ of the tube and every label vector, with $\Lambda$ as in~\eqref{eq:contracts} and $\Xi_j=(1-\alpha)^2[(1-\tau)\Xi_{j+1}+\tau P^\top(L_j^2\odot\Xi_{j+1})]+(1-\alpha)\tau P^\top(\mu_j\odot\Lambda_{j+1})$, $\Xi_{T,i}=\frac14\Gamma_\Delta^2w_i$.
\end{proposition}
\emph{Proof.} $V_j=V_{j+1}\circ F^{\rm r}_j$ with $DF[X]_i=(1-\alpha)[(1-\tau)X_i+\tau D\varphi((PH)_i)(PX)_i]$. On the tube, $\|DF[X]_i\|\le(1-\alpha)[(1-\tau)\|X_i\|+\tau L_{j,i}\sum_kP_{ik}\|X_k\|]$. Substituting and collecting coefficients gives the first line. The second follows from $D^2V_j[X,X]=D^2V_{j+1}[DF[X],DF[X]]+DV_{j+1}[D^2F[X,X]]$ and Jensen. The terminal rows follow from $|\langle p-e_y,\delta z\rangle|\le\Gamma_\Delta\|X_i\|$ and Popoviciu's inequality, uniformly in labels. \qed

\begin{lemma}[Forward enclosure]\label{lem:tube}
With $r_j=0$ before the first accepted depth and the recursion of~\eqref{eq:contracts}, the tube contains every incumbent continuation started on the segment of any executed step, including the incumbent's own trajectory.
\end{lemma}
\emph{Proof.} Let $X=\bar H_{\ell+1}+sv_\ell$ with $\bar H_{\ell+1}=F^{\rm r}_\ell(H_\ell)$ and $s\in[0,1]$, and let $X_j$ be its incumbent continuation. Then $\|X_{\ell+1,i}-H_{\ell+1,i}\|=(1-s)\|v_{\ell,i}\|$. By the mean-value theorem along the segment from $H_j$ to $X_j$ (Lebourg's for Lipschitz maps), $\|X_{j+1,i}-H_{j+1,i}\|\le(1-\alpha)[(1-\tau)e_{j,i}+\tau L_{j,i}(Pe_j)_i]+\|v_{j,i}\|$ for $e_j=\|X_j-H_j\|$, provided the segment lies in the tube. This holds by induction. The toy instance of Figure~\ref{fig:certificate}(a) checks the statement for every continuation. \qed

\emph{Proofs of Lemmas~\ref{lem:factors}--\ref{lem:step}.} Lemma~\ref{lem:factors}: $\tanh'$ is even and decreasing in $|y|$, and $\|W\diag(s)\|_2$ is non-decreasing in every entry of $s\ge0$. The split is Weyl's inequality applied to $\diag(\bar s)W^\top W\diag(\bar s)=\diag(\bar s)(A_k+(W^\top W-A_k))\diag(\bar s)$ with $A_k=U_k\Sigma_kU_k^\top$ the rank-$k$ part. Lemma~\ref{lem:relu}: the Clarke Jacobian of $\mathrm{ReLU}(x\tilde W)$ is $\tilde W\diag(s)$ with $s_u\in[0,1]$, and $s_u=0$ wherever the unit is inactive on the whole tube; Lebourg's mean-value theorem places every increment in the convex hull of these Jacobians along the segment. Lemma~\ref{lem:twobranch} is immediate. Lemma~\ref{lem:step}: strict convexity in $v$ (from $\rho_x>0$), the maximum theorem, and the exchange argument $(\mu_2-\mu_1)(\sigma(t_2)-\sigma(t_1))\le0$. For the screen, setting $t_k=0$ removes the non-negative term $\hat a_{t,k}^\top t_k$ and does not increase the robust term (the covariance is block-diagonal), the quadratic terms, or the charge. The quadratic terms and contract charge add across blocks; the joint uncertainty norm is nonincreasing when any block is removed. Lemma~\ref{lem:invariance}: the listed operations leave the post-trim hull, and hence every support-function feature and the feasible polytope in $v$, unchanged. Each convex branch has a unique minimizing displacement. A representation-independent physical tie-break selects between equal branch optima. Coefficients enter only through minimum-norm recovery, so the selected displacement is invariant even when its candidate coefficients are not. \qed

\textbf{Proof of Theorem~\ref{thm:served}.} Let $f_\ell(t)=V_{\ell+1}(\bar H_{\ell+1}+D_\ell t)-V_{\ell+1}(\bar H_{\ell+1})$, $\bar H_{\ell+1}=F^{\rm r}_\ell(H_\ell)$, for a fixed label vector. Telescoping gives $J(p^\pi;y)-J(\pr;y)=\sum_\ell f_\ell(t_\ell)$ \citep{kakade}.
(i)~At contract depths the mean-value theorem on the certified tube gives $f_\ell\le\sigma_\ell$. In the window, $f_\ell\le{\rm env}^+_\ell$ by the tangent basis and the second-order contract, and the minimum of two valid bounds is valid. At the last depth $f_{T-1}\le C^+$ exactly. The damage predicate gives $d^{\rm own}\le\sum_\ell{\rm charge}_\ell\le H^+$ for every label vector. Since $S^+$ is the worst case over label vectors, this is $\sum_iw_i\Dinf(\pr_i\|p^\pi_i)\le H^+$. For the rows, the tube encloses the incumbent's own trajectory, so $\|H_{T,i}-H^{\rm r}_{T,i}\|\le r_{T,i}$. With $\Delta z_i$ the logit difference, $\Dinf(\pr_i\|p^\pi_i)=\LSE(z^\pi_i)-\LSE(z^{\rm r}_i)-\min_c\Delta z_{ic}\le\osc(\Delta z_i)\le\Gamma_\Delta r_{T,i}\le\Hrow$ by the row predicate. On fallback $\ps=\pr$. A sound directed-rounding implementation preserves the enclosures used by each release predicate, including signed intermediate expressions and the numerical allowances (Appendix~\ref{app:cost}).
(ii) Work in physical coordinates $v_\ell=Q_\ell x_\ell$ and use the full-loss slope $a_{x,\ell}$. Write
\[
 f_\ell(v_\ell)=a_{x,\ell}^\top x_\ell+\mathrm{rem}_\ell(v_\ell).
\]
Let $b_\ell(v)$ be the nonnegative remainder majorizer validated by the checker. At a smooth nonterminal depth,
$b_\ell(v)=\|v\|_\Xi^2/2$; at the last depth, $b_{T-1}(v)=\varrho_{T-1}(v)$ is the exact KL remainder. For a nonsmooth tail, this notation refers to the separately validated remainder model, not an unsupported Hessian bound. Define
\[
 S_\ell^{\mathrm{rem}}=b_\ell(v_\ell)-\mathrm{rem}_\ell(v_\ell)\ge0,
 \qquad \chi_\ell(x)=\kappa_\ell\|\hat\Sigma_{x,\ell}^{1/2}x\|,
\]
and $e_\ell=(a_{x,\ell}^\top x_\ell-\hat a_{x,\ell}^\top x_\ell-\chi_\ell(x_\ell))_+$,
$\bar e_\ell=(\hat a_{x,\ell}^\top x_\ell+\chi_\ell(x_\ell)-a_{x,\ell}^\top x_\ell)_+$.
With $\mathcal R_\ell(v)=(\rho_D\|v\|_\Xi^2+\rho_x\|v\|_2^2)/2$, the recomputed model is
\[
 \Phi_\ell^{\rm cert}(x)=\hat a_{x,\ell}^\top x+\chi_\ell(x)
              +\frac{b_\ell(Q_\ell x)+\mathcal R_\ell(Q_\ell x)}{\omega}.
\]
Adding and subtracting this model gives the exact identity
\[
 f_\ell=\Phi_\ell^{\rm cert}(x_\ell)+e_\ell-\bar e_\ell
       -\Bigl(\frac1\omega-1\Bigr)b_\ell(v_\ell)
       -\frac{\mathcal R_\ell(v_\ell)}{\omega}-S_\ell^{\mathrm{rem}}.
\]
Therefore
\[
 \begin{aligned}
 \mathcal D_T&=-\sum_\ell\Phi_\ell^{\rm cert}(x_\ell),
 &\mathcal P_T&=\sum_\ell e_\ell,\\
 \mathcal S_T&=\sum_\ell\left[\bar e_\ell+S_\ell^{\mathrm{rem}}
 +\Bigl(\frac1\omega-1\Bigr)b_\ell(v_\ell)
 +\frac{\mathcal R_\ell(v_\ell)}{\omega}\right].
 \end{aligned}
\]
The certified descent predicate makes $\mathcal D_T\ge0$, and each summand of $\mathcal S_T$ is nonnegative. At the last depth, $S_{T-1}^{\mathrm{rem}}=0$ because the remainder is exact. As an unregularized scalar illustration, take $f(t)=-t/2+t^2/2$, $\omega=0.9$, and $t=0.45$. Then $\Phi=-0.1125$, $f=-0.12375$, and $S=0.01125$. With nonzero regularization the additional $\mathcal R(t)/\omega$ must appear in both the model and the slack; it cannot be omitted from the accounting.
(iii)~On the coverage event $\mathcal E_\ell=\{\|\hat\Sigma_{x,\ell}^{-1/2}(a_{x,\ell}-\hat a_{x,\ell})\|\le\kappa_\ell\}$, Cauchy--Schwarz gives $a_{x,\ell}^\top x\le\hat a_{x,\ell}^\top x+\kappa_\ell\|\hat\Sigma_{x,\ell}^{1/2}x\|$ for every $x$, so $e_\ell=0$. The descent predicate gives $\mathcal D_T\ge0$, hence $d^{\rm own}\le0$; on fallback $d^{\rm own}=0$. Lemma~\ref{lem:seqcal} gives the fixed-member marginal probability. Extending it to the selected member requires the selection-valid calibration condition stated in the theorem. \qed

\textbf{Proof of Theorem~\ref{thm:frontier}.} (i)~The set $\{(h,p):p_c\ge e^{-h}\pr_c\ \forall c,\ p\in\Delta\}$ is convex because $-e^{-h}$ is concave. So $\min_{p\in\Fl_h}\mathrm{KL}(\pB\|p)$ is convex in $h$ as a partial minimization, and $G$ is concave. It is non-decreasing because the sets are nested, and $G(0)=0$. With the active set $A$ of the projection and the normalizer $\nu$, $G'(h)=\nu e^{-h}\sum_{c\in A}\pr_c-\sum_{c\in A}\pB_c$. As $h\downarrow0$, $A$ is every class except the maximizers of $r$, and $\nu\to R$, so, writing $M=\arg\max_c r_c$, $G'(0^+)=R\sum_{c\notin M}\pr_c-\sum_{c\notin M}\pB_c=R-1$. This also covers ties among maximizing classes. The saturation point is where $\pB\in\Fl_h$. Any feasible served prediction has per-row costs $h_i$ that satisfy the constraints and gain at most $\sum_iw_iG_i(h_i)$, so $U_{\rm out}$ equals the allocation problem. That problem is a concave separable maximization under one linear constraint and boxes, and its KKT conditions are the water-filling conditions.
(ii)~Per row, with $u=\delta z_i(v)$ and $p=p_{\ell,i}$, $(\pB_i-p)^\top u=\sum_c(\pB_{ic}-p_c)(u_c-\min_{c'}u_{c'})\le(R_{\ell,i}-1)\sum_cp_c(u_c-\min_{c'}u_{c'})$, because $\pB_{ic}-p_c\le(R_{\ell,i}-1)p_c$ and the second factor is non-negative. Summing with the weights gives $g_\ell\le\max_i(R_{\ell,i}-1)c_\ell$. The Charnes--Cooper transformation turns the ratio into a linear program. The analysis script re-checks the bound on random instances.
(iii)~The chain follows from nested feasible sets: a reachable intervention is an output-space competitor, certified charges dominate exact damage, and every released RTT intervention is certified. Fallback is the zero-gain feasible alternative. These statements concern exact optima, not guarantees of global convergence for the numerical solver. \qed

\textbf{Local learning comparison, equation~\eqref{eq:localgap}.} Fix a depth and a convex feasible branch $\mathcal C$ containing zero. All optimization variables below are physical coordinates $x$, not possibly redundant candidate coefficients. Condition on $\mathcal O_\ell$, which must make $\mathcal C$, the positive-definite majorizer $M=\bar\Xi_\ell$, the learned slope, and the uncertainty penalty measurable. Assume the true loss change satisfies $f(x)\le a^\top x+\tfrac12x^\top Mx$ on $\mathcal C$. The majorizer may include the safety-factor inflation and positive physical regularizer, so that $M^{-1}$ is defined even when the candidate matrix is rank deficient.

Set $m=\E[a\mid\mathcal O_\ell]$, $\phi(x)=m^\top x+\tfrac12x^\top Mx$, and $\mathcal A(\mathcal O_\ell)=-\min_{x\in\mathcal C}\phi(x)$. Let $x_m$ minimize $\phi$ and let $x^*$ minimize the certified robust objective $\Psi(x)=\hat a^\top x+\chi(x)+\tfrac12x^\top Mx$ on that same branch. For the realized step $\hat x\in\mathcal C$, define $\Upsilon_\ell=\Psi(\hat x)-\Psi(x^*)\ge0$, including solver error and any provisional-to-certified model mismatch. The latter quantity is well-defined because the final step is checked for certified feasibility.

Strong convexity on the convex set gives
\[
 \Psi(x^*)\le\Psi(x_m)-\tfrac12\|x_m-x^*\|_M^2,
 \qquad \|\hat x-x^*\|_M\le\sqrt{2\Upsilon_\ell}.
\]
Writing $e=\|m-\hat a\|_{M^{-1}}$ and $s=\|x^*-x_m\|_M$, conditional expectation and Cauchy--Schwarz yield
\[
 \E[f(\hat x)\mid\mathcal O_\ell]
 \le-\mathcal A(\mathcal O_\ell)+\chi(x_m)+\Upsilon_\ell
      +e\sqrt{2\Upsilon_\ell}+es-\tfrac12s^2.
\]
Maximizing the last two terms over $s$ gives
\[
 \E[f(\hat x)\mid\mathcal O_\ell]
 \le-\mathcal A(\mathcal O_\ell)+\chi_\ell
      +\tfrac12\big(e+\sqrt{2\Upsilon_\ell}\big)^2,
 \qquad\chi_\ell=\chi(x_m).
\]
Assume $\mathcal O_\ell\subseteq\mathcal O$. Let $\bar m=\E[a\mid\mathcal O]$ and
$A_\ell^{\rm cert}=-\min_{x\in\mathcal C}(\bar m^\top x+\tfrac12x^\top Mx)$.
The same feasible geometry is used under both information sets. Since the negative minimum is convex in its slope argument, conditional Jensen gives
$I_\ell=\E[A_\ell^{\rm cert}\mid\mathcal O_\ell]-\mathcal A(\mathcal O_\ell)\ge0$.
Negating the previous display proves~\eqref{eq:localgap}. Summing such conditional comparisons requires consistent branch and information accounting. Minkowski bounds the summed estimation/execution penalties by $\tfrac12(\sqrt{2E}+\sqrt{2\Upsilon})^2$, where $E=\sum_\ell e_\ell^2/2$ and $\Upsilon=\sum_\ell\Upsilon_\ell$. No identity between the sum of local quadratic opportunities and the global optimum $U_{\rm cert}$ follows. In particular, strong convexity on either branch does not by itself give the distance inequality across their nonconvex union. \qed

\textbf{Proof of Proposition~\ref{prop:separation}.} (a)~is Proposition~\ref{prop:price}(iv).
(b)~Take a five-node path with self-loops, the row-normalized $P$, a scalar state equal to the class-1 logit margin, incumbent margin $-1$ everywhere, one diffusion step after the intervened depth, and scored node 2 with $\pB=0.5$. The raw update adds $+\delta$ at node 1 and $-\delta$ at node 3 ($\delta=\aSepDelta$). Its terminal effect at node 2 is $(\delta-\delta)/3=0$, so $\pu_2=\pr_2$ and every inert rule leaves node 2 unchanged. Node 3's candidate has a positive slope for the scored loss and is screened out. Admitting node 1's candidate alone moves node 2's margin by $\delta/3$ at one-sided cost \aSepH\ nats and gains \aSepGain\ nats, which is also the oracle projection's gain at that cost (Figure~\ref{fig:sepinst}).
(c)~On a directed chain of $L$ nodes the evidence is in the proposal's candidate row of node 0 at intervention depth $\ell=T-L$ (injected into $H_{T-L+1}$). A corrector with receptive field $k<L-1$ at node $L-1$ sees the same inputs with and without the evidence, so its prediction cannot depend on that evidence. If the incumbent is Bayes-optimal conditional on these local inputs, Proposition~\ref{prop:price}(i) gives zero possible conditional gain for the restricted corrector. RTT's step at node 0 is carried by the incumbent's tail to node $L-1$ and gains \aSepTgain\ nats at cost \aSepTh. \qed
\begin{remark}[Certificates compose]\label{rem:compose}
$\Dinf(p\|p'')\le\Dinf(p\|p')+\Dinf(p'\|p'')$ because $\log(p_c/p''_c)=\log(p_c/p'_c)+\log(p'_c/p''_c)$. A corrector applied to a certified adapted prediction with per-row budgets $\Hrow-\Gamma_\Delta r_{T,i}$ and per-call budget $H^+-\sum_\ell\mathrm{charge}_\ell$ therefore keeps G1--G3 relative to $\pr$. If RTT falls back, the starting prediction is $\pr$ and the available budgets reset to $\Hrow$ and $H^+$. Coverage and statistical performance must be checked separately for the composed policy.
\end{remark}

\begin{theorem}[Assurance per call]\label{thm:assurance}
(i)~For a policy frozen before $N$ i.i.d.\ episodes with $N_\zeta$ regressions, $u=\mathrm{Beta}^{-1}(1-\delta;N_\zeta+1,N-N_\zeta)$ bounds the regression rate with probability at least $1-\delta$; set $u=1$ when $N_\zeta=N$. (ii)~Testing every event of~\eqref{eq:problem}, the auxiliary accuracy and flip events, the stress test and a minimum certified mean at level $\delta$ each, and releasing only if all pass, has false-release probability at most $\delta$. Along a library ordered before the data, stopping at the first failure controls the family-wise error. (iii)~Assume $\E[|d|]<\infty$. Because $d_j\le H^+$, predictable nonnegative $c_j$ that do not depend on the candidate mean $m$ define bets $\lambda_j(m)=\min\{c_j,\frac34/(H^+-m)\}$ give wealth $\mathcal W_t(m)=\prod_{j\le t}(1+\lambda_j(m)(m-d_j))$, non-decreasing in $m$. $\mathrm{UCB}_t=\sup\{m<H^+:\mathcal W_t(m)<1/\delta\}$ then satisfies $\Pr(\exists t:\E[d]>\mathrm{UCB}_t)\le\delta$, with the upper endpoint capped at $H^+$; if $\E[d]=H^+$, then $d=H^+$ almost surely and this endpoint is valid. No lower bound on $d$ is used. (iv)~The $k$ statements hold simultaneously at $\delta/k$ each, and any reweighting $\varpi$ of the components has rate at most $\max\{\varpi^\top p:0\le p\le u^{\rm sim},\varpi_0^\top p\le u^{\rm sim}_{\rm mix}\}$. (v)~\emph{Conditional robustness}: for a prespecified grid $\lambda_1,\dots,\lambda_J>0$ and simultaneous upper confidence bounds $U_j$ on $\E_P[e^{\lambda_j(d-H^+)}]$, every law $Q$ with $\mathrm{KL}(Q\|P)\le\rho$ has $\E_Q[d]\le H^++\min_j(\rho+\log U_j)/\lambda_j$ with probability at least $1-\delta$. This is a statement about the hypothesized ball; a score-based divergence estimate does not provide the upper bound on $\mathrm{KL}(Q\|P)$ needed to place $Q$ in it.
\end{theorem}
\begin{lemma}[Monotone one-sided wealth]\label{lem:mono}
If $d_j\le H^+$ and $\lambda_j(m)=\min\{c_j,\frac34/(H^+-m)\}$ with $c_j\ge0$ not depending on $m$, then each factor is at least $\frac14$ and non-decreasing in $m<H^+$.
\end{lemma}
\emph{Proofs.} (i)~is binomial tail inversion \citep{clopper1934}. (ii)~is the intersection--union principle and fixed-sequence testing \citep{ltt}. (iii)~At $m=\E[d]$ the wealth is a non-negative martingale, so Ville's inequality applies \citep{wsr}. Lemma~\ref{lem:mono} makes the rejected set an interval, since on each branch the derivative of the factor is non-negative. (iv)~is the union bound and a linear program. (v)~is the Donsker--Varadhan inequality with $X_\lambda=e^{\lambda(d-H^+)}\in(0,1]$ bounded \citep{cauchois}. \qed

\begin{theorem}[Per node]\label{thm:node}
(i)~A node can flip from correct to wrong only if its predicted class changes, so the per-call negative-flip rate is at most the disagreement with the incumbent, which the offline audit records. (ii)~Let $\mathcal N$ contain $n_{\mathcal N}$ reserved nodes that are not used as scored or labeled targets by any other stage. They may occur as unlabeled graph context. Fix their predictions by calls of the registered size before drawing the inspection subset, and let a uniformly random subset of $n_{\rm cal}$ nodes be labeled afterwards, with $k$ flips. With $\hat F=\max\{F':\Pr_{\mathrm{Hyp}(n_{\mathcal N},F',n_{\rm cal})}(X\le k)\ge\delta\}$, the flip rate on the uninspected nodes is at most $(\hat F-k)/(n_{\mathcal N}-n_{\rm cal})$ with probability at least $1-\delta$.
\end{theorem}
\emph{Proof.} (i)~is immediate. (ii)~After the calls, the flip vector is fixed and the inspected set is a uniform subset, so $k$ is hypergeometric; $\hat F$ inverts its tail. \qed

\section{Statistical protocol}\label{app:stats}
\textbf{Four uses of labels.} Fitting labels ($\le$2017) train every network. 2018 labels fit the temperature, the radii and the selection. The 2019 development half serves every design decision, the $\zeta$ anchor, the sweeps of Figure~\ref{fig:deploy}(c,d) and Figure~\ref{fig:robust}. The reported protocol assigns the final-half audit pool, the reserved node set, and the 2020 cohort to their respective evaluation stages (Table~\ref{tab:costs}). Any later reuse must be recorded as reanalysis, not as evaluation on an untouched population. Every audited statement concerns one frozen policy on the registered generator, conditional on the graph and the label masks: the i.i.d.\ premise requires fresh, independent query and corruption draws conditional on that population. Stateful proposals or caches additionally require independently initialized histories for this audit; a continuous shared-state stream is a different sampling unit. Intervals are episode-level and paired where methods share episodes. No design effect is mixed with the betting bounds, which rest on the same i.i.d.\ premise. A node-cluster bootstrap for the population-level estimand widens intervals by a factor \vClusterDE\ and is reported separately.

\textbf{The release ledger.} Outcomes are mutually exclusive and counted as integers (Table~\ref{tab:audit}). First-pass failures are attributed to the first failing predicate in the order damage, row, descent. Re-certified releases and fallbacks partition the failures, and the analysis script asserts $0\le N_{\rm F}\le N_{\rm fail}\le N$ before any rate is printed. The release rule also requires a certified mean of at least $3.0\times10^{-3}$ nats ($1.5\times10^{-3}$ at the tolerance point), registered before the audit.

\begin{table}[!htbp]\centering
\caption{Final audit of the registered member: (a)~release ledger (integer counts of 2,048 calls); (b)~the ten statements, pointwise at $\delta$ and simultaneous at $\delta/10$, with the component means of the separate 4,096-episode audits, which are independent samples and differ from the mixture's component means within sampling error.}\label{tab:audit}
\TabAudit
\end{table}

\textbf{Selection.} Library members are ordered before any data from least to most aggressive, all with trajectory-valid radii except a per-depth member placed last. Each member reached is tested on four pools of 1,024 fresh 2018 episodes (mixture, clean, cache, noise). The path stops at the first failure, and among the passing members the one with the largest certified mean is registered: \vSelReg, with certified mean \vSelCert. The path tests five members and stops at the fifth; the per-depth member, placed sixth, gains \vPdGain\ but is not reached, so it is reported as a variant, not selected. The independent final audit validates the frozen served policy. The recorded selection procedure alone does not establish a post-selection $\delta_{\rm cov}$ guarantee; see Section~\ref{sec:assurance}.

\textbf{Betting bounds and their implementations.} The registered bets are tuned to $N$. Their values at every $t$ form an anytime-valid sequence whose last value is the fixed-$N$ bound. The reported horizon-free schedule is wider at this $N$; that empirical comparison is not a universal ordering of confidence-sequence constructions. The Gaussian first-order width $\sqrt{2\sigma^2\log(1/\delta)/N}$ is the benchmark: the registered bets come within \vBetRatio\ of it. A finite-sample bound can lie below this first-order approximation because the exact wealth depends on higher-order features of the observed contrasts; the approximation is not a lower limit on all valid bounds. Every implementation is computed from the same per-episode record (Table~\ref{tab:stats}; Figure~\ref{fig:statistics}).

\textbf{Transfer and robustness.} The member frozen after selection is audited on 2,048 episodes of the reserved 2020 cohort with its own Clopper--Pearson and betting bounds. The reported label-free divergence estimate is $\hat\rho=\vRhoHat$. Data processing makes the exact divergence between the score distributions no larger than the divergence between the full episode laws; the estimate itself also has sampling error. It is a diagnostic, not a certified upper bound on the episode-law divergence, and cannot establish membership in the ball of Theorem~\ref{thm:assurance}(v). The cohort's direct audit gives a certified mean of \vTrRttLB, whereas the conditional curve is \vRobAtRhoHat\ at $\hat\rho$. Comparing these two lower bounds does not quantify the actual shift; the direct audit carries the cohort-specific claim.

\begin{table}[!htbp]\centering
\caption{Statistics: (a)~lower bounds on the mean gain by implementation at $N=2{,}048$ against the Gaussian first-order width; (b)~transfer to the reserved 2020 cohort, audited directly (the pass/fail column tests $u\le p^\star$, not the principal audit's minimum-mean requirement); (c)~the conditional robustness curve for laws within KL radius $\rho$ of the audit law.}\label{tab:stats}
\TabStats
\end{table}

\textbf{Deployment unit, primary comparison and multiplicity.} Theorem~\ref{thm:node} concerns the registered call size. The per-call rates are reported for $n_{\rm q}\in\{64,256,1{,}024,4{,}096\}$, both for the registered member and for members re-selected per size (Figure~\ref{fig:deploy}c). The primary comparison is a one-sided paired test at level $\delta$. The eight secondary comparisons are Holm-corrected at the same level, and all pass (Table~\ref{tab:frontierfull}).

\begin{figure}[!htbp]
\centering\includegraphics[width=\linewidth]{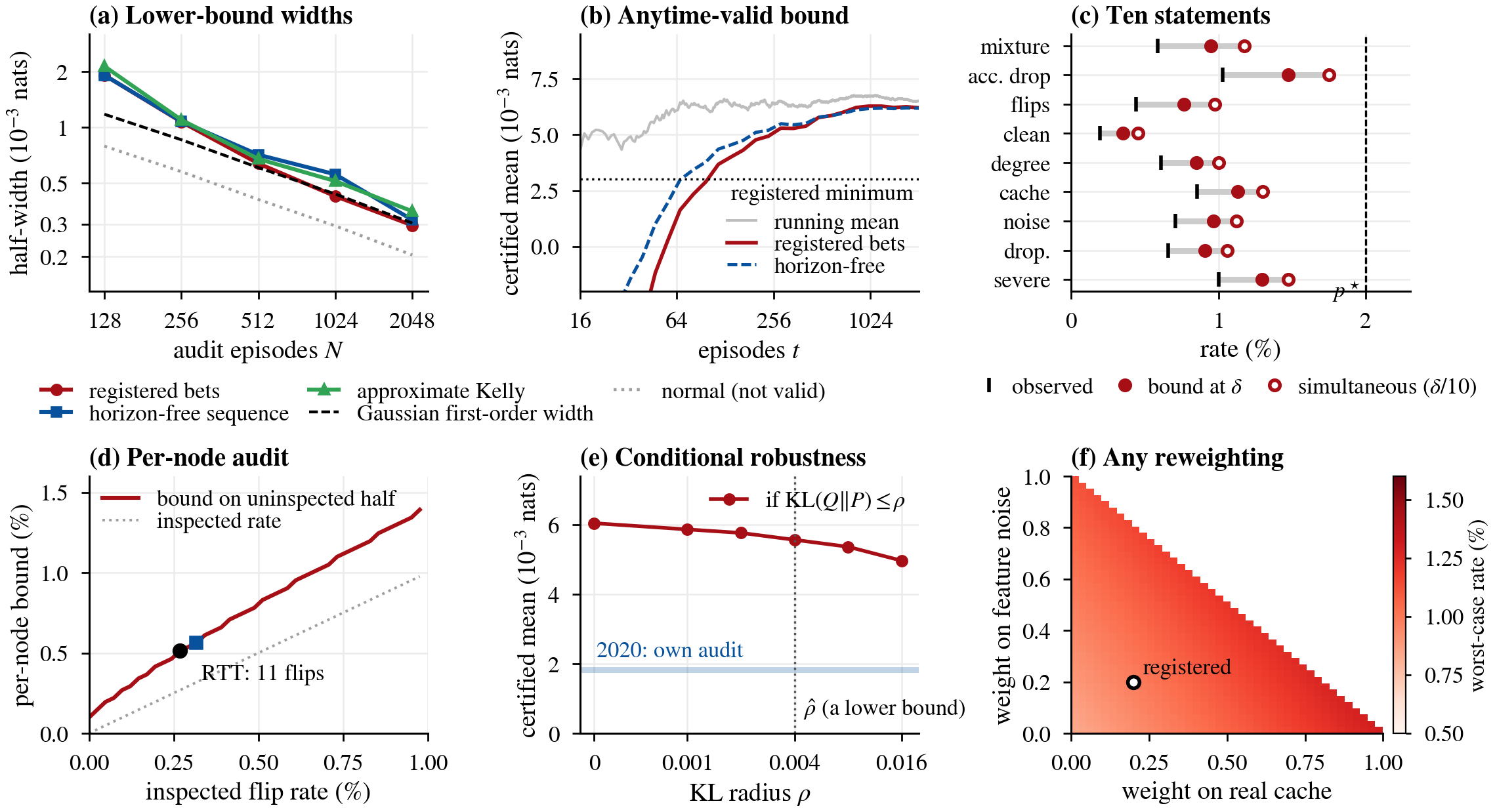}
\caption{Statistical machinery on the audit record. (a)~Half-widths of lower bounds against $N$ with the Gaussian first-order width. (b)~Running bounds: the registered bets at every $t$ (anytime-valid) and the horizon-free sequence. (c)~The ten statements with the observed rate, the bound at $\delta$ and the simultaneous bound. (d)~Per-node bound against the inspected flip rate; circle: RTT, square: the one-pass corrector. (e)~The conditional robustness curve, and the 2020 cohort's own audit (band). (f)~Worst-case regression rate over reweightings of the components.}\label{fig:statistics}
\end{figure}

\section{Where the gain goes: definitions}\label{app:gap}
\textbf{Planted partition and numerical scope.} Labels follow a known law on a graph with the arxiv architecture. At $\gamma=0$ the incumbent is calibrated by construction. A planted community containing 30\% of rows is made overconfident by drawing labels from $\softmax((1-\gamma)z^{\rm r})$, or underconfident by using $\softmax(z^{\rm r}/(1-\gamma))$, with the described global temperature refit. The label law is a function of the graph and therefore belongs to $\mathcal O$. Water-filling computes the exact output-space frontier for this known law.

The internal levels in Table~\ref{tab:gap} and Figure~\ref{fig:where} use the reported numerical procedure: the reference-tail solver is supplied with conditional-mean slopes, with endpoint damage evaluated by an additional rollout for the reachability diagnostic and certified charges for the certification diagnostic. Table~\ref{tab:gap} denotes the achieved numerical levels by $\widetilde U_{\rm reach}$ and $\widetilde U_{\rm cert}$. The shorter labels in Figure~\ref{fig:where} refer to those same computed values, not to certified global optima of Theorem~\ref{thm:frontier}. Warm-starting from the next lower level establishes the reported numerical ordering, not global optimality. If a level is the gain of a feasible trajectory, it lower-bounds its own exact supremum; a difference between two such computed levels need not bound the difference between the suprema.

Table~\ref{tab:gap} also gives local information, estimation, uncertainty, and execution terms as surrogate diagnostics. The identification $\sum_\ell A_\ell^{\rm cert}=U_{\rm cert}$ is not assumed. Relating the last column to the true gain requires verification that each term uses the same convex branch, physical-coordinate majorizer, conditioning information, and trajectory convention as~\eqref{eq:localgap}, including any branch-selection contribution. Heads that omit a planted-community identifier may have an information gap, but its positivity does not follow from omission of that identifier alone because other inputs can encode the same information.

\textbf{Real data.} $\pB$ is replaced by a cross-fitted calibrated classifier trained on the labels of the 2018 calibration episodes. $\hat R_i-1$ and $\hat U_{\rm out}$ are therefore plug-in quantities, and clipping $\pr_{ic}$ at $10^{-3}$ changes the medians by at most \val{0.2}. Theorem language is reserved for the known law.

\textbf{Efficiency and placement.} $\eta_\ell$ is the linear program of Theorem~\ref{thm:frontier}(ii) on the candidate hull, with the cross-fitted posterior, averaged over development calls; Figure~\ref{fig:where}(c) and Table~\ref{tab:factors} report it normalized by its bound $\max_i(R_{\ell,i}-1)$. Caps use a water-filling-inspired allocation heuristic based on replayed certified efficiency, the gain per certified nat, under an expected-spend budget. The exact water-filling theorem concerns the output-space frontier; it does not establish global optimality of this internal cap-placement heuristic. The final checker, not the allocation heuristic, enforces each call's budgets.

\textbf{Row groups of Figure~\ref{fig:where}(d).} A \emph{near-cancellation row} is one whose raw-update terminal prediction satisfies $\|\pu_i-\pr_i\|_1<10^{-3}$ while the proposal moved some row within its remaining-tail receptive field (at most $T-\ell-1$ transitions after injection) at an open depth. A \emph{beyond-two-hops row} is one for which more than half of the raw update's influence on its terminal logits, computed with the tangent basis, comes from rows more than two hops away. Contributions are the reported paired per-row gains summed with call weights and averaged over the mixture. The theoretical remaining-tail count and the empirical tangent-influence rule are distinct definitions; the reported group memberships follow the latter measurement protocol. A disjoint assignment rule is required before the three displayed contributions can be interpreted as a partition of the total gain.

\begin{table}[!htbp]\centering
\caption{Reported frontier and internal-optimization diagnostics ($10^{-3}$ nats per call). (a) Arxiv with a cross-fitted posterior. (b) Planted-partition diagnostics: $I$ denotes information, $E$ estimation, $\chi$ uncertainty conservatism, and $\Upsilon$ execution suboptimality. Tildes denote achieved numerical levels, hats denote empirical or posterior-based estimates, and the unmarked planted-law output frontier is analytic. These internal levels need not attain the exact suprema. The last column is the reported surrogate expression $\widetilde U_{\rm cert}-I-\chi-\frac12(\sqrt{2E}+\sqrt{2\Upsilon})^2$; without the additional identifications discussed above, it is not a proved global lower bound on RTT's gain. The rounded component values do not uniquely determine this nonlinear expression; reproducing it requires the unrounded records.}\label{tab:gap}
\TabGap
\end{table}

\begin{figure}[!htbp]
\centering\includegraphics[width=\linewidth]{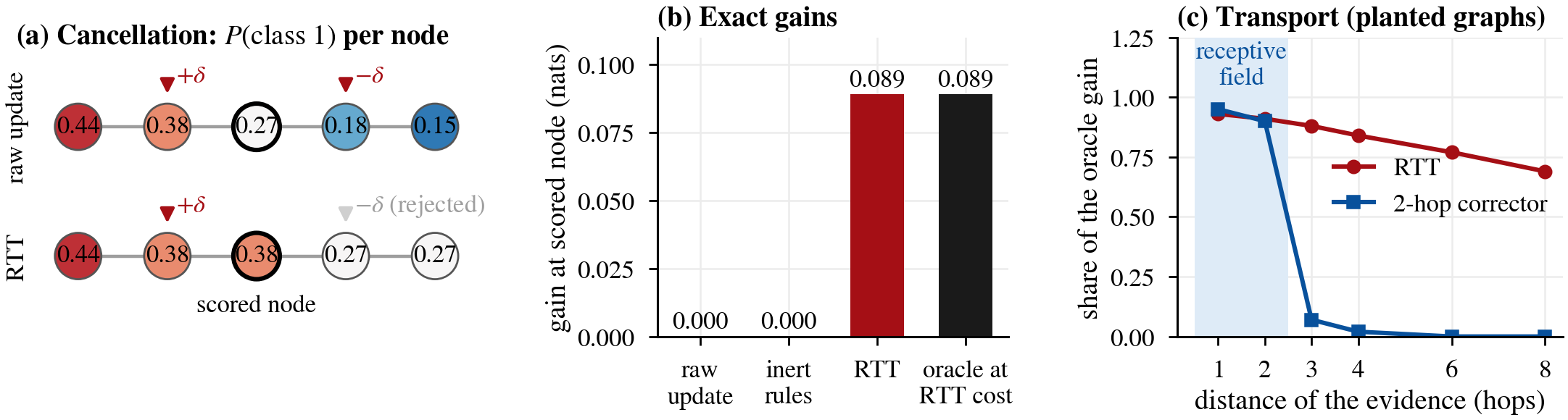}
\caption{Separation diagnostics for Proposition~\ref{prop:separation}. (a--b) The cancellation instance is computed analytically. (a)~Cancellation: probability of class 1 per node under the raw update and under RTT, which admits only the helpful candidate. (b)~Gains at the scored node. (c)~Planted graphs: share of the oracle gain against the distance of the evidence, for RTT and a two-hop corrector.}\label{fig:sepinst}
\end{figure}

\section{Extended results}\label{app:results}
\textbf{Datasets and incumbents.} Breadth uses the GOOD-Arxiv degree and time shifts \citep{good}; the Twitch, Facebook-100 and Elliptic shifts of \citet{eerm} on the data of \citet{twitch}, \citet{fb100} and \citet{elliptic}; ogbn-products and ogbg-molpcba \citep{ogb}; Folktables ACSIncome \citep{folktables} with a residual tanh MLP; and ImageNet-C \citep{imagenetc} with the last four blocks of a ResNet-50 \citep{resnet} adapted by Tent. The incumbent families are APPNP \citep{appnp}, GCNII \citep{gcnii}, GRAND \citep{grand}, shallow GraphSAGE \citep{hamilton2017} and GCN \citep{kipf2017}, and GIN-style molecular encoders \citep{xu2019gin}; Table~\ref{tab:families} lists their primitives. Every entry below is computed from the registered per-episode records by the same script as the main text.

\begin{table}[!htbp]\centering
\caption{Complete frontier on arxiv (final audit, as Table~\ref{tab:frontier}); Holm-adjusted $p$-values for the secondary comparisons.}\label{tab:frontierfull}
\TabFrontierFull
\end{table}
\begin{table}[!htbp]\centering
\caption{What each mechanism buys: one swap at a time on the registered configuration (final audit; $10^{-3}$ nats; alternative $\to$ registered). Swaps are order-dependent comparisons, not an additive decomposition.}\label{tab:swaps}
\TabSwaps
\end{table}
\begin{table}[!htbp]\centering
\caption{Neighboring methods, and test-time adaptation under a batch-shift collapse (20 streams of 600 calls, a separate experiment from Figure~\ref{fig:service}, shift at call 150; damage in $10^{-3}$ nats on the shifted segment; positive is a loss). Part (a) reports the neighboring methods themselves; proposal-specific RTT comparisons are given in Table~\ref{tab:class}. These rows do not establish an RTT wrapper experiment for each neighboring method.}\label{tab:neighbors}
\TabNeighbours
\end{table}
\begin{table}[!htbp]\centering
\caption{Breadth under one protocol (final audits, 2,048 episodes; $10^{-3}$ nats per task unit; $\zeta$ in $10^{-3}$). $\hat U_{\rm out}$: estimated oracle frontier; share: RTT gain over $\hat U_{\rm out}$; verified: calls released on the first pass; the last column gives the task metric of the incumbent and of RTT on the audited episodes.}\label{tab:breadth}
\TabBreadth
\end{table}
\begin{table}[!htbp]\centering
\caption{ogbg-molpcba on four reference families: the reserved scaffold halves define a custom finite-population protocol with exact empirical test-half AP. The separately reported official-test evaluation uses the full official test population and is not independent of any custom split drawn from that same population.}\label{tab:mol}
\TabMol
\end{table}
\FloatBarrier
\begin{figure}[H]
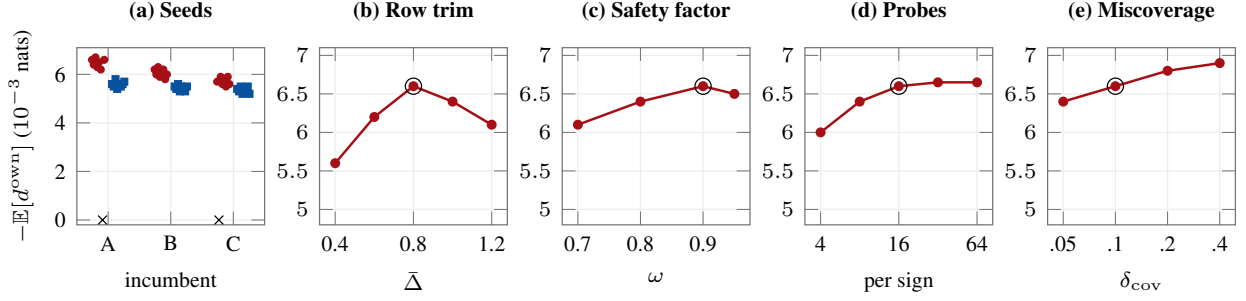

\centering\resizebox{\linewidth}{!}{\FigRobust}
\caption{Robustness on the development half. (a)~Three incumbents $\times$ eight training seeds: gain of every passing seed for RTT (red circles) and the one-pass corrector (blue squares); crosses: failed selections, deployed as the incumbent (gain 0, plotted at zero). (b--e)~Gain against the row trim, the safety factor, the number of support-function probes and the trajectory miscoverage budget (circle: registered value).}\label{fig:robust}
\end{figure}
\begin{figure}[H]
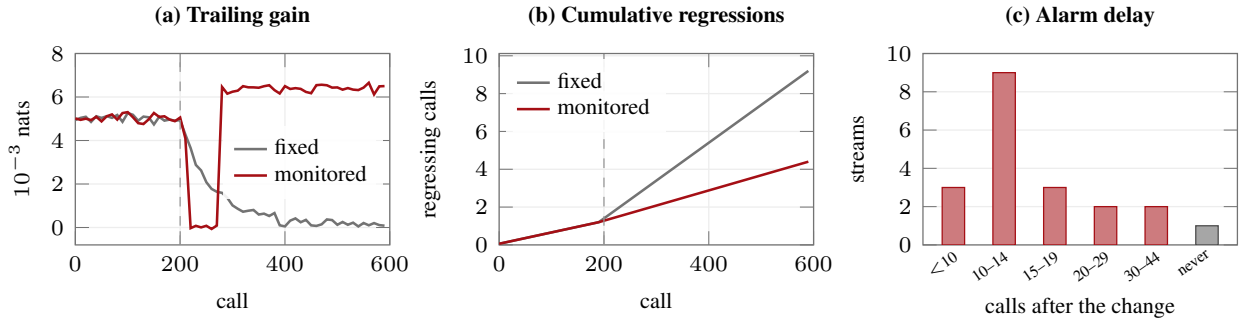

\centering\resizebox{\linewidth}{!}{\FigService}
\caption{Reported monitored-service traces and alarm-delay counts from the 20-stream study (held-out corruption family begins at call 200): (a)~trailing gain of the fixed and the monitored policy; (b)~cumulative regressions; (c)~alarm delays, including streams that never alarm.}\label{fig:service}
\end{figure}

\section{Cost, numerics and the checker}\label{app:cost}
\textbf{One accounting convention.} Latency is defined per fresh 4,096-node call on the full arxiv graph on one A100 80GB, relative to one incumbent pass of \vTinc\ ms. RTT's machinery (heads, window, two-branch solver, provisional contracts and tube, certificate, expected fallback) adds \vRttOverhead\ ms to the proposal's own cost. Latency ratios therefore follow $(t_{\rm inc}+t_{\rm prop}+t_{\rm RTT})/t_{\rm inc}$, and the proposal's cost is counted once. Two rollouts cost $t_{\rm inc}+t_{\rm raw}$ sequentially and $\max\{t_{\rm inc},t_{\rm raw}\}$ on two devices (Table~\ref{tab:costs}). For the native, GPR-GNN, Co-GNN, and AMP rows, the listed proposal cost is incremental to the shared incumbent computation and $t_{\rm raw}=t_{\rm inc}+t_{\rm prop}$. For GTrans, Matcha, v2, and LoRA, it is the complete proposal runtime and $t_{\rm raw}=t_{\rm prop}$. The corrector row instead reports its additional cost beyond one incumbent pass; it does not add RTT's machinery. For in-network proposals RTT uses fewer FLOPs than two rollouts. For proposals that are full networks (v2, LoRA, GTrans, Matcha) it is slower than two sequential rollouts, and its case there rests on gain and on the per-node floor.

\begin{table}[!htbp]\centering
\caption{Costs: (a)~latency ratios by proposal, with two rollouts sequential and on two devices; (b)~episode-evaluation and compute ledger of the reported study. Part (a) contains no separate FLOP column; relative FLOPs for the primary comparison are in Table~\ref{tab:frontier}.}\label{tab:costs}
\TabCosts
\end{table}

\textbf{The reference is the deployed float32 incumbent.} Theorems are stated in real arithmetic. The certificate is computed by a checker in float64 with directed rounding. It adds two allowances: the executed pass's state-update residual, and the forward error of the incumbent's own float32 continuation, using accumulation-error factors $\gamma_m=m\epsilon_{\rm m}/(1-m\epsilon_{\rm m})$, where $m$ is the relevant accumulation length, $m\epsilon_{\rm m}<1$, and $\epsilon_{\rm m}$ is the unit roundoff. These factors must multiply the corresponding absolute-sum bounds and be propagated through the remaining primitives, including a validated $\tanh$ enclosure. Both are deducted from the budgets. The checker is implemented separately from the executor. The reported tests include 10,000 planted violations whose hidden margins range from $10^{-7}$ to $10^{-3}$ nats, following SoundnessBench \citep{soundnessbench}; on 1,000 certificates re-evaluated in 50-digit arithmetic; and against corrupted factors, corrupted steps and switched-off rounding (Table~\ref{tab:checker}).

\begin{table}[!htbp]\centering
\caption{Checker and stress-test results: planted violations, high-precision re-evaluation, endpoint costs, adversarial directions, corrupted inputs, and displacement invariance. A passing finite test is evidence about the tested instances, not a proof of universal numerical soundness.}\label{tab:checker}
\TabChecker
\end{table}

\begin{figure}[!htbp]
\centering\includegraphics[width=\linewidth]{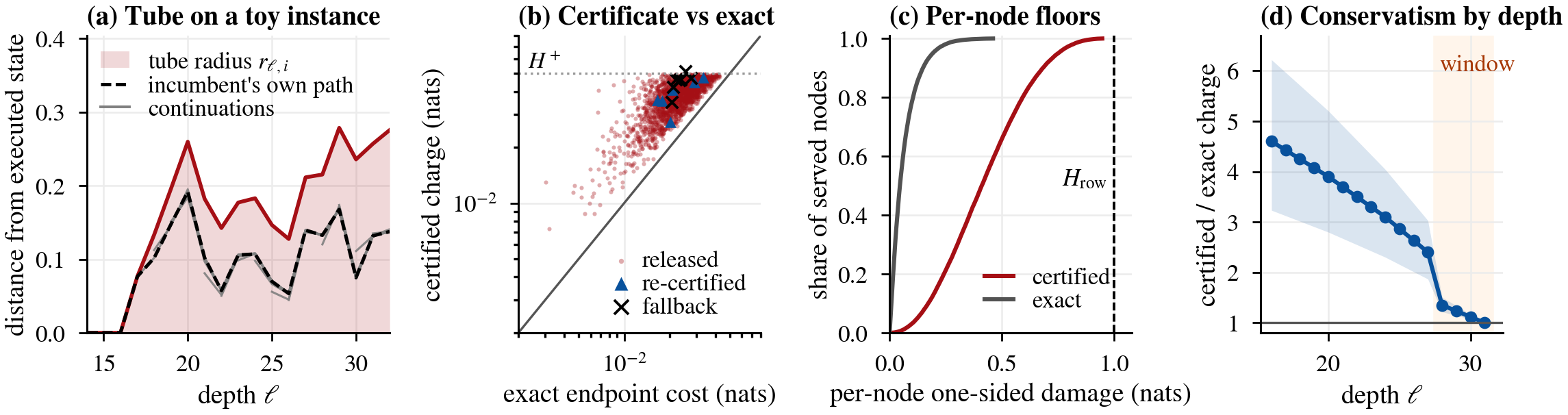}
\caption{The certificate. (a)~Toy contractive incumbent, computed exactly: distance of every incumbent continuation started on an executed segment from the executed state, against the tube radius of Lemma~\ref{lem:tube}. (b)~Certified charge against the offline endpoint cost for the 2,048 candidate trajectories, including trajectories subsequently rejected by the selector; no trajectory lies below the diagonal. A fallback serves the incumbent and therefore has zero served loss contrast. (c)~Distribution over served nodes of the certified per-node bound $\Gamma_\Delta r_{T,i}$ (red) and of the exact one-sided damage $h_i=\Dinf(\pr_i\|\ps_i)$ (grey). (d)~Certified-to-exact charge ratio by depth, median and 5--95\% band; shaded: the window.}\label{fig:certificate}
\end{figure}
\FloatBarrier
\begin{figure}[H]
\centering\includegraphics[width=\linewidth]{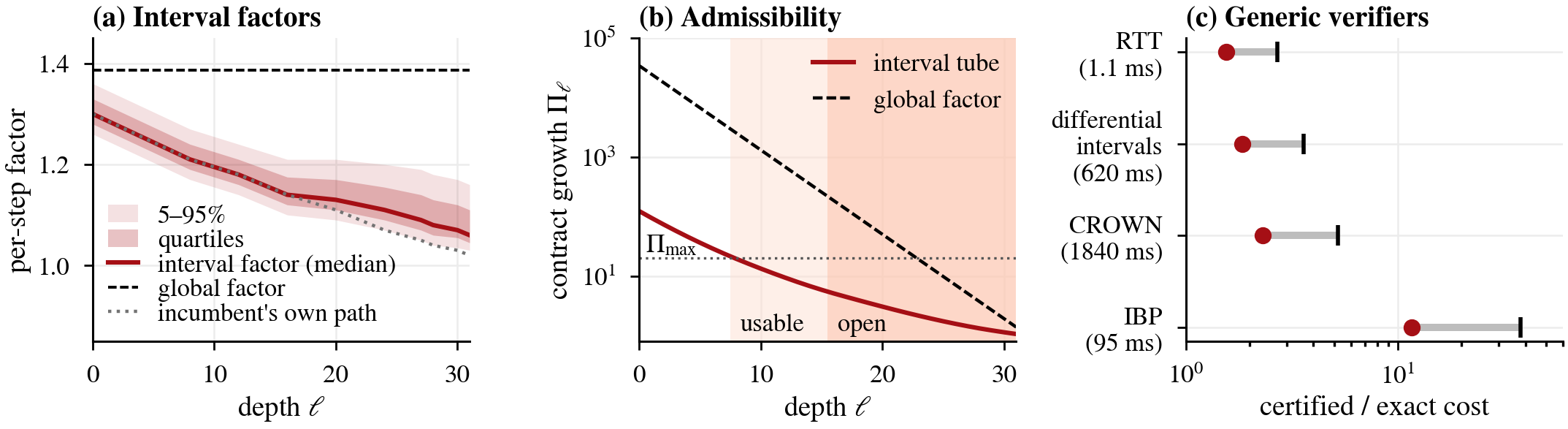}
\caption{Mechanism. (a)~Per-step factors by depth: interval factors on the tube, the factor recorded along the incumbent's own trajectory, and the global factor. (b)~Contract growth $\Pi_\ell$ with the admissibility cap $\Pi_{\max}$; the open depths must be admissible. (c)~Generic verifiers on the same executed interventions (64-node calls, last eight depths): certified-to-exact ratio and time per call.}\label{fig:mechanism}
\end{figure}

\Needspace{24\baselineskip}
\section{Positioning and notation}\label{app:repro}
\begin{table}[H]\centering
\caption{Positioning on explicit axes: when the check happens, what is changed, what is guaranteed, whether the reference prediction is computed, and which updates are supported.}\label{tab:positioning}
\TabPositioning
\end{table}
\begin{table}[H]\centering
\caption{Notation.}\label{tab:notation}
\TabNotation
\end{table}
\end{document}